\documentclass{article}

\usepackage{mathtools,latexsym}
\usepackage{amsthm,thmtools,thm-restate}

\PassOptionsToPackage{sort&compress, round}{natbib}

\usepackage[preprint]{neurips_2019}

\usepackage[utf8]{inputenc} 
\usepackage[T1]{fontenc}    
\PassOptionsToPackage{hyphens}{url}
\usepackage{hyperref}       
\usepackage{url}            
\usepackage{booktabs}       
\usepackage{amsfonts}       
\usepackage{bbm}
\usepackage{amssymb}
\usepackage{nicefrac}       
\usepackage{microtype}      

\usepackage[capitalize]{cleveref}
\usepackage{csquotes}

\usepackage{pgfplots,pgfplotstable}
\pgfplotsset{compat=1.14}
\usepgfplotslibrary{external}
\immediate\write18{mkdir -p cache}
\usepgfplotslibrary{groupplots}

\usepgfplotslibrary{colorbrewer}
\pgfplotsset{
  cycle list/Dark2-8,
  cycle multiindex* list={
    mark list*\nextlist
    Dark2-8\nextlist
  },
}

\declaretheoremstyle[
  notefont = \bfseries,
  bodyfont=\normalfont\itshape,
  spaceabove = \parskip,
  spacebelow = 0pt
  ]{theoremstyle}
\declaretheorem[style=theoremstyle]{theorem}
\declaretheorem[style=theoremstyle]{lemma}

\declaretheorem[style=theoremstyle]{corollary}

\declaretheoremstyle[
  postheadspace = \newline,
  notefont = \bfseries,
  headpunct = {},
  bodyfont = \normalfont,
  spaceabove = \parskip,
  spacebelow = 0pt
  ]{examplestyle}
\declaretheorem[style=examplestyle]{example}

\declaretheoremstyle[
  notefont = \bfseries,
  bodyfont = \normalfont,
  spaceabove = \parskip,
  spacebelow = 0pt
  ]{definitionstyle}
\declaretheorem[style=definitionstyle]{definition}

\declaretheoremstyle[
  headfont = \normalfont\itshape,
  notefont = \normalfont\itshape,
  bodyfont = \normalfont,
  spaceabove = \parskip,
  spacebelow = 0pt
  ]{remarkstyle}
\declaretheorem[style=remarkstyle]{remark}

\DeclareMathOperator{\Prob}{\mathbb{P}}
\DeclareMathOperator{\Expect}{\mathbb{E}}
\DeclareMathOperator{\Var}{Var}
\DeclareMathOperator{\Ber}{Ber}
\DeclareMathOperator{\Dir}{Dir}
\DeclareMathOperator{\Categorical}{Cat}
\DeclareMathOperator*{\argmax}{arg\,max}
\DeclareMathOperator{\Span}{span}
\DeclareMathOperator{\ECE}{ECE}
\DeclareMathOperator{\measure}{CE}
\DeclareMathOperator{\kernelmeasure}{KCE}
\DeclareMathOperator{\squaredkernelmeasure}{SKCE}
\DeclareMathOperator{\biasedestimator}{\widehat{SKCE}_b}
\DeclareMathOperator{\unbiasedestimator}{\widehat{SKCE}_{uq}}
\DeclareMathOperator{\linearestimator}{\widehat{SKCE}_{ul}}
\DeclareMathOperator{\MCE}{MCE}
\DeclareMathOperator{\MMCE}{MMCE}
\DeclareMathOperator*{\esssup}{ess\,sup}
\newcommand{\nullhypothesis}{H_0}
\newcommand{\given}{\,|\,}

\title{Calibration tests in multi-class classification: \\ A unifying framework}

\author{%
  David Widmann \\
  Department of Information Technology \\
  Uppsala University, Sweden \\
  \url{david.widmann@it.uu.se} \\
  \And
  Fredrik Lindsten \\
  Division of Statistics and Machine Learning \\
  Linköping University, Sweden \\
  \url{fredrik.lindsten@liu.se} \\
  \AND
  Dave Zachariah \\
  Department of Information Technology \\
  Uppsala University, Sweden \\
  \url{dave.zachariah@it.uu.se} \\
}

\begin{document}

\maketitle

\begin{abstract}
    In safety-critical applications a probabilistic model is usually required to
    be calibrated, i.e., to capture the uncertainty of its predictions accurately.
    In multi-class classification, calibration of the most confident predictions
    only is often not sufficient. We propose and study calibration measures for
    multi-class classification that generalize existing measures such as the
    expected calibration error, the maximum calibration error, and the maximum
    mean calibration error. We propose and evaluate empirically different
    consistent and unbiased estimators for a specific class of measures based on
    matrix-valued kernels. Importantly, these estimators can be interpreted as
    test statistics associated with well-defined bounds and approximations of
    the p-value under the null hypothesis that the model is calibrated,
    significantly improving the interpretability of calibration measures, which
    otherwise lack any meaningful unit or scale.
\end{abstract}

\section{Introduction}

Consider the problem of analyzing microscopic images of tissue samples and
reporting a tumour grade, i.e., a score that indicates whether cancer cells are
well-differentiated or not, affecting both prognosis and treatment of patients.
Since for some pathological images not even experienced pathologists might all
agree on one classification, this task contains an inherent component of
uncertainty. This type of uncertainty that can not be removed by increasing the
size of the training data set is typically called aleatoric uncertainty
\citep{kiureghian09_aleat_or_epist}. Unfortunately, even if the ideal model is
among the class of models we consider, with a finite training data set we will
never obtain the ideal model but we can only hope to learn a model that is, in
some sense, close to it. Worse still, our model might not even be close to the
ideal model if the model class is too restrictive or the number of training data
is small---which is not unlikely given the fact that annotating pathological
images is expensive. Thus ideally our model should be able to express not only
aleatoric uncertainty but also the uncertainty about the model itself. In
contrast to aleatoric uncertainty this so-called epistemic uncertainty can be
reduced by additional training data.

Dealing with these different types of uncertainty is one of the major problems in
machine learning. The application of our model in clinical practice demands
\enquote{meaningful} uncertainties to avoid doing harm to patients. Being too
certain about high tumour grades might cause harm due to unneeded aggressive
therapies and overly pessimistic prognoses, whereas being too certain about low
tumour grades might result in insufficient therapies. \enquote{Proper}
uncertainty estimates are also crucial if the model is supervised by a
pathologist that takes over if the uncertainty reported by the model is too high.
False but highly certain gradings might incorrectly keep the pathologist out of
the loop, and on the other hand too uncertain gradings might demand unneeded and
costly human intervention.

Probability theory provides a solid framework for dealing with uncertainties.
Instead of assigning exactly one grade to each pathological image, so-called
probabilistic models report subjective probabilities, sometimes also called
confidence scores, of the tumour grades for each image. The model can be
evaluated by comparing these subjective probabilities to the ground truth.

One desired property of such a probabilistic model is sharpness (or high
accuracy), i.e., if possible, the model should assign the highest probability to
the true tumour grade (which maybe can not be inferred from the image at hand
but only by other means such as an additional immunohistochemical staining).
However, to be able to trust the predictions the probabilities should be
calibrated (or reliable) as well
\citep{murphy77_reliab_subjec_probab_forec_precip_temper,degroot83_compar_evaluat_forec}.
This property requires the subjective probabilities to match the relative
empirical frequencies: intuitively, if we could observe a long run of
predictions $(0.5, 0.1, 0.1, 0.3)$ for tumour grades $1$, $2$, $3$, and $4$, the
empirical frequencies of the true tumour grades should be
$(0.5, 0.1, 0.1, 0.3)$. Note that accuracy and calibration are two complementary
properties: a model with over-confident predictions can be highly accurate but
miscalibrated, whereas a model that always reports the overall proportion of
patients of each tumour grade in the considered population is calibrated but
highly inaccurate.

Research of calibration in statistics and machine learning literature has been
focused mainly on binary classification problems or the most confident
predictions: common calibration measures such as the expected calibration error
($\ECE$)~\citep{naeini15_obtain_bayes}, the maximum calibration error
($\MCE$)~\citep{naeini15_obtain_bayes}, and the kernel-based maximum mean
calibration error ($\MMCE$)~\citep{kumar18_train_calib_measur_neural_networ},
and reliability
diagrams~\citep{murphy77_reliab_subjec_probab_forec_precip_temper} have been
developed for binary classification. This is insufficient since many recent
applications of machine learning involve multiple classes. Furthermore, the
crucial finding of \citet{guo17_calib_moder_neural_networ} that many modern deep
neural networks are miscalibrated is also based only on the most confident
prediction.

Recently \citet{vaicenavicius19_evaluat} suggested that this analysis might
be too reduced for many realistic scenarios. In our example, a prediction of
$(0.5, 0.3, 0.1, 0.1)$ is fundamentally different from a prediction of
$(0.5, 0.1, 0.1, 0.3)$, since according to the model in the first case it is only
half as likely that a tumour is of grade 3 or 4, and hence the subjective
probability of missing out on a more aggressive therapy is smaller. However,
commonly in the study of calibration all predictions with a highest reported
confidence score of $0.5$ are grouped together and a calibrated model has only to
be correct about the most confident tumour grade in 50\% of the cases, regardless
of the other predictions. Although the $\ECE$ can be generalized to multi-class
classification, its applicability seems to be limited since its
histogram-regression based estimator requires partitioning of the potentially
high-dimensional probability simplex and is asymptotically inconsistent in many
cases~\citep{vaicenavicius19_evaluat}. Sample complexity bounds for a
bias-reduced estimator of the $\ECE$ introduced in metereological
literature~\citep{ferro12_bias_correc_decom_brier_score,broecker11_estim_reliab_resol_probab_forec}
were derived in concurrent work~\citep{kumar19_verif_uncer_calib}.

\section{Our contribution}

In this work, we propose and study a general framework of calibration measures
for multi-class classification. We show that this framework encompasses common
calibration measures for binary classification such as the expected calibration
error ($\ECE$), the maximum calibration error ($\MCE$), and the maximum mean
calibration error ($\MMCE$) by \citet{kumar18_train_calib_measur_neural_networ}. In
more detail we study a class of measures based on vector-valued reproducing
kernel Hilbert spaces, for which we derive consistent and unbiased estimators.
The statistical properties of the proposed estimators are not only theoretically
appealing, but also of high practical value, since they allow us to address two
main problems in calibration evaluation.

As discussed by \citet{vaicenavicius19_evaluat}, all calibration error
estimates are inherently random, and comparing competing models based on these
estimates without taking the randomness into account can be very misleading, in
particular when the estimators are biased (which, for instance, is the case for
the commonly used histogram-regression based estimator of the $\ECE$). Even more
fundamentally, all commonly used calibration measures lack a meaningful unit or
scale and are therefore not interpretable as such (regardless of any finite
sample issues).

The consistency and unbiasedness of the proposed estimators facilitate
comparisons between competing models, and allow us to derive multiple
statistical tests for calibration that exploit these properties. Moreover, by
viewing the estimators as \emph{calibration test statistics}, with well-defined
bounds and approximations of the corresponding p-value, we give them an
interpretable meaning.

We evaluate the proposed estimators and statistical tests empirically and
compare them with existing methods. To facilitate multi-class calibration
evaluation we provide the Julia packages \href{https://github.com/devmotion/ConsistencyResampling.jl}{\texttt{ConsistencyResampling.jl}}~\citep{widmann19_consis_resampling},
\href{https://github.com/devmotion/CalibrationErrors.jl}{\texttt{CalibrationErrors.jl}}~\citep{widmann19_calib_errors}, and
\href{https://github.com/devmotion/CalibrationTests.jl}{\texttt{CalibrationTests.jl}}~\citep{widmann19_calib_tests}
for consistency resampling, calibration error estimation, and calibration tests,
respectively.

\section{Background}

We start by shortly summarizing the most relevant definitions and concepts. Due
to space constraints and to improve the readability of our paper, we do not
provide any proofs in the main text but only refer to the results in the
supplementary material, which is intended as a reference for mathematically
precise statements and proofs.

\subsection{Probabilistic setting}

Let $(X,Y)$ be a pair of random variables with $X$ and $Y$ representing inputs
(features) and outputs, respectively. We focus on classification problems and
hence without loss of generality we may assume that the outputs consist of the
$m$ classes $1$, \ldots, $m$.

Let $\Delta^m$ denote the $(m-1)$-dimensional probability simplex
$\Delta^m \coloneqq \{ z \in \mathbb{R}_{\geq 0}^m \colon \|z\|_1 = 1\}$. Then
a \emph{probabilistic model} $g$ is a function that for every input $x$ outputs
a prediction $g(x) \in \Delta^m$ that models the distribution
\begin{equation*}
  \big(\Prob[Y = 1 \given X = x], \ldots, \Prob[Y=m \given X = x]\big) \in \Delta^m
\end{equation*}
of class $Y$ given input $X = x$.

\subsection{Calibration}

\subsubsection{Common notion}

The common notion of calibration, as, e.g., used by \citet{guo17_calib_moder_neural_networ},
considers only the most confident predictions $\max_y g_y(x)$ of a model $g$.
According to this definition, a model is calibrated if
\begin{equation}\label{eq:common_calibration}
    \Prob[Y = \argmax_y g_y(X) \given \max_y g_y(X)] = \max_y g_y(X)
\end{equation}
holds almost always. Thus a model that is calibrated according to
\cref{eq:common_calibration} ensures that we can \emph{partly trust} the
uncertainties reported by its predictions. As an example, for a prediction of
$(0.4, 0.3, 0.3)$ the model would only guarantee that in the long run inputs that
yield a most confident prediction of $40\%$ are in the corresponding class $40\%$
of the time.\footnote{This notion of calibration does not consider for which class
the most confident prediction was obtained.}

\subsubsection{Strong notion}

According to the more general calibration definition of
\citet{broecker09_reliab_suffic_decom_proper_scores,vaicenavicius19_evaluat}, a
probabilistic model $g$ is calibrated if for almost all inputs $x$ the
prediction $g(x)$ is equal to the distribution of class $Y$ given prediction
$g(X) = g(x)$. More formally, a calibrated model satisfies
\begin{equation}\label{eq:strong_calibration}
  \Prob[Y = y \given g(X)] = g_y(X)
\end{equation}
almost always for all classes $y \in \{1,\ldots,m\}$. As
\citet{vaicenavicius19_evaluat} showed, for multi-class classification
this formulation is stronger than the definition of \citet{zadrozny02_trans} that
only demands calibrated marginal probabilities. Thus we can \emph{fully trust}
the uncertainties reported by the predictions of a model that is calibrated
according to \cref{eq:strong_calibration}. The prediction $(0.4, 0.3, 0.3)$
would actually imply that the class distribution of the inputs that yield this
prediction is $(0.4, 0.3, 0.3)$. To emphasize the difference to the definition
in \cref{eq:common_calibration}, we call calibration according to
\cref{eq:strong_calibration} \emph{calibration in the strong sense} or
\emph{strong calibration}.

To simplify our notation, we rewrite \cref{eq:strong_calibration} in vectorized
form. Equivalent to the definition above, a model $g$ is calibrated in the
strong sense if
\begin{equation}\label{eq:strong_calibration_vector}
  r(g(X)) - g(X) = 0
\end{equation}
holds almost always, where
\begin{equation*}
  r(\xi) \coloneqq \big(\Prob[Y = 1 \given g(X) = \xi], \ldots, \Prob[Y = m \given g(X) = \xi]\big)
\end{equation*}
is the distribution of class $Y$ given prediction $g(X) = \xi$.

The calibration of certain aspects of a model, such as the five largest
predictions or groups of classes, can be investigated by studying the
strong calibration of models induced by so-called calibration lenses. For more
details about evaluation and visualization of strong calibration we refer to
\citet{vaicenavicius19_evaluat}.

\subsection{Matrix-valued kernels}

The miscalibration measure that we propose in this work is based on
matrix-valued kernels $k \colon \Delta^m \times \Delta^m \to \mathbb{R}^{m \times m}$.
Matrix-valued kernels can be defined in a similar way as the more common
real-valued kernels, which can be characterized as symmetric positive definite
functions \citep[Lemma~4]{berlinet04_reprod_kernel_hilber_spaces_probab_statis}.

\begin{definition}[{{\citet[Definition~2]{micchelli05_learn_vector_valued_funct}; \citet[Definition~1]{caponnetto08_univer_multi_task_kernel}}}]\label{def:kernel}
  We call a function $k \colon \Delta^m \times \Delta^m \to \mathbb{R}^{m \times m}$
  a \emph{matrix-valued kernel} if for all $s, t \in \Delta^m$
  $k(s,t) = k(t,s)^\intercal$ and it is positive semi-definite, i.e., if for all
  $n \in \mathbb{N}$, $t_1,\ldots,t_n \in \Delta^m$, and $u_1, \ldots, u_n \in \mathbb{R}^m$
  \begin{equation*}
    \sum_{i,j=1}^n u_i^\intercal k(t_i, t_j) u_j \geq 0.
  \end{equation*}
\end{definition}

There exists a one-to-one mapping of such kernels and reproducing kernel Hilbert
spaces (RKHSs) of vector-valued functions $f \colon \Delta^m \to \mathbb{R}^m$.
We provide a short summary of RKHSs of vector-valued functions on the probability
simplex in \cref{app:rkhs}. A more detailed general introduction to RKHSs of
vector-valued functions can be found in the publications by
\citet{micchelli05_learn_vector_valued_funct,carmeli10_vector_valued_reprod_kernel_hilber_spaces_univer,caponnetto08_univer_multi_task_kernel}.

Similar to the scalar case, matrix-valued kernels can be constructed from other
matrix-valued kernels and even from real-valued kernels. Very simple
matrix-valued kernels are kernels of the form
$k(s,t) = \tilde{k}(s,t) \mathbf{I}_m$, where $\tilde{k}$ is a scalar-valued
kernel, such as the Gaussian or Laplacian kernel, and $\mathbf{I}_m$ is the
identity matrix. As \cref{ex:lifted_kernel} shows, this construction can be
generalized by, e.g., replacing the identity matrix with an arbitrary positive
semi-definite matrix.

An important class of kernels are so-called universal kernels. Loosely speaking,
a kernel is called universal if its RKHS is a dense subset of the space of
continuous functions, i.e., if in the neighbourhood of every continuous function
we can find a function in the RKHS. Prominent real-valued kernels on the
probability simplex such as the Gaussian and the Laplacian kernel are universal,
and can be used to construct universal matrix-valued kernels of the form in
\cref{ex:lifted_kernel}, as \cref{lemma:lifted_kernel} shows.

\section{Unification of calibration measures}

In this section we introduce a general measure of strong calibration and show its
relation to other existing measures.

\subsection{Calibration error}

In the analysis of strong calibration the discrepancy in the left-hand side of
\cref{eq:strong_calibration_vector} lends itself naturally to the following
calibration measure.

\begin{definition}\label{def:ce_supremum}
  Let $\mathcal{F}$ be a non-empty space of functions
  $f \colon \Delta^m \to \mathbb{R}^m$. Then the calibration error ($\measure$)
  of model $g$ with respect to class $\mathcal{F}$ is
  \begin{equation*}
    \measure[\mathcal{F}, g]  \coloneqq \sup_{f \in \mathcal{F}} \Expect\left[{(r(g(X)) - g(X))}^\intercal f(g(X)) \right].
  \end{equation*}
\end{definition}

A trivial consequence of the design of the $\measure$ is that the measure is
zero for every model that is calibrated in the strong sense, regardless of the
choice of $\mathcal{F}$. The converse statement is not true in general. As we
show in \cref{thm:ce_continuous}, the class of continuous functions is a space
for which the $\measure$ is zero if and only if model $g$ is strongly
calibrated, and hence allows to characterize calibrated models. However, since
this space is extremely large, for every model the $\measure$ is either $0$ or
$\infty$.\footnote{Assume $\measure[\mathcal{F}, g] < \infty$ and let
  $f_1,f_2,\ldots$ be a sequence of continuous functions with
  $\measure[\mathcal{F}, g] = \lim_{n \to \infty} \Expect\left[{(r(g(X)) - g(X))}^\intercal f_n(g(X)) \right]$.
  From \cref{remark:positive} we know that $\measure[\mathcal{F}, g] \geq 0$.
  Moreover, $\tilde{f}_n \coloneqq 2f_n$ are continuous functions with
  $2\measure[\mathcal{F}, g] = \lim_{n \to \infty} \Expect\left[{(r(g(X)) - g(X))}^\intercal \tilde{f}_n(g(X)) \right] \leq \sup_{f \in \mathcal{F}} \Expect\left[{(r(g(X)) - g(X))}^\intercal f(g(X)) \right] = \measure[\mathcal{F}, g]$.
  Thus $\measure[\mathcal{F}, g] = 0$.}. Thus a measure based on this space
does not allow us to compare miscalibrated models and hence is rather
impractical.

\subsection{Kernel calibration error}

Due to the correspondence between kernels and RKHSs we can define the
following kernel measure.

\begin{definition}\label{def:kce}
  Let $k$ be a matrix-valued kernel as in \cref{def:kernel}. Then we define
  the kernel calibration error ($\kernelmeasure$) with respect to kernel $k$
  as $\kernelmeasure[k, g] \coloneqq \measure[\mathcal{F}, g]$, where
  $\mathcal{F}$ is the unit ball in the RKHS corresponding to kernel $k$.
\end{definition}

As mentioned above, a RKHS with a universal kernel is a dense subset of the
space of continuous functions. Hence these kernels yield a function space
that is still large enough for identifying strongly calibrated models.

\begin{theorem}[{{cf.~\cref{thm:ce_zero}}}]
  Let $k$ be a matrix-valued kernel as in \cref{def:kernel}, and assume that $k$
  is universal. Then $\kernelmeasure[k, g] = 0$ if and only if model $g$
  is calibrated in the strong sense.
\end{theorem}

From the supremum-based \cref{def:ce_supremum} it might not be obvious how the
$\kernelmeasure$ can be evaluated. Fortunately, there exists an equivalent
kernel-based formulation.

\begin{lemma}[{{cf.~\cref{lemma:meance}}}]
  Let $k$ be a matrix-valued kernel as in \cref{def:kernel}. If
  $\Expect[\|k(g(X),g(X))\|] < \infty$, then
  \begin{equation}\label{eq:ce2}
    \kernelmeasure[k,g] = {\Big(\Expect\big[{(e_Y - g(X))}^{\intercal} k(g(X), g(X')) {(e_{Y'} - g(X'))}\big]\Big)}^{1/2},
  \end{equation}
  where $(X', Y')$ is an independent copy of $(X,Y)$ and $e_i$ denotes the $i$th unit vector.
\end{lemma}

\subsection{Expected calibration error}

The most common measure of calibration error is the expected calibration error
($\ECE$). Typically it is used for quantifying calibration in a binary
classification setting but it generalizes to strong calibration in a
straightforward way. Let
$d \colon \Delta^m \times \Delta^m \to \mathbb{R}_{\geq 0}$ be a distance measure
on the probability simplex. Then the expected calibration error of a model $g$
with respect to $d$ is defined as
\begin{equation}\label{eq:ece_def}
  \ECE[d,g] = \Expect[d(r(g(X)), g(X))].
\end{equation}
If $d(p, q) = 0 \Leftrightarrow p = q$, as it is the case for standard choices
of $d$ such as the total variation distance or the (squared) Euclidean distance,
then $\ECE[d,g]$ is zero if and only if $g$ is strongly calibrated as per
\cref{eq:strong_calibration_vector}.

The ECE with respect to the cityblock distance, the total variation distance, or
the squared Euclidean distance, are special cases of the calibration error
$\measure$, as we show in \cref{lemma:ce_ece}.

\subsection{Maximum mean calibration error}

\Citet{kumar18_train_calib_measur_neural_networ} proposed a kernel-based
calibration measure, the so-called maximum mean calibration error ($\MMCE$), for
training (better) calibrated neural networks. In contrast to their work, in our
publication we do not discuss how to obtain calibrated models but focus on the
evaluation of calibration and on calibration tests. Moreover, the $\MMCE$
applies only to a binary classification setting whereas our measure quantifies
strong calibration and hence is more generally applicable. In fact, as we show
in \cref{ex:mmce_kumar_special}, the $\MMCE$ is a special case of the
$\kernelmeasure$.

\section{Calibration error estimators}\label{sec:estimators}

Consider the task of estimating the calibration error of model $g$ using a
validation set $\mathcal{D} = \{(X_i, Y_i)\}_{i=1}^n$ of $n$ i.i.d.\ random pairs
of inputs and labels that are distributed according to $(X,Y)$.

From the expression for the $\ECE$ in \cref{eq:ece_def}, the natural (and,
indeed, standard) approach for estimating the $\ECE$ is as the sample
average of the distance $d$ between the predictions $g(X)$ and the
calibration function $r(g(X))$. However, this is problematic since the
calibration function is not readily available and needs to be estimated as
well. Typically, this is addressed using histogram-regression, see, e.g.,
\citet{guo17_calib_moder_neural_networ,naeini15_obtain_bayes,vaicenavicius19_evaluat},
which unfortunately leads to inconsistent and biased estimators in many cases
\citep{vaicenavicius19_evaluat} and can scale poorly to large $m$. In
contrast, for the $\kernelmeasure$ in \cref{eq:ce2} there is no explicit
dependence on $r$, which enables us to derive multiple consistent and also
unbiased estimators.

Let $k$ be a matrix-valued kernel as in \cref{def:kernel} with
$\Expect[\|k(g(X),g(X))\|] < \infty$, and define for $1 \leq i,j \leq n$
\begin{equation*}
  h_{i,j} \coloneqq {(e_{Y_i} - g(X_i))}^\intercal k(g(X_i), g(X_j)) (e_{Y_j} - g(X_j)).
\end{equation*}
Then the estimators listed in \cref{tab:estimators} are consistent estimators of
the squared kernel calibration error
$\squaredkernelmeasure[k, g] \coloneqq \kernelmeasure^2[k,g]$
(see~\cref{lemma:skce_biased,lemma:skce_unbiased,lemma:skce_linear}). The
subscript letters \enquote{q} and \enquote{l} refer to the quadratic and linear
computational complexity of the unbiased estimators, respectively.

\begin{table}[!htbp]
  \begin{center}
    \caption{Three consistent estimators of the $\squaredkernelmeasure$.}
    \label{tab:estimators}
    \begin{tabular}{llll} \toprule
       Notation & Definition & Properties & Complexity\\ \midrule
       $\biasedestimator$ & $n^{-2} \sum_{i,j=1}^n h_{i,j}$ & biased & $O(n^2)$ \\
       $\unbiasedestimator$ & $ {\binom{n}{2}}^{-1} \sum_{1 \leq i < j \leq n} h_{i,j}$ & unbiased & $O(n^2)$ \\
       $\linearestimator$ & $ {\lfloor n/2\rfloor}^{-1} \sum_{i = 1}^{\lfloor n / 2\rfloor} h_{2i-1,2i}$ & unbiased & $O(n)$ \\ \bottomrule
    \end{tabular}
  \end{center}
\end{table}

\section{Calibration tests}\label{sec:tests}

In general, calibration errors do not have a meaningful unit or scale. This
renders it difficult to interpret an estimated non-zero error and to compare
different models. However, by viewing the estimates as test statistics with
respect to \cref{eq:strong_calibration_vector}, they obtain an interpretable
probabilistic meaning.

Similar to the derivation of the two-sample tests by
\citet{gretton12_kernel_two_sampl_test}, we can use the consistency and
unbiasedness of the estimators of the $\squaredkernelmeasure$ presented above to
obtain bounds and approximations of the p-value for the null hypothesis
$\nullhypothesis$ that the model is calibrated, i.e., for the probability that
the estimator is greater than or equal to the observed calibration error estimate,
if the model is calibrated. These bounds and approximations do not only allow us
to perform hypothesis testing of the null hypothesis $\nullhypothesis$, but they
also enable us to transfer unintuitive calibration error estimates to an
intuitive and interpretable probabilistic setting.

As we show in \cref{thm:biased_bound_uniform,thm:linear_bound_uniform,thm:unbiased_bound_uniform},
we can obtain so-called distribution-free bounds without making any assumptions
about the distribution of $(X,Y)$ or the model $g$. A downside of these uniform
bounds is that usually they provide only a loose bound of the p-value.

\begin{lemma}[{{Distribution-free bounds (see~\cref{thm:biased_bound_uniform,thm:linear_bound_uniform,thm:unbiased_bound_uniform})}}]
  Let $k$ be a matrix-valued kernel as in \cref{def:kernel}, and assume that
  $K_{p;q} \coloneqq \sup_{s,t \in \Delta^m} \|k(s,t)\|_{p;q} < \infty$ for some
  $1 \leq p,q \leq \infty$.\footnote{For a matrix $A$ we denote by
    $\|A\|_{p;q}$ the induced matrix norm
    $\sup_{x \neq 0} \|Ax\|_q / \|x\|_p$.} Let $t > 0$ and
  $B_{p;q} \coloneqq 2^{1 + 1 / p - 1 / q} K_{p;q}$, then for the biased estimator
  we can bound
  \begin{equation*}
    \Prob\Big[\biasedestimator \geq t \,\Big|\, \nullhypothesis \Big] \leq \exp{\bigg(-\frac{1}{2}{\Big(\max\Big\{0, \sqrt{nt / B_{p;q}} - 1 \Big\}\Big)}^2\bigg)},
  \end{equation*}
  and for either of the unbiased estimators
  $T \in \{\unbiasedestimator,\linearestimator\}$, we can bound
  \begin{equation*}
    \Prob\Big[T \geq t \,\Big|\, \nullhypothesis\Big] \leq \exp{\bigg(- \frac{\lfloor n/2 \rfloor t^2}{2 B_{p;q}^2}\bigg)}.
  \end{equation*}
\end{lemma}

Asymptotic bounds exploit the asymptotic distribution of the test statistics
under the null hypothesis, as the number of validation data points goes to
infinity. The central limit theorem implies that the linear estimator is
asymptotically normally distributed.

\begin{lemma}[{{Asymptotic distribution of $\linearestimator$ (see~\cref{cor:linear_asymptotic})}}]
  Let $k$ be a matrix-valued kernel as in \cref{def:kernel}, and assume that
  $\Expect[\|k(g(X),g(X))\|] < \infty$. If
  $\Var[h_{i,j}] < \infty$, then
  \begin{equation*}
    \Prob\Big[\sqrt{\lfloor n / 2\rfloor} \linearestimator \geq t \,\Big| \, \nullhypothesis \Big] \to 1 - \Phi\bigg(\frac{t}{\hat{\sigma}}\bigg) \quad \text{as} \quad n \to \infty,
  \end{equation*}
  where $\hat{\sigma}$ is the sample standard deviation of $h_{2i-1,2i}$
  ($i = 1,\ldots,\lfloor n / 2 \rfloor$) and $\Phi$ is the cumulative
  distribution function of the standard normal distribution.
\end{lemma}

In \cref{thm:quadratic_asymptotic_degenerate} we derive a theoretical expression
of the asymptotic distribution of $n \unbiasedestimator$, under the assumption
of strong calibration. This limit distribution can be approximated, e.g., by
bootstrapping \citep{arcones92_boots_u_v_statis} or Pearson curve fitting, as
discussed by \citet{gretton12_kernel_two_sampl_test}.

\section{Experiments}

We conduct experiments to confirm the derived theoretical properties of the
proposed calibration error estimators empirically and to compare them with
a standard histogram-regression based estimator of the $\ECE$, denoted by
$\widehat{\ECE}$.\footnote{The implementation of the experiments is available
  online at \url{https://github.com/devmotion/CalibrationPaper}.}

We construct synthetic data sets $\{(g(X_i), Y_i)\}_{i=1}^{250}$ of 250 labeled
predictions for $m = 10$ classes from three generative models. For each model we
first sample predictions $g(X_i) \sim \Dir(0.1, \dots, 0.1)$, and then
simulate corresponding labels $Y_i$ conditionally on $g(X_i)$ from
\begin{equation*}
  \mathbf{M1}\colon \, \Categorical(g(X_i)), \quad
  \mathbf{M2}\colon \, 0.5\Categorical(g(X_i)) + 0.5\Categorical(1,0,\dots,0), \quad
  \mathbf{M3}\colon \, \Categorical(0.1, \dots, 0.1),
\end{equation*}
where $\mathbf{M1}$ gives a calibrated model, and $\mathbf{M2}$ and
$\mathbf{M3}$ are uncalibrated. In \cref{sec:generative_models} we investigate
the theoretical properties of these models in more detail.

For simplicity, we use the matrix-valued kernel
$k(x, y) = \exp{(- \|x - y\| / \nu)} \mathbf{I}_{10}$, where the kernel bandwidth
$\nu > 0$ is chosen by the median heuristic. The total variation distance is a
common distance measure of probability distributions and the standard distance
measure for the
$\ECE$~\citep{guo17_calib_moder_neural_networ,vaicenavicius19_evaluat,broecker07_increas_reliab_reliab_diagr},
and hence it is chosen as the distance measure for all studied calibration
errors.

\subsection{Calibration error estimates}

In \cref{fig:errors_comparison} we show the distribution of $\widehat{\ECE}$ and
of the three proposed estimators of the $\squaredkernelmeasure$, evaluated on
$10^4$ randomly sampled data sets from each of the three models. The true
calibration error of these models, indicated by a dashed line, is calculated
theoretically for the $\ECE$ (see~\cref{sec:theoretical_ece}) and empirically
for the $\squaredkernelmeasure$ using the sample mean of all unbiased estimates
of $\unbiasedestimator$. The results confirm the unbiasedness of
$\linearestimator$.

\begin{figure}[!htbp]
  \begin{center}
    \tikzsetnextfilename{errors_comparison}
    \input{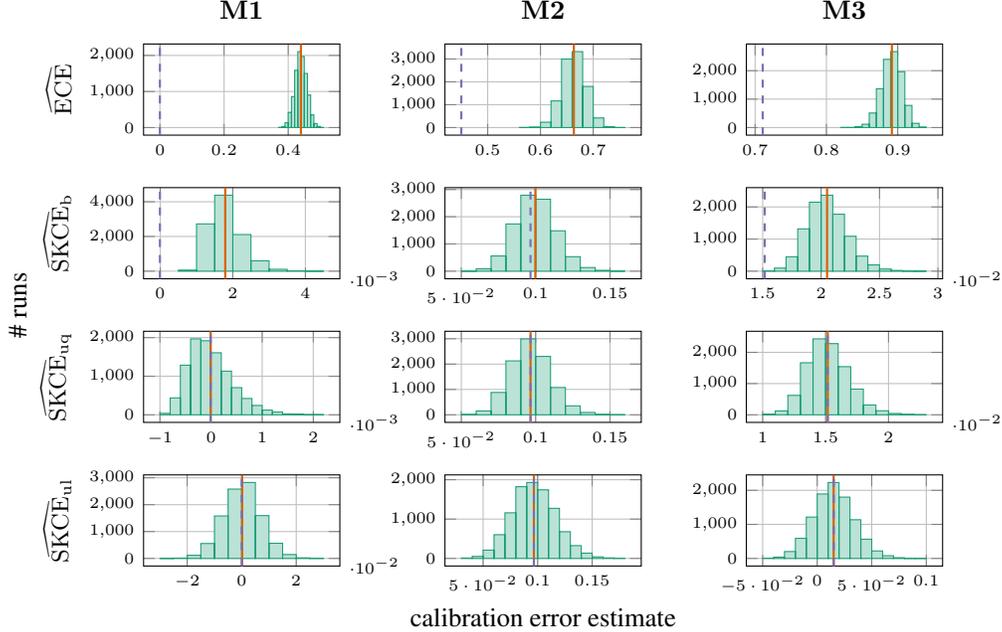}
    \caption{Distribution of calibration error estimates of $10^4$ data sets
      that are randomly sampled from the generative models $\mathbf{M1}$,
      $\mathbf{M2}$, and $\mathbf{M3}$. The solid line indicates the mean of the
      calibration error estimates, and the dashed line displays the true
      calibration error.}
    \label{fig:errors_comparison}
  \end{center}
\end{figure}

\subsection{Calibration tests}

We repeatedly compute the bounds and approximations of the p-value for the
calibration error estimators that were derived in \cref{sec:tests} on $10^4$
randomly sampled data sets from each of the three models. More concretely, we
evaluate the distribution-free bounds for $\biasedestimator$
($\mathbf{D}_{\mathrm{b}}$), $\unbiasedestimator$ ($\mathbf{D}_{\mathrm{uq}}$), and
$\linearestimator$ ($\mathbf{D}_{\mathrm{ul}}$) and the asymptotic
approximations for $\unbiasedestimator$ ($\mathbf{A}_{\mathrm{uq}}$) and
$\linearestimator$ ($\mathbf{A}_{\mathrm{l}}$), where the former is approximated
by bootstrapping. We compare them with a previously proposed hypothesis test for
the standard $\ECE$ estimator based on consistency resampling ($\mathbf{C}$), in
which data sets are resampled under the assumption that the model is calibrated
by sampling labels from resampled predictions
\citep{broecker07_increas_reliab_reliab_diagr,vaicenavicius19_evaluat}.

For a chosen significance level $\alpha$ we compute from the p-value
approximations $p_1,\ldots,p_{10^4}$ the empirical test error
\begin{equation*}
  \frac{1}{10^4} \sum_{i=1}^{10^4} \mathbbm{1}_{[0, \alpha]}(p_i) \quad \text{ (for } \mathbf{M1} \text{)}
  \qquad \text{and} \qquad
  \frac{1}{10^4} \sum_{i=1}^{10^4} \mathbbm{1}_{(\alpha, 1]}(p_i) \quad \text{ (for } \mathbf{M2} \text{ and } \mathbf{M3} \text{)}.
\end{equation*}
In \cref{fig:pvalues_comparison} we plot these empirical test errors versus the
significance level.

As expected, the distribution-free bounds seem to be very loose upper bounds of
the p-value, resulting in statistical tests without much power. The asymptotic
approximations, however, seem to estimate the p-value quite well on average, as
can be seen from the overlap with the diagonal in the results for the calibrated
model $\mathbf{M1}$ (the empirical test error matches the chosen significance
level). Additionally, calibration tests based on asymptotic distribution of
these statistics, and in particular of $\unbiasedestimator$, are quite powerful
in our experiments, as the results for the uncalibrated models $\mathbf{M2}$ and
$\mathbf{M3}$ show. For the calibrated model, consistency resampling leads to an
empirical test error that is not upper bounded by the significance level, i.e.,
the null hypothesis of the model being calibrated is rejected too often. This
behaviour is caused by an underestimation of the p-value on average, which
unfortunately makes the calibration test based on consistency resampling for the
standard $\ECE$ estimator unreliable.

\begin{figure}[!htbp]
  \begin{center}
    \tikzsetnextfilename{pvalues_comparison}
    \input{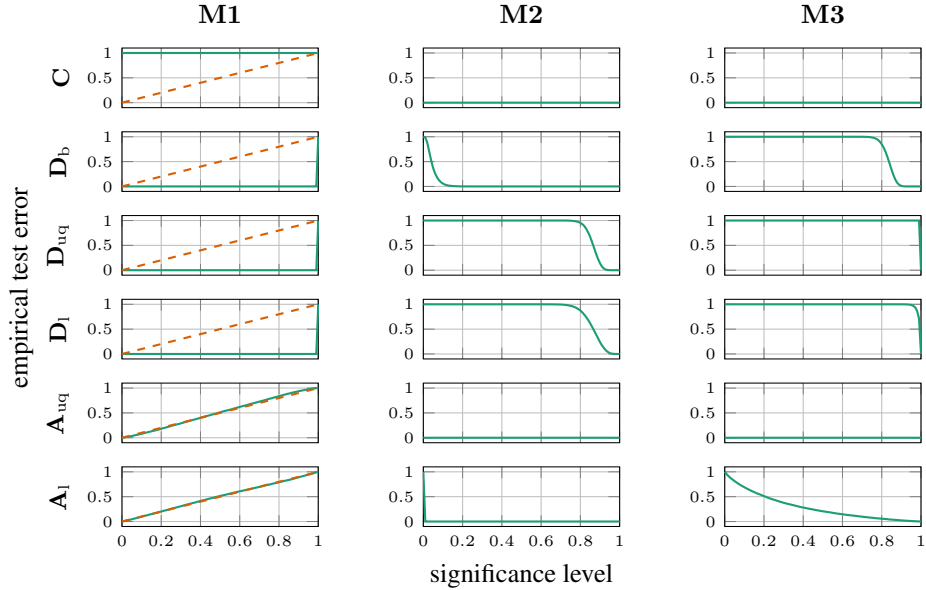}
    \caption{Empirical test error versus significance level for different
      bounds and approximations of the p-value, evaluated on $10^4$ data sets
      that are randomly sampled from the generative models $\mathbf{M1}$,
      $\mathbf{M2}$, and $\mathbf{M3}$. The dashed line highlights the diagonal
      of the unit square.}
    \label{fig:pvalues_comparison}
  \end{center}
\end{figure}

\subsection{Additional experiments}

In \cref{sec:additional_estimates} we provide results for varying
number of classes and a non-standard $\ECE$ estimator with data-dependent bins.
We observe negative and positive bias for both $\ECE$ estimators, whereas
$\biasedestimator$ is guaranteed to be biased upwards. The bias of $\widehat{\ECE}$
becomes more prominent with increasing number of classes, showing high calibration
error estimates even for calibrated models. The $\squaredkernelmeasure$
estimators are not affected by the number of classes in the same way. In some
experiments with 100 and 1000 classes, however, the distribution of
$\linearestimator$ shows multi-modality.

The empirical evaluations in \cref{sec:additional_tests} indicate
that the reliability of the p-value approximations based on $\widehat{\ECE}$
decreases with increasing number of classes, both for the estimator with bins
of uniform size and, to a smaller extent, for the estimator with data-dependent bins.

The considered calibration measures depend only on the predictions and the
true labels, not on how these predictions are computed. Hence directly
specifying the predictions allows a clean numerical evaluation and enables
comparisons of the estimates with the true calibration error. Nevertheless, we
provide a more practical evaluation of calibration for several modern neural
networks in \cref{sec:neural_networks}.

\section{Conclusion}

We have presented a unified framework for quantifying the calibration error of
probabilistic classifiers. The framework encompasses several existing error
measures and enables the formulation of a new kernel-based measure. We have
derived unbiased and consistent estimators of the kernel-based error measures,
which are properties not readily enjoyed by the more common and less tractable
$\ECE$. The impact of the kernel and its hyperparameters on the estimators is an
important question for future research. We have refrained from investigating it
in this paper, since it deserves a more exhaustive study than, what we felt,
would have been possible in this work.

The calibration error estimators can be viewed as test statistics. This confers
probabilistic interpretability to the error measures. Specifically, we can
compute well-founded bounds and approximations of the p-value for the observed
error estimates under the null hypothesis that the model is calibrated. We
have derived distribution-free bounds and asymptotic approximations for the
estimators of the proposed kernel-based error measure, that allow reliable
calibration tests in contrast to previously proposed tests based on consistency
resampling with the standard estimator of the $\ECE$.

\subsubsection*{Acknowledgements}

We thank the reviewers for all the constructive feedback on our paper.
This research is financially supported by the Swedish Research Council via the
projects \emph{Learning of Large-Scale Probabilistic Dynamical Models} (contract
number: 2016-04278) and \emph{Counterfactual Prediction Methods for Heterogeneous
  Populations} (contract number: 2018-05040), by the Swedish Foundation for
Strategic Research via the project \emph{Probabilistic Modeling and Inference
  for Machine Learning} (contract number: ICA16-0015), and by the Wallenberg AI,
Autonomous Systems and Software Program (WASP) funded by the Knut and Alice
Wallenberg Foundation.

\bibliography{references}

\begin{thebibliography}{34}
\providecommand{\natexlab}[1]{#1}
\providecommand{\url}[1]{\texttt{#1}}
\expandafter\ifx\csname urlstyle\endcsname\relax
  \providecommand{\doi}[1]{doi: #1}\else
  \providecommand{\doi}{doi: \begingroup \urlstyle{rm}\Url}\fi

\bibitem[Arcones and Gin{\'e}(1992)]{arcones92_boots_u_v_statis}
M.~A. Arcones and E.~Gin{\'e}.
\newblock On the bootstrap of {$U$} and {$V$} statistics.
\newblock \emph{The Annals of Statistics}, 20\penalty0 (2):\penalty0 655--674,
  1992.

\bibitem[Aronszajn(1950)]{aronszajn50_theor_reprod_kernel}
N.~Aronszajn.
\newblock Theory of reproducing kernels.
\newblock \emph{Transactions of the American Mathematical Society}, 68\penalty0
  (3):\penalty0 337--337, 3 1950.

\bibitem[Bartlett and Mendelson(2002)]{bartlett02_radem_gauss_compl}
P.~L. Bartlett and S.~Mendelson.
\newblock Rademacher and {Gaussian} complexities: Risk bounds and structural
  results.
\newblock \emph{Journal of Machine Learning Research}, 3:\penalty0 463--482,
  2002.

\bibitem[Berlinet and
  Thomas-Agnan(2004)]{berlinet04_reprod_kernel_hilber_spaces_probab_statis}
A.~Berlinet and C.~Thomas-Agnan.
\newblock \emph{Reproducing Kernel Hilbert Spaces in Probability and
  Statistics}.
\newblock Springer {US}, 2004.

\bibitem[Br\"{o}cker(2009)]{broecker09_reliab_suffic_decom_proper_scores}
J.~Br\"{o}cker.
\newblock Reliability, sufficiency, and the decomposition of proper scores.
\newblock \emph{Quarterly Journal of the Royal Meteorological Society},
  135\penalty0 (643):\penalty0 1512--1519, 7 2009.

\bibitem[Br{\"o}cker(2011)]{broecker11_estim_reliab_resol_probab_forec}
J.~Br{\"o}cker.
\newblock Estimating reliability and resolution of probability forecasts
  through decomposition of the empirical score.
\newblock \emph{Climate Dynamics}, 39\penalty0 (3--4):\penalty0 655--667, 2011.

\bibitem[Br{\"o}cker and Smith(2007)]{broecker07_increas_reliab_reliab_diagr}
J.~Br{\"o}cker and L.~A. Smith.
\newblock Increasing the reliability of reliability diagrams.
\newblock \emph{Weather and Forecasting}, 22\penalty0 (3):\penalty0 651--661, 6
  2007.

\bibitem[Caponnetto et~al.(2008)Caponnetto, Micchelli, Pontil, and
  Ying]{caponnetto08_univer_multi_task_kernel}
A.~Caponnetto, C.~A. Micchelli, M.~Pontil, and Y.~Ying.
\newblock Universal multi-task kernels.
\newblock \emph{Journal of Machine Learning Research}, 9:\penalty0 1615--1646,
  6 2008.

\bibitem[Carmeli et~al.(2010)Carmeli, De~Vito, Toigo, and
  Umanit{\`a}]{carmeli10_vector_valued_reprod_kernel_hilber_spaces_univer}
C.~Carmeli, E.~De~Vito, A.~Toigo, and V.~Umanit{\`a}.
\newblock Vector valued reproducing kernel {Hilbert} spaces and universality.
\newblock \emph{Analysis and Applications}, 08\penalty0 (01):\penalty0 19--61,
  01 2010.

\bibitem[Christmann and Steinwart(2008)]{christmann08_suppor_vector_machin}
A.~Christmann and I.~Steinwart.
\newblock \emph{Support Vector Machines}.
\newblock Information Science and Statistics. Springer New York, 2008.

\bibitem[DeGroot and Fienberg(1983)]{degroot83_compar_evaluat_forec}
M.~H. DeGroot and S.~E. Fienberg.
\newblock The comparison and evaluation of forecasters.
\newblock \emph{The Statistician}, 32\penalty0 (1/2):\penalty0 12, 3 1983.

\bibitem[Ferro and Fricker(2012)]{ferro12_bias_correc_decom_brier_score}
C.~A.~T. Ferro and T.~E. Fricker.
\newblock A bias-corrected decomposition of the {Brier} score.
\newblock \emph{Quarterly Journal of the Royal Meteorological Society},
  138\penalty0 (668):\penalty0 1954--1960, 2012.

\bibitem[Gretton et~al.(2012)Gretton, Borgwardt, Rasch, Sch\"{o}lkopf, and
  Smola]{gretton12_kernel_two_sampl_test}
A.~Gretton, K.~M. Borgwardt, M.~J. Rasch, B.~Sch\"{o}lkopf, and A.~Smola.
\newblock A kernel two-sample test.
\newblock \emph{Journal of Machine Learning Research}, 13:\penalty0 723--773, 3
  2012.

\bibitem[Guo et~al.(2017)Guo, Pleiss, Sun, and
  Weinberger]{guo17_calib_moder_neural_networ}
C.~Guo, G.~Pleiss, Y.~Sun, and K.~Q. Weinberger.
\newblock On calibration of modern neural networks.
\newblock In \emph{Proceedings of the 34th International Conference on Machine
  Learning}, volume~70 of \emph{Proceedings of Machine Learning Research},
  pages 1321--1330, 2017.

\bibitem[Hoeffding(1948)]{hoeffding48_class_statis_with_asymp_normal_distr}
W.~Hoeffding.
\newblock A class of statistics with asymptotically normal distribution.
\newblock \emph{The Annals of Mathematical Statistics}, 19\penalty0
  (3):\penalty0 293--325, 9 1948.

\bibitem[Hoeffding(1963)]{hoeffding63_probab_inequal_sums_bound_random_variab}
W.~Hoeffding.
\newblock Probability inequalities for sums of bounded random variables.
\newblock \emph{Journal of the American Statistical Association}, 58\penalty0
  (301):\penalty0 13--30, 1963.

\bibitem[Kiureghian and Ditlevsen(2009)]{kiureghian09_aleat_or_epist}
A.~D. Kiureghian and O.~Ditlevsen.
\newblock Aleatory or epistemic? {Does} it matter?
\newblock \emph{Structural Safety}, 31\penalty0 (2):\penalty0 105--112, 3 2009.

\bibitem[Krizhevsky(2009)]{krizhevsky09_learn_multip_layer_featur_from_tiny_images}
A.~Krizhevsky.
\newblock Learning multiple layers of features from tiny images.
\newblock 2009.

\bibitem[Kumar et~al.(2018)Kumar, Sarawagi, and
  Jain]{kumar18_train_calib_measur_neural_networ}
A.~Kumar, S.~Sarawagi, and U.~Jain.
\newblock Trainable calibration measures for neural networks from kernel mean
  embeddings.
\newblock In \emph{Proceedings of the 35th International Conference on Machine
  Learning}, volume~80 of \emph{Proceedings of Machine Learning Research},
  pages 2805--2814, 2018.

\bibitem[Kumar et~al.(2019)Kumar, Liang, and Ma]{kumar19_verif_uncer_calib}
A.~Kumar, P.~Liang, and T.~Ma.
\newblock Verified uncertainty calibration.
\newblock In \emph{Advances in Neural Information Processing Systems 32}. 2019.

\bibitem[Micchelli and Pontil(2005)]{micchelli05_learn_vector_valued_funct}
C.~A. Micchelli and M.~Pontil.
\newblock On learning vector-valued functions.
\newblock \emph{Neural Computation}, 17\penalty0 (1):\penalty0 177--204, 2005.

\bibitem[Murphy and
  Winkler(1977)]{murphy77_reliab_subjec_probab_forec_precip_temper}
A.~H. Murphy and R.~L. Winkler.
\newblock Reliability of subjective probability forecasts of precipitation and
  temperature.
\newblock \emph{Applied Statistics}, 26\penalty0 (1):\penalty0 41, 1977.

\bibitem[Naeini et~al.(2015)Naeini, Cooper, and
  Hauskrecht]{naeini15_obtain_bayes}
M.~P. Naeini, G.~Cooper, and M.~Hauskrecht.
\newblock Obtaining well calibrated probabilities using {Bayesian} binning.
\newblock In \emph{Twenty-Ninth {AAAI} Conference on Artificial Intelligence},
  2015.

\bibitem[Pastell(2017)]{pastell17_weave}
M.~Pastell.
\newblock Weave.jl: Scientific reports using {Julia}.
\newblock \emph{The Journal of Open Source Software}, 2\penalty0 (11):\penalty0
  204, Mar. 2017.

\bibitem[Phan(2019)]{phan19_pytor_cifar}
H.~Phan.
\newblock {PyTorch-CIFAR10}, 2019.
\newblock URL \url{https://github.com/huyvnphan/PyTorch-CIFAR10/}.

\bibitem[Rudin(1986)]{rudin86_real}
W.~Rudin.
\newblock \emph{Real and complex analysis}.
\newblock McGraw-Hill, New York, 3 edition, 1986.

\bibitem[Serfling(1980)]{serfling80_approx_theor_mathem_statis}
R.~J. Serfling, editor.
\newblock \emph{Approximation Theorems of Mathematical Statistics}.
\newblock John Wiley {\&} Sons, Inc., 11 1980.

\bibitem[Vaicenavicius et~al.(2019)Vaicenavicius, Widmann, Andersson, Lindsten,
  Roll, and Sch\"{o}n]{vaicenavicius19_evaluat}
J.~Vaicenavicius, D.~Widmann, C.~Andersson, F.~Lindsten, J.~Roll, and T.~B.
  Sch\"{o}n.
\newblock Evaluating model calibration in classification.
\newblock In \emph{Proceedings of Machine Learning Research}, volume~89 of
  \emph{Proceedings of Machine Learning Research}, pages 3459--3467, 2019.

\bibitem[van~der Vaart(1998)]{vaart98_asymp_statis}
A.~W. van~der Vaart.
\newblock \emph{Asymptotic Statistics}.
\newblock Cambridge University Press, 1998.

\bibitem[van~der Vaart and Wellner(1996)]{vaart96_weak_conver_empir_proces}
A.~W. van~der Vaart and J.~A. Wellner.
\newblock \emph{Weak Convergence and Empirical Processes}.
\newblock Springer New York, 1996.

\bibitem[Widmann(2019{\natexlab{a}})]{widmann19_calib_errors}
D.~Widmann.
\newblock {devmotion/CalibrationErrors.jl}: v0.1.0, Sept. 2019{\natexlab{a}}.
\newblock URL \url{https://doi.org/10.5281/zenodo.3457945}.

\bibitem[Widmann(2019{\natexlab{b}})]{widmann19_calib_tests}
D.~Widmann.
\newblock {devmotion/CalibrationTests.jl}: v0.1.0, Oct. 2019{\natexlab{b}}.
\newblock URL \url{https://doi.org/10.5281/zenodo.3514933}.

\bibitem[Widmann(2019{\natexlab{c}})]{widmann19_consis_resampling}
D.~Widmann.
\newblock {devmotion/ConsistencyResampling.jl}: v0.2.0, May 2019{\natexlab{c}}.
\newblock URL \url{https://doi.org/10.5281/zenodo.3232854}.

\bibitem[Zadrozny and Elkan(2002)]{zadrozny02_trans}
B.~Zadrozny and C.~Elkan.
\newblock Transforming classifier scores into accurate multiclass probability
  estimates.
\newblock In \emph{Proceedings of the eighth {ACM} {SIGKDD} international
  conference on Knowledge discovery and data mining - {KDD}}. {ACM} Press,
  2002.

\end{thebibliography}
\bibliographystyle{abbrvnat}

\clearpage

\appendix
\numberwithin{equation}{section}
\renewcommand*{\thetheorem}{\thesection.\arabic{theorem}}
\renewcommand*{\thelemma}{\thesection.\arabic{lemma}}
\renewcommand*{\theproposition}{\thesection.\arabic{proposition}}
\renewcommand*{\thecorollary}{\thesection.\arabic{corollary}}
\renewcommand*{\thedefinition}{\thesection.\arabic{definition}}
\renewcommand*{\theexample}{\thesection.\arabic{example}}
\renewcommand*{\theremark}{\thesection.\arabic{remark}}
\makeatletter
\@addtoreset{theorem}{section}
\@addtoreset{lemma}{section}
\@addtoreset{proposition}{section}
\@addtoreset{corollary}{section}
\@addtoreset{definition}{section}
\@addtoreset{example}{section}
\@addtoreset{remark}{section}
\makeatother

\section{Notation}

We introduce some additional notation to keep the following discussion concise.

Let $U$ be a compact metric space and $V$ be a Hilbert space. In this paper we
only consider $U = [0,1]$ and $U = \Delta^m$. By $\mathcal{C}(U, V)$ we denote
the space of continuous functions $f \colon U \to V$.

For $1 \leq p < \infty$, $L^p(U, \mu; V)$ is the space of (equivalence classes
of) measurable functions $f \colon U \to V$ such that $\|f\|^p$ is
$\mu$-integrable, equipped with norm
$\|f\|_{\mu,p} = {(\int_{U} \|f(x)\|^p \,\mu(\mathrm{d}x))}^{1/p}$; for
$p = \infty$, $L^{\infty}(U, \mu; V)$ is the space of $\mu$-essentially bounded
measurable functions $f \colon U \to V$ with norm
$\|f\|_{\mu,\infty} = \mu\text{-}\esssup_{x \in \mathcal{X}} \|f(x)\|$. We denote the
closed unit ball of the space $L^p(U, \mu; V)$ by
$K^p(U, \mu; V) \coloneqq \{f \in L^p(U, \mu; V) \colon \|f\|_{\mu,p} \leq 1\}$.

If the norm $\|.\|_V$ on $V$ is not clear from the context we indicate it by
writing $L^p(U, \mu; V, \|.\|_V)$, $K^p(U, \mu; V, \|.\|_V)$, and
$\|.\|_{\mu,p;\|.\|_V}$. If $V \subset \mathbb{R}^d$, all possible norms
$\|.\|_V$ are equivalent and hence we choose $\|.\|_p$ on $V$ to simplify our
calculations, if not stated otherwise. Moreover, if $V \subset \mathbb{R}^d$
and $\|.\|_V = \|.\|_q$ for some $1 \leq q \leq \infty$, for convenience we
write $\|.\|_{\mu,p;q} = \|.\|_{\mu,p;\|.\|_V}$.

Let $W$ be another Hilbert space. Then $\mathcal{L}(V, W)$ denotes the space of
bounded linear operators $T \colon V \to W$; if $W = V$, we write
$\mathcal{L}(V) \coloneqq \mathcal{L}(V, V)$. The induced operator norm on
$\mathcal{L}(V, W)$ is defined by
\begin{equation*}
  \begin{split}
     \|T\|_{\|.\|_V; \|.\|_W} &= \inf \{ c \geq 0 \colon \|Tv\|_W \leq c \|v\|_V \text{for all } v \in V\} \\
     &= \sup_{v \in V \colon \|v\|_V \leq 1} \|Tv\|_W.
  \end{split}
\end{equation*}
If $W = V$, we write $\|.\|_{\|.\|_V} = \|.\|_{\|.\|_V; \|.\|_V}$. Moreover, for
convenience if $V \subset \mathbb{R}^d$ and $\|.\|_V = \|.\|_p$ for some
$1 \leq p \leq \infty$, we use the notation
$\|.\|_{p;\|.\|_W} = \|.\|_{\|.\|_V; \|.\|_W}$. In the same way, if
$W \subset \mathbb{R}^d$ and $\|.\|_W = \|.\|_q$ for some
$1 \leq q \leq \infty$, we write $\|.\|_{\|.\|_V; q} = \|.\|_{\|.\|_V;\|.\|_W}$.

By $\mathcal{B}(T)$ we denote the Borel $\sigma$-algebra of a topological space
$T$.

We write $\mu_A$ for the distribution of a random variable $A$, i.e., the
pushforward measure it induces, if $A$ is defined on a probability space with
probability measure $\mu$.

\section{Probabilistic setting}

Let $(\Omega, \mathcal{A}, \mu)$ be a probability space. Let $m \in \mathbb{N}$
and define the random variables
$X \colon (\Omega, \mathcal{A}) \to (\mathcal{X}, \Sigma_X)$ and
$Y \colon (\Omega, \mathcal{A}) \to (\{1,\ldots,m\}, 2^{\{1,\ldots,m\}})$, such that
$\Sigma_X$ contains all singletons. We denote a version of the regular
conditional distribution of $Y$ given $X$ by $\mu_{Y|X}(\cdot|x)$ for all
$x \in \mathcal{X}$.

We consider the problem of learning a measurable function
$g \colon (\mathcal{X}, \Sigma_X) \to (\Delta^m, \mathcal{B}(\Delta^m))$ that
returns the prediction $g_y(x)$ of $\mu_{Y|X}(\{y\}|x)$ for all $x \in \mathcal{X}$
and $y \in \{1,\ldots,m\}$. We define the random variable
$G \colon (\Omega, \mathcal{A}) \to (\Delta^m, \mathcal{B}(\Delta^m))$ as $G \coloneqq g(X)$.

In the same way as above, we denote a version of the regular conditional
distribution of $Y$ given $G$ by $\mu_{Y|G}(\cdot|t)$ for all $t \in \Delta^m$.
The function
$\delta \colon (\Delta^m, \mathcal{B}(\Delta^m)) \to (\mathbb{R}^m, \mathcal{B}(\mathbb{R}^m))$,
given by
\begin{equation*}
  \delta(t) \coloneqq \begin{pmatrix}
    \mu_{Y|G}(1|t) - t \\
    \vdots \\
    \mu_{Y|G}(m|t) - t
  \end{pmatrix}
\end{equation*}
for all $t \in \Delta^m$, gives rise to another random variable
$\Delta \colon (\Omega, \mathcal{A}) \to (\mathbb{R}^m, \mathcal{B}(\mathbb{R}^m))$
that is defined by $\Delta \coloneqq \delta(G)$.

Using the newly introduced mathematical notation, we can reformulate strong
calibration in a compact way. A model $g$ is calibrated in the strong sense if
$\mu_{Y|G}(\cdot|G) = G$ almost surely, or equivalently if $\Delta = 0$ almost surely.

\section{Calibration error}

\begin{definition}[Calibration error]
  Let $\mathcal{F} \subset L^1(\Delta^m, \mu_G; \mathbb{R}^m)$ be non-empty.
  Then the calibration error $\measure$ of model $g$ with respect to class
  $\mathcal{F}$ is
  \begin{equation*}
    \measure[\mathcal{F}, g]  \coloneqq \sup_{f \in \mathcal{F}} \Expect\left[\langle \Delta, f(G) \rangle_{\mathbb{R}^m}\right].
  \end{equation*}
\end{definition}

\begin{remark}\label{remark:exists}
  Note that $\|\Delta\|_{\infty} \leq 1$ almost surely, and thus
  $\|\delta\|_{\mu_G,\infty} \leq 1$. Hence by Hölder's inequality for all
  $f \in \mathcal{F}$ we have
  \begin{equation*}
    |\Expect[\langle \Delta, f(G) \rangle_{\mathbb{R}^m}]| \leq \Expect[|\langle \Delta, f(G) \rangle_{\mathbb{R}^m}|] \leq \|\delta\|_{\mu_G,\infty} \|f\|_{\mu_G,1} \leq \|f\|_{\mu_G,1} < \infty.
  \end{equation*}
  However, it is still possible that $\measure[\mathcal{F}, g] = \infty$.
\end{remark}

\begin{remark}\label{remark:positive}
  If $\mathcal{F}$ is symmetric in the sense that $f \in \mathcal{F}$ implies
  $-f \in \mathcal{F}$, then $\measure[\mathcal{F}, g] \geq 0$.
\end{remark}

The measure highly depends on the choice of $\mathcal{F}$ but strong calibration
always implies a calibration error of zero.

\begin{theorem}[Strong calibration implies zero error]\label{thm:ce_zero}
  Let $\mathcal{F} \subset L^1(\Delta^m,\mu_G;\mathbb{R})$. If model $g$ is
  calibrated in the strong sense, then $\measure[\mathcal{F}, g] = 0$.
\end{theorem}

\begin{proof}
  If model $g$ is calibrated in the strong sense, then $\Delta = 0$ almost
  surely. Hence for all $f \in \mathcal{F}$ we have
  $\Expect[\langle \Delta, f(G)\rangle_{\mathbb{R}}] = 0$, which implies
  $\measure[\mathcal{F}, g] = 0$.
\end{proof}

Of course, the converse statement is not true in general. A similar result as
above shows that the class of continuous functions, albeit too large and
impractical, allows to identify calibrated models.

\begin{theorem}\label{thm:ce_continuous}
  Let $\mathcal{F} = \mathcal{C}(\Delta^m, \mathbb{R}^m)$. Then
  $\measure[\mathcal{F}, g] = 0$ if and only if model $g$ is calibrated in the
  strong sense.
\end{theorem}

\begin{proof}
  Note that $\mathcal{F}$ is well defined since
  $\mathcal{F} \subset L^1(\Delta^m, \mu_G; \mathbb{R}^m)$.

  If model $g$ is calibrated in the strong sense, then
  $\measure[\mathcal{F}, g] = 0$ by \cref{thm:ce_zero}.

  If model $g$ is not calibrated in the strong sense, then $\Delta = 0$ does not
  hold almost surely. In particular, there exists $s \in \{-1,1\}^m$ such that
  $\langle \Delta, s\rangle_{\mathbb{R}^m} \leq 0$ does not hold almost surely.
  Define the function $f_s \colon \Delta^m \to \mathbb{R}^m$ by
  $f_s \coloneqq \langle \delta(\cdot), s \rangle_{\mathbb{R}^m}$ and let
  $A_s \coloneqq f_s^{-1}((0, \infty))$. Then $A_s \in \mathcal{B}(\Delta^m)$
  since $f_s$ is Borel measurable, and $\mu_G(A_s) > 0$. Hence we know that
  \begin{equation*}
    \alpha_s \coloneqq \Expect[\langle \Delta, s\mathbbm{1}_{A_s}(G) \rangle_{\mathbb{R}^m}] > 0.
  \end{equation*}

  Since $\mu_G$ is a Borel probability measure on a compact metric space, it is
  regular and hence there exist a compact set $K$ and an open set $U$ such that
  $K \subset A_s \subset U$ and $\mu_G(U \setminus K) < \alpha_s / 4$
  \citep[Theorem~2.17]{rudin86_real}. Thus by Urysohn's lemma applied to the
  closed sets $K$ and $U^c$, there exists a continuous function
  $h \in \mathcal{C}(\Delta^m, \mathbb{R})$ such that
  $\mathbbm{1}_K \leq h \leq \mathbbm{1}_U$. By defining
  $f = sh \in \mathcal{C}(\Delta^m, \mathbb{R}^m)$ and applying Hölder's
  inequality we obtain
  \begin{equation*}
    \begin{split}
      \Expect[\langle \Delta, f(G) \rangle_{\mathbb{R}^m}] &= \Expect[\langle \Delta, s\mathbbm{1}_{A_s}(G) \rangle_{\mathbb{R}^m}] \\
      &\quad + \Expect[\langle \Delta, f(G) - s\mathbbm{1}_{A_s}(G) \rangle_{\mathbb{R}^m}] \\
      &\geq \alpha_s - |\Expect[(h(G) - \mathbbm{1}_{A_s}(G)) \langle \Delta, s \rangle_{\mathbb{R}^m}]| \\
      &\geq \alpha_s - \Expect[|h(G) - \mathbbm{1}_{A_s}(G)| |\langle \Delta, s \rangle_{\mathbb{R}^m}|] \\
      &\geq \alpha_s - \Expect[\mathbbm{1}_{U \setminus K}(G) \|\Delta\|_1 \|s\|_{\infty}] \geq \alpha_s - 2 \mu_G(U\setminus K) \\
      &> \alpha_s - \alpha_s/2 = \alpha_s/2 > 0.
    \end{split}
  \end{equation*}
  This implies $\measure[\mathcal{F}, g] > 0$.
\end{proof}

\section{Reproducing kernel Hilbert spaces of vector-valued functions on the probability simplex}\label{app:rkhs}

\begin{definition}[{{\citet[Definition~1]{micchelli05_learn_vector_valued_funct}}}]\label{def:rkhs}
  Let $\mathcal{H}$ be a Hilbert space of vector-valued functions
  $f \colon \Delta^m \to \mathbb{R}^m$ with inner product
  $\langle ., . \rangle_{\mathcal{H}}$. We call $\mathcal{H}$ a reproducing
  kernel Hilbert space (RKHS), if for all $t \in \Delta^m$ and
  $u \in \mathbb{R}^m$ the functional $E_{t,u} \colon \mathcal{H} \to \mathbb{R}$,
  $E_{t,u} f \coloneqq \langle u, f(t) \rangle_{\mathbb{R}^m}$, is a bounded (or
  equivalently continuous) linear operator.
\end{definition}

Riesz's representation theorem ensures that there exists a unique function
$k \colon \Delta^m \times \Delta^m \to \mathbb{R}^{m \times m}$ such that for
all $t \in \Delta^m$ the function $k(\cdot, t)$ is a linear map from $\Delta^m$
to $\mathcal{H}$ and for all $u \in \mathbb{R}^m$ it satisfies the so-called
reproducing property
\begin{equation}\label{eq:reproducing_property}
  \langle u, f(t) \rangle_{\mathbb{R}^m} = E_{t,u}f = \langle k(\cdot, t) u, f\rangle_{\mathcal{H}}.
\end{equation}
It can be shown that function $k$ is self-adjoint\footnote{Let $U$, $V$ be two
  Hilbert spaces. Then the adjoint of a linear operator
  $T \in \mathcal{L}(U, V)$ is the linear operator $T^* \in \mathcal{L}(V, U)$
  such that for all $u \in U$, $v \in V$
  $\langle Tu, v\rangle_{V} = \langle u, T^*v \rangle_U$.} and positive
semi-definite, and hence a kernel according to \cref{def:kernel}
\citep[Proposition~1]{micchelli05_learn_vector_valued_funct}. Similar to the
scalar-valued case, conversely by the Moore-Aronszaijn theorem
\citep{aronszajn50_theor_reprod_kernel} to every kernel
$k \colon \Delta^m \times \Delta^m \to \mathcal{L}(\mathbb{R}^m)$ there exists a
unique RKHS $\mathcal{H} \subset {(\mathbb{R}^m)}^{\Delta^m}$ with $k$ as
reproducing kernel \citep[][Theorem~1]{micchelli05_learn_vector_valued_funct}.

Other useful properties are summarized in \cref{lemma:rkhs_properties}.
\Citet{micchelli05_learn_vector_valued_funct} considered only the Euclidean norm
on $\mathbb{R}^m$, corresponding to $p = q = 2$ in our statement. For
convenience we use the notation
$\|.\|_{p; \mathcal{H}} = \|.\|_{\|.\|_p; \|.\|_{\mathcal{H}}}$.

\begin{lemma}[{{\citet[Proposition~1]{micchelli05_learn_vector_valued_funct}}}]\label{lemma:rkhs_properties}
  Let $\mathcal{H} \subset {(\mathbb{R}^m)}^{\Delta^m}$ be a RKHS with kernel
  $k \colon \Delta^m \times \Delta^m \to \mathcal{L}(\mathbb{R}^m)$. Let
  $1\leq p, q \leq \infty$ with Hölder conjugates $p'$ and $q'$, respectively.

  \begin{enumerate}
  \item For all $t \in \Delta^m$
    \begin{equation}\label{eq:rkhs_kernel}
      \|k(\cdot, t)\|_{p;\mathcal{H}} = {\|k(t,t)\|}^{1/2}_{p;p'}.
    \end{equation}
  \item For all $s, t \in \Delta^m$
    \begin{equation}\label{eq:rkhs_st_bound}
      \|k(s,t)\|_{p;q} \leq \|k(s,s)\|_{q';q}^{1/2} \|k(t,t)\|_{p;p'}^{1/2}.
    \end{equation}
  \item For all $f \in \mathcal{H}$ and $t \in \Delta^m$
    \begin{equation}\label{eq:rkhs_bound}
      \|f(t)\|_p \leq \|f\|_{\mathcal{H}} \|k(\cdot, t)\|_{p';\mathcal{H}} = \|f\|_{\mathcal{H}} {\|k(t,t)\|}^{1/2}_{p';p}.
    \end{equation}
  \end{enumerate}
\end{lemma}

\begin{proof}
  Let $t \in \Delta^m$. From the reproducing property, Hölder's inequality, and
  the definition of the operator norm, we obtain for all $u \in \mathbb{R}^m$
  \begin{equation*}
    \|k(\cdot, t)u\|_{\mathcal{H}}^2 = \langle k(\cdot, t) u, k(\cdot, t) u\rangle_{\mathcal{H}} = \langle u, k(t, t) u \rangle_{\mathbb{R}^m} \leq \|u\|_p \|k(t, t) u\|_{p'} \leq {\|u\|}^2_p \|k(t,t)\|_{p;p'},
  \end{equation*}
  which implies that
  \begin{equation}\label{eq:rkhs_kernel_1}
    \|k(\cdot, t)\|_{p;\mathcal{H}} = \sup_{u \in \mathbb{R}^m \setminus \{0\}} \frac{\|k(\cdot, t)u\|_{\mathcal{H}}}{\|u\|_p} \leq \|k(t,t)\|_{p;p'}^{1/2}.
  \end{equation}
  On the other hand, for all $u,v \in \mathbb{R}^m$ it follows from the
  reproducing property, the Cauchy-Schwarz inequality, and the definition of the
  operator norm that
  \begin{equation*}
    \langle u, k(t,t) v \rangle_{\mathbb{R}^m} = \langle k(\cdot, t) u, k(\cdot, t) v \rangle_{\mathcal{H}} \leq \|k(\cdot, t) u\|_{\mathcal{H}} \|k(\cdot, t) v\|_{\mathcal{H}} \leq \|u\|_p \|v\|_p \|k(\cdot, t)\|^2_{p;\mathcal{H}}.
  \end{equation*}
  Since the $\ell_p$-norm is the dual norm of the $\ell_{p'}$-norm, it follows
  that
  \begin{equation*}
    \|k(t,t)v\|_{p'} = \sup_{u \in \mathbb{R}^m \colon \|u\|_p \leq 1} \langle u, k(t,t)v\rangle_{\mathbb{R}^m} \leq \|v\|_p \|k(\cdot, t)\|^2_{p;\mathcal{H}},
  \end{equation*}
  which implies that
  \begin{equation}\label{eq:rkhs_kernel_2}
    \|k(t,t)\|_{p;p'} = \sup_{v \in \mathbb{R}^m \setminus \{0\}} \frac{\|k(t,t)v\|_{p'}}{\|v\|_p} \leq \|k(\cdot, t)\|^2_{p;\mathcal{H}}.
  \end{equation}
  \Cref{eq:rkhs_kernel} follows from \cref{eq:rkhs_kernel_1,eq:rkhs_kernel_2}.

  Let $s, t \in \Delta^m$. From the reproducing property, the Cauchy-Schwarz
  inequality, and the definition of the operator norm, we get for all
  $u, v \in \mathbb{R}^m$
  \begin{equation*}
    \begin{split}
      \langle u, k(s,t)v\rangle_{\mathbb{R}^m} &\leq \langle k(\cdot, s)u, k(\cdot, t)v\rangle_{\mathcal{H}} \leq \|k(\cdot,s)u\|_{\mathcal{H}} \|k(\cdot, t)v\|_{\mathcal{H}} \\
      &\leq \|u\|_{q'} \|v\|_p \|k(\cdot, s)\|_{q';\mathcal{H}} \|k(\cdot, t)\|_{p;\mathcal{H}}.
    \end{split}
  \end{equation*}
  Thus we obtain
  \begin{equation*}
    \|k(s,t)v\|_q = \sup_{u \in \mathbb{R}^m \colon \|u\|_{q'} \leq 1} \langle u, k(s,t)v \rangle_{\mathbb{R}^m} \leq \|v\|_p \|k(\cdot,s)\|_{q';\mathcal{H}} \|k(\cdot, t)\|_{p;\mathcal{H}},
  \end{equation*}
  which implies
  \begin{equation*}
    \|k(s,t)\|_{p;q} = \sup_{v \in \mathbb{R}^m \setminus \{0\}} \frac{\|k(s,t)v\|_q}{\|v\|_p} \leq \|k(\cdot,s)\|_{q';\mathcal{H}} \|k(\cdot, t)\|_{p;\mathcal{H}}.
  \end{equation*}
  Hence from \cref{eq:rkhs_kernel} we obtain \cref{eq:rkhs_st_bound}.

  For the third statement, let $f \in \mathcal{H}$ and $t \in \Delta^m$. From
  the reproducing property, the Cauchy-Schwarz inequality, and the definition of
  the operator norm, we obtain for all $u \in \mathbb{R}^m$
  \begin{equation*}
    \langle u, f(t) \rangle_{\mathbb{R}^m} = \langle k(\cdot, t) u, f \rangle_{\mathcal{H}} \leq \|k(\cdot, t) u\|_{\mathcal{H}} \|f\|_{\mathcal{H}} \leq \|u\|_{p'} \|k(\cdot, t)\|_{p';\mathcal{H}} \|f\|_{\mathcal{H}}.
  \end{equation*}
  Thus the duality of the $\ell_p$- and the $\ell_{p'}$-norm implies
  \begin{equation*}
    \|f(t)\|_p = \sup_{u \in \mathbb{R}^m \colon \|u\|_{p'} \leq 1} \langle u, f(t)\rangle_{\mathbb{R}^m} \leq \|k(\cdot, t)\|_{p';\mathcal{H}} \|f\|_{\mathcal{H}},
  \end{equation*}
  which together with \cref{eq:rkhs_kernel} yields \cref{eq:rkhs_bound}.
\end{proof}

If $\mu$ is a measure on $\Delta^m$, we define for $1 \leq p, q \leq \infty$
$\|k\|_{\mu,p;q} \coloneqq \|\tilde{k}_q\|_{\mu,p}$ where
$\tilde{k}_q \colon \Delta^m \to \mathbb{R}$ is given by
$\tilde{k}(t) \coloneqq \|k(\cdot,t)\|_{q;\mathcal{H}} = \|k(t,t)\|_{q;q'}^{1/2}$.
We omit the value of $q$ if it is clear from the context or does not matter,
since all norms on $\mathbb{R}^m$ are equivalent.

It is possible to construct certain classes of matrix-valued kernels from
scalar-valued kernels, as the following example shows.

\begin{example}[{{\citet{micchelli05_learn_vector_valued_funct}; \citet[Example~1]{caponnetto08_univer_multi_task_kernel}}}]\label{ex:lifted_kernel}
  For all $i \in \{1, \ldots, n\}$, let
  $k_i \colon \Delta^m \times \Delta^m \to \mathbb{R}$ be a scalar-valued kernel
  and $A_i \in \mathbb{R}^{m \times m}$ be a positive semi-definite matrix. Then
  the function
  \begin{equation}\label{eq:lifted_kernel}
    k \colon \Delta^m \times \Delta^m \to \mathbb{R}^{m \times m}, \qquad k(s, t) \coloneqq \sum_{i=1}^n k_i(s, t) A_i,
  \end{equation}
  is a matrix-valued kernel.
\end{example}

We state a simple result about measurability of functions in the considered
RKHSs. The result can be formulated in a much more general fashion and is
similar to a result by \citet[Lemma~4.24]{christmann08_suppor_vector_machin}.

\begin{lemma}[Measurable RKHS]\label{lemma:rkhs_measurable}
  Let $\mathcal{H} \subset {(\mathbb{R}^m)}^{\Delta^m}$ be a RKHS with kernel
  $k \colon \Delta^m \times \Delta^m \to \mathcal{L}(\mathbb{R}^m)$. Then all
  $f \in \mathcal{H}$ are measurable if and only if
  $k(\cdot, t)u \in {(\mathbb{R}^m)}^{\Delta^m}$ is measurable for all
  $t \in \Delta^m$ and $u \in \mathbb{R}^m$.
\end{lemma}

\begin{proof}
  If all $f \in \mathcal{H}$ are measurable, then $k(\cdot, t)u \in \mathcal{H}$
  is measurable for all $t \in \Delta^m$ and $u \in \mathbb{R}^m$.

  If $k(\cdot, t)u$ is measurable for all $t \in \Delta^m$ and
  $u \in \mathbb{R}^m$, then all functions in
  $\mathcal{H}_0 \coloneqq \Span{\{k(\cdot, t)u \colon t \in \Delta^m, u \in \mathbb{R}^m\}} \subset \mathcal{H}$
  are measurable.

  Let $f \in \mathcal{H}$. Since $\mathcal{H} = \overline{\mathcal{H}}_0$
  \citep[see, e.g.,][]{carmeli10_vector_valued_reprod_kernel_hilber_spaces_univer},
  there exists a sequence $(f_n)_{n \in \mathbb{N}} \subset \mathcal{H}_0$ such
  that $\lim_{n \to \infty} \|f - f_n\|_{\mathcal{H}} = 0$. For all
  $t \in \Delta^m$, since the operator $k^*(\cdot, t)$ is continuous, by the
  reproducing property we obtain $\lim_{n \to \infty} f_n(t) = f(t)$. Thus $f$
  is measurable.
\end{proof}

By definition
\citep[see, e.g.,][Definition~1]{carmeli10_vector_valued_reprod_kernel_hilber_spaces_univer},
a RKHS with a continuous kernel is a subspace of the space of continuous
functions. The following equivalent formulation is an immediate consequence of
the result by \citet{carmeli10_vector_valued_reprod_kernel_hilber_spaces_univer}.

\begin{corollary}[{{\citet[Proposition~1]{carmeli10_vector_valued_reprod_kernel_hilber_spaces_univer}}}]
  A kernel $k \colon \Delta^m \times \Delta^m \to \mathbb{R}^m$ is continuous if
  for all $t \in \Delta^m$ $t \mapsto \|k(t, t)\|$ is bounded and for all
  $t \in \Delta^m$ and $u \in \mathbb{R}^m$ $k(\cdot, t) u$ is a continuous
  function from $\Delta^m$ to $\mathbb{R}^m$.
\end{corollary}

An important class of continuous kernels are so-called universal kernels, for
which the corresponding RKHS is a dense subset of the space of continuous
functions with respect to the uniform norm. A result by
\citet{caponnetto08_univer_multi_task_kernel} shows under what assumptions
matrix-valued kernels of the form in \cref{ex:lifted_kernel} are universal.

\begin{lemma}[{{\citet[Theorem~14]{caponnetto08_univer_multi_task_kernel}}}]\label{lemma:lifted_kernel}
  For all $i \in \{1, \ldots, n\}$, let
  $k_i \colon \Delta^m \times \Delta^m \to \mathbb{R}$ be a universal
  scalar-valued kernel and $A_i \in \mathbb{R}^{m \times m}$ be a positive
  semi-definite matrix. Then the matrix-valued kernel defined in
  \cref{eq:lifted_kernel} is universal if and only if $\sum_{i=1}^n A_i$ is
  positive definite.
\end{lemma}

\section{Kernel calibration error}

The one-to-one correspondence between matrix-valued kernels and RKHSs of
vector-valued functions motivates the introduction of the kernel calibration
error ($\kernelmeasure$) in \cref{def:kce}. For certain kernels we are able to
identify strongly calibrated models.

\begin{theorem}\label{thm:ce_rkhs}
  Let $k \colon \Delta^m \times \Delta^m \to \mathcal{L}(\mathbb{R}^m)$ be a
  universal continuous kernel. Then $\kernelmeasure[k, g] = 0$ if and only if
  model $g$ is calibrated in the strong sense.
\end{theorem}

\begin{proof}
  Let $\mathcal{F}$ be the unit ball in the RKHS
  $\mathcal{H} \subset {(\mathbb{R}^m)}^{\Delta^m}$ corresponding to kernel $k$.
  Since kernel $k$ is continuous, by definition
  $\mathcal{H} \subset \mathcal{C}(\Delta^m, \mathbb{R}^m)$
  \citep[Definition~1]{carmeli10_vector_valued_reprod_kernel_hilber_spaces_univer}.
  Thus $\mathcal{F}$ is well defined since
  $\mathcal{F} \subset \mathcal{C}(\Delta^m, \mathbb{R}^m) \subset L^1(\Delta^m, \mu_G; \mathbb{R}^m)$.

  If $g$ is calibrated in the strong sense, it follows from \cref{thm:ce_zero}
  that $\kernelmeasure[k, g] = \measure[\mathcal{F}, g] = 0$.

  Assume that $\kernelmeasure[k, g] = \measure[\mathcal{F}, g] = 0$. This
  implies $\Expect[\langle \Delta, f(G)\rangle_{\mathbb{R}^m}] = 0$ for all
  $f \in \mathcal{F}$. Let $f \in \mathcal{C}(\Delta^m, \mathbb{R}^m)$. Since
  $\mathcal{H}$ is dense in $\mathcal{C}(\Delta^m, \mathbb{R}^m)$
  \citep[Theorem~1]{carmeli10_vector_valued_reprod_kernel_hilber_spaces_univer},
  for all $\epsilon > 0$ there exists a function $h \in \mathcal{H}$ with
  $\|f - h\|_{\infty} < \epsilon / 2$. Define $\tilde{h} \in \mathcal{F}$ by
  $\tilde{h} \coloneqq h / \|h\|_{\mathcal{H}}$ if $\|h\|_{\mathcal{H}} \neq 0$ and
  $\tilde{h} \coloneqq h$ otherwise. Since
  \begin{equation*}
    \Expect[\langle \Delta, h(G)\rangle_{\mathbb{R}^m}] = \|h\|_{\mathcal{H}} \Expect[\langle \Delta, \tilde{h}(G)\rangle_{\mathbb{R}^m}] = 0,
  \end{equation*}
  by Hölder's inequality we obtain
  \begin{equation*}
    \begin{split}
      |\Expect[\langle \Delta, f(G)\rangle_{\mathbb{R}^m}]| &= |\Expect[\langle \Delta, f(G) - h(G)\rangle_{\mathbb{R}^m}] |\\
      &\leq \Expect[|\langle \Delta, f(G) - h(G)\rangle_{\mathbb{R}^m}|] \\
      &\leq \|\delta\|_{\mu_G,1} \|f - h\|_{\mu_G,\infty} \\
      &\leq 2 \|f - h\|_{\infty} < \epsilon.
    \end{split}
  \end{equation*}
  Thus $\measure[\mathcal{C}(\Delta^m, \mathbb{R}^m), g] = 0$, and hence $g$ is
  calibrated in the strong sense by \cref{thm:ce_continuous}.
\end{proof}

Similar to the maximum mean discrepancy~\citep{gretton12_kernel_two_sampl_test},
if we consider functions in a RKHS, we can rewrite the expectation
$\Expect[\langle \Delta, f(G) \rangle_{\mathbb{R}^m}]$ as an inner product in
the Hilbert space.

\begin{lemma}[Existence and uniqueness of embedding]\label{lemma:embedding}
  Let $\mathcal{H} \subset {(\mathbb{R}^m)}^{\Delta^m}$ be a RKHS with kernel
  $k \colon \Delta^m \times \Delta^m \to \mathcal{L}(\mathbb{R}^m)$, and assume
  that $k(\cdot, t) u$ is measurable for all $t \in \Delta^m$ and
  $u \in \mathbb{R}^m$, and $\|k\|_{\mu_G,1} < \infty$.

  Then there exists a unique embedding $\mu_g \in \mathcal{H}$ such that for all
  $f \in \mathcal{H}$
  \begin{equation*}
    \Expect[\langle \Delta, f(G) \rangle_{\mathbb{R}^m}] = \langle f, \mu_g \rangle_{\mathcal{H}}.
  \end{equation*}
  The embedding $\mu_g$ satisfies for all $t \in \Delta^m$ and $y \in \mathbb{R}^m$
  \begin{equation*}
    \langle y, \mu_g(t) \rangle_{\mathbb{R}^m} = \Expect[\langle \Delta, k(G, t) y \rangle_{\mathbb{R}^m}].
  \end{equation*}
\end{lemma}

\begin{proof}
  By \cref{lemma:rkhs_measurable} all $f \in \mathcal{H}$ are measurable.
  Moreover, by \cref{eq:rkhs_bound} for all $f \in \mathcal{H}$ we have
  \begin{equation*}
    \begin{split}
      \int_{\Delta^m} \|f(t)\|_1 \mu_G(\mathrm{d}t) &\leq \|f\|_{\mathcal{H}} \int_{\Delta^m} {\|k(t,t)\|}_{\infty;1}^{1/2} \,\mu_G(\mathrm{d}t) \\
      &= \|f\|_{\mathcal{H}} \|k\|_{\mu_G,1;\infty} < \infty,
    \end{split}
  \end{equation*}
  and thus $\mathcal{H} \subset L^1(\Delta^m, \mu_G; \mathbb{R}^m)$. Hence from
  \cref{remark:exists} (with $\mathcal{F} = \mathcal{H}$) we know that for all
  $f \in \mathcal{H}$ the expectation
  $\Expect[\langle \Delta, f(G)\rangle_{\mathbb{R}^m}]$ exists and is finite.

  Define the linear operator $T_g \colon \mathcal{H} \to \mathbb{R}$ by
  $T_gf \coloneqq \Expect[\langle \Delta, f(G)\rangle_{\mathbb{R}^m}]$ for all
  $f \in \mathcal{H}$. In the same way as above, for all functions
  $f \in \mathcal{H}$ Hölder's inequality and \cref{eq:rkhs_bound} imply
  \begin{equation*}
    \begin{split}
      |T_gf| = |\Expect[\langle \Delta, f(G) \rangle_{\mathbb{R}^m}] &\leq \Expect[|\langle \Delta, f(G) \rangle_{\mathbb{R}^m}|] \\
      &\leq \|\delta\|_{\mu_G,\infty} \|f\|_{\mu_G,1} \\
      &\leq \|f\|_{\mu_G,1} \leq \|f\|_{\mathcal{H}} \|k\|_{\mu_G,1;\infty} < \infty.
    \end{split}
  \end{equation*}
  Thus $T_g$ is a continuous linear operator, and therefore it follows from
  Riesz's representation theorem that there exists a unique function
  $\mu_g \in \mathcal{H}$ such that
  \begin{equation*}
    \Expect[\langle \Delta, f(G)\rangle_{\mathbb{R}^m}] = T_g f = \langle f, \mu_g \rangle_{\mathcal{H}}
  \end{equation*}
  for all $f \in \mathcal{H}$. This implies that for all $t \in \Delta^m$ and
  $y \in \mathbb{R}^m$
  \begin{equation*}
    \langle y, \mu_g(t) \rangle_{\mathbb{R}^m} = \langle k(\cdot, t) y, \mu_g \rangle_{\mathcal{H}} = \Expect[\langle \Delta, k(G, t) y \rangle_{\mathbb{R}^m}]. \qedhere
  \end{equation*}
\end{proof}

\Cref{lemma:embedding} allows us to rewrite $\kernelmeasure[k, g]$ in a
more explicit way.

\begin{lemma}[Explicit formulation]\label{lemma:meance}
  Let $\mathcal{H} \subset {(\mathbb{R}^m)}^{\Delta^m}$ be a RKHS with kernel
  $k \colon \Delta^m \times \Delta^m \to \mathcal{L}(\mathbb{R}^m)$, and assume
  that $k(\cdot, t) u$ is measurable for all $t \in \Delta^m$ and
  $u \in \mathbb{R}^m$ and $\|k\|_{\mu_G,1} < \infty$. Then
  \begin{equation*}
    \kernelmeasure[k, g] = \|\mu_g\|_{\mathcal{H}},
  \end{equation*}
  where $\mu_g$ is the embedding defined in \cref{lemma:embedding}. Moreover,
  \begin{equation*}
    \squaredkernelmeasure[k, g] \coloneqq \kernelmeasure^2[k, g] = \Expect[\langle e_Y - g(X), k(g(X), g(X'))(e_{Y'} - g(X'))\rangle_{\mathbb{R}^m}],
  \end{equation*}
  where $(X', Y')$ is an independent copy of $(X,Y)$ and $e_i$ denotes the $i$th
  unit vector.
\end{lemma}

\begin{proof}
  Let $\mathcal{F}$ be the unit ball in the RKHS $\mathcal{H}$. From
  \cref{lemma:embedding} we know that for all $f \in \mathcal{F}$ the
  expectation $\Expect[\langle \Delta, f(G) \rangle_{\mathbb{R}^m}]$ exists and
  is given by
  \begin{equation*}
    \Expect[\langle \Delta, f(G) \rangle_{\mathbb{R}^m}] = \langle f, \mu_g \rangle_{\mathcal{H}},
  \end{equation*}
  where $\mu_g$ is the embedding defined in \cref{lemma:embedding}. Thus the
  definition of the dual norm yields
  \begin{equation*}
    \kernelmeasure[k, g] = \measure[\mathcal{F}, g] = \sup_{f \in \mathcal{F}} \Expect[\langle \Delta, f(G) \rangle_{\mathbb{R}^m}] = \sup_{f \in \mathcal{F}} \langle f, \mu_g \rangle_{\mathcal{H}} = \|\mu_g\|_{\mathcal{H}}.
  \end{equation*}

  Thus from the reproducing property and \cref{lemma:embedding} we obtain
  \begin{equation*}
    \begin{split}
      \squaredkernelmeasure[k, g] = \kernelmeasure^2[k, g] &= \langle \mu_g, \mu_g \rangle_{\mathcal{H}} = \Expect[\langle \Delta, \mu_g(G)\rangle_{\mathbb{R}^m}] \\
      &= \Expect[\Expect[\langle \Delta', k(G', G) \Delta \rangle_{\mathbb{R}^m}| G]] \\
      &= \Expect[\langle \Delta, k(G, G') \Delta' \rangle_{\mathbb{R}^m}],
    \end{split}
  \end{equation*}
  where $G'$ is an independent copy of $G$ and $\Delta' \coloneqq \delta(G')$.

  By rewriting
  \begin{equation*}
    \Delta = \Expect[e_Y | G] - G = \Expect[e_Y - G| G]
  \end{equation*}
  and $\Delta'$ in the same way, we get
  \begin{equation*}
    \squaredkernelmeasure[k, g] = \Expect[\langle e_Y - G, k(G, G') (e_{Y'} - G') \rangle_{\mathbb{R}^m}].
  \end{equation*}
  Plugging in the definitions of $G$ and $G'$ yields
  \begin{equation*}
    \squaredkernelmeasure[k, g] = \Expect[\langle e_Y - g(X), k(g(X), g(X')) (e_{Y'} - g(X')) \rangle_{\mathbb{R}^m}]. \qedhere
  \end{equation*}
\end{proof}

\section{Estimators}

Let $\mathcal{D} = \{(X_i, Y_i)\}_{i=1}^n$ be a set of random variables that are
i.i.d.\ as $(X,Y)$. Regardless of the space $\mathcal{F}$ the plug-in estimator
of $\measure[\mathcal{F},g]$ is
\begin{equation*}
  \widehat{\measure}[\mathcal{F},g,\mathcal{D}] \coloneqq \sup_{f \in \mathcal{F}} \frac{1}{n} \sum_{i=1}^n \langle \delta(X_i,Y_i), f(g(X_i)) \rangle_{\mathbb{R}^m}.
\end{equation*}
If $\mathcal{F}$ is the unit ball in a RKHS, i.e., for the kernel calibration
error, we can calculate this estimator explicitly.

\begin{lemma}[Biased estimator]\label{lemma:skce_biased}
  Let $\mathcal{F}$ be the unit ball in a RKHS
  $\mathcal{H} \subset {(\mathbb{R}^m)}^{\Delta^m}$ with kernel
  $k \colon \Delta^m \times \Delta^m \to \mathcal{L}(\mathbb{R}^m)$. Then
  \begin{equation*}
    \widehat{\measure}[\mathcal{F},g,\mathcal{D}] = \frac{1}{n} {\left[\sum_{i,j=1}^n \left\langle \delta(X_i, Y_i), k(g(X_i), g(X_j)) \delta(X_j, Y_j) \right\rangle_{\mathbb{R}^m} \right]}^{1/2}.
  \end{equation*}
\end{lemma}

\begin{proof}
  From the reproducing property and the definition of the dual norm it
  follows that
  \begin{equation*}
    \widehat{\measure}[\mathcal{F},g,\mathcal{D}] = \sup_{f \in \mathcal{F}} \left\langle \frac{1}{n} \sum_{i=1}^n k(\cdot, g(X_i)) \delta(X_i, Y_i), f \right\rangle_{\mathcal{H}} = \frac{1}{n} {\left\|\sum_{i=1}^n k(\cdot, g(X_i)) \delta(X_i, Y_i) \right\|}_{\mathcal{H}}.
  \end{equation*}
  Applying the reproducing property yields the result.
\end{proof}

Since we can uniquely identify the unit ball $\mathcal{F}$ with the
matrix-valued kernel $k$ and the plug-in estimator in \cref{lemma:skce_biased}
does not depend on $\mathcal{F}$ explicitly, we introduce the notation
\begin{equation*}
  \widehat{\kernelmeasure}[k, g, \mathcal{D}] \coloneqq \widehat{\measure}[\mathcal{F}, g, \mathcal{D}]
  \quad \text{and} \quad
  \biasedestimator[k, g, \mathcal{D}] \coloneqq \widehat{\kernelmeasure}^2[k, g, \mathcal{D}],
\end{equation*}
where $\mathcal{F}$ is the unit ball in the RKHS
$\mathcal{H} \subset {(\mathbb{R}^m)}^{\Delta^m}$ corresponding to kernel $k$.
By removing the terms involving the same random variables we obtain an unbiased
estimator.

\begin{lemma}[Unbiased estimator]\label{lemma:skce_unbiased}
  Let $k \colon \Delta^m \times \Delta^m \to \mathcal{L}(\mathbb{R}^m)$ be a
  kernel, and assume that $k(\cdot,t) u$ is measurable for all $t \in \Delta^m$
  and $u \in \mathbb{R}^m$, and $\|k\|_{\mu_G,1} < \infty$. Then
  \begin{equation*}
    \unbiasedestimator[k, g, \mathcal{D}] \coloneqq \frac{1}{n(n-1)} \sum_{\substack{i,j=1,\\i\neq j}}^n \left\langle \delta(X_i, Y_i), k(g(X_i), g(X_j)) \delta(X_j, Y_j)\right\rangle_{\mathbb{R}^m}
  \end{equation*}
  is an unbiased estimator of $\squaredkernelmeasure[k, g]$.
\end{lemma}

\begin{proof}
  The assumptions of \cref{lemma:meance} are satisfied, and hence we know that
  \begin{equation*}
    \squaredkernelmeasure[k, g] = \Expect[\langle \delta(X, Y), k(g(X), g(X')) \delta(X', Y')\rangle_{\mathbb{R}^m}],
  \end{equation*}
  where $(X', Y')$ is an independent copy of $(X,Y)$. Since $(X,Y)$, $(X',Y')$,
  and $(X_i, Y_i)$ are i.i.d., we have
  \begin{equation*}
    \begin{split}
      \Expect[\unbiasedestimator[k, g, \mathcal{D}]] &= \frac{1}{n(n-1)} \sum_{\substack{i=1,\\i\neq j}}^n \Expect[\langle \delta(X,Y), k(g(X), g(X')) \delta(X',Y')\rangle_{\mathbb{R}^m}] \\
      &= \Expect[\langle \delta(X,Y), k(g(X),g(X')) \delta(X',Y')\rangle_{\mathbb{R}^m}] \\
      &= \squaredkernelmeasure[k, g],
    \end{split}
  \end{equation*}
  which shows that $\unbiasedestimator[k, g, \mathcal{D}]$ is an unbiased
  estimator of $\squaredkernelmeasure[k, g]$.
\end{proof}

There exists an unbiased estimator with higher variance that scales not
quadratically but only linearly with the number of samples.

\begin{lemma}[Linear estimator]\label{lemma:skce_linear}
  Let $k \colon \Delta^m \times \Delta^m \to \mathcal{L}(\mathbb{R}^m)$ be a
  kernel, and assume that $k(\cdot,t) u$ is measurable for all $t \in \Delta^m$
  and $u \in \mathbb{R}^m$, and $\|k\|_{\mu_G,1} < \infty$. Then
  \begin{equation*}
    \linearestimator[k, g, \mathcal{D}] \coloneqq \frac{1}{\lfloor n / 2 \rfloor} \sum_{i=1}^{\lfloor n / 2 \rfloor} \left\langle \delta(X_{2i-1}, Y_{2i-1}), k(g(X_{2i-1}), g(X_{2i})) \delta(X_{2i}, Y_{2i})\right\rangle_{\mathbb{R}^m}
  \end{equation*}
  is an unbiased estimator of $\squaredkernelmeasure[k, g]$.
\end{lemma}

\begin{proof}
  The assumptions of \cref{lemma:meance} are satisfied, and hence we know that
  \begin{equation*}
    \squaredkernelmeasure[k, g] = \Expect[\langle \delta(X, Y), k(g(X), g(X')) \delta(X', Y')\rangle_{\mathbb{R}^m}],
  \end{equation*}
  where $(X', Y')$ is an independent copy of $(X,Y)$. Since $(X,Y)$, $(X',Y')$,
  and $(X_i, Y_i)$ are i.i.d., we have
  \begin{equation*}
    \begin{split}
      \Expect[\linearestimator[k, g, \mathcal{D}]] &= \frac{1}{\lfloor n / 2\rfloor} \sum_{i=1}^{\lfloor n / 2 \rfloor} \Expect[\langle \delta(X,Y), k(g(X), g(X')) \delta(X',Y')\rangle_{\mathbb{R}^m}] \\
      &= \Expect[\langle \delta(X,Y), k(g(X),g(X')) \delta(X',Y')\rangle_{\mathbb{R}^m}] \\
      &= \squaredkernelmeasure[k, g],
    \end{split}
  \end{equation*}
  which shows that $\linearestimator[k, g, \mathcal{D}]$ is an unbiased
  estimator of $\squaredkernelmeasure[k, g]$.
\end{proof}

\section{Asymptotic distributions}

In this section we investigate the asymptotic behaviour of the proposed
estimators. We start with a simple but very useful statement.

\begin{lemma}\label{lemma:square_integrable}
  Let $k \colon \Delta^m \times \Delta^m \to \mathcal{L}(\mathbb{R}^m)$ be a
  kernel, and assume that $k(\cdot,t) u$ is measurable for all
  $t \in \Delta^m$ and $u \in \mathbb{R}^m$, and $\|k\|_{\mu_G,2} < \infty$.

  Then $\Var[\langle \Delta, k(G,G') \Delta'\rangle_{\mathbb{R}^m}] < \infty$,
  where $G'$ is an independent copy of $G$ and $\Delta' \coloneqq \delta(G')$.
\end{lemma}

\begin{proof}
  From the Cauchy-Schwarz inequality and the definition of the operator norm we
  obtain
  \begin{equation*}
    \Expect[\langle \Delta, k(G,G') \Delta'\rangle_{\mathbb{R}^m}^2] \leq
    \Expect[\|\Delta\|^2_2 \|k(G,G')\|^2_{2;2} \|\Delta\|^2_2] \leq 4 \Expect[\|k(G,G')\|^2_{2;2}].
  \end{equation*}
  Hence by \cref{eq:rkhs_st_bound}
  \begin{equation*}
    \begin{split}
      \Expect[\langle \Delta', k(G,G') \Delta' \rangle_{\mathbb{R}^m}^2] &\leq 4 \Expect[\|k(G,G)\|_{2;2} \|k(G',G')\|_{2;2}] \\
      &= 4 \Expect[\|k(G,G)\|_{2;2}] \Expect[\|k(G',G')\|_{2;2}] = 4 {\left(\Expect[\|k(G,G)\|_{2;2}]\right)}^2,
    \end{split}
  \end{equation*}
  which implies
  \begin{equation*}
    \Expect[\langle \Delta, k(G,G') \Delta' \rangle_{\mathbb{R}^m}^2] < \infty
  \end{equation*}
  since by assumption $\|k\|_{\mu_G,2;2} < \infty$.
\end{proof}

Since the unbiased estimator $\unbiasedestimator$ is a U-statistic, we know that
the random variable $\sqrt{n}(\unbiasedestimator - \squaredkernelmeasure)$ is
asymptotically normally distributed under certain conditions.

\begin{theorem}\label{thm:quadratic_asymptotic}
  Let $k \colon \Delta^m \times \Delta^m \to \mathcal{L}(\mathbb{R}^m)$ be a
  kernel, and assume that $k(\cdot,t) u$ is measurable for all $t \in \Delta^m$
  and $u \in \mathbb{R}^m$, and $\|k\|_{\mu_G,1} < \infty$.

  If $\Var[\langle \Delta, k(G,G') \Delta'\rangle_{\mathbb{R}^m}] < \infty$,
  then
  \begin{equation*}
    \sqrt{n} \left(\unbiasedestimator[k, g, \mathcal{D}] - \squaredkernelmeasure[k, g]\right) \xrightarrow{d} \mathcal{N}(0, 4\zeta_1),
  \end{equation*}
  where
  \begin{equation*}
    \zeta_1 \coloneqq \Expect[\langle \Delta, k(G,G') \Delta' \rangle_{\mathbb{R}^m} \langle \Delta, k(G,G'') \Delta'' \rangle_{\mathbb{R}^m}] - \squaredkernelmeasure^2[k, g],
  \end{equation*}
  where $G'$ and $G''$ are independent copies of $G$ and
  $\Delta' \coloneqq \delta(G')$ and $\Delta'' \coloneqq \delta(G'')$.
\end{theorem}

\begin{proof}
  The statement follows immediately from \citet[Theorem~12.3][]{vaart98_asymp_statis}.
\end{proof}

If model $g$ is strongly calibrated, then $\zeta_1 = 0$, and hence
$\unbiasedestimator$ is a so-called degenerate U-statistic
\citep[see, e.g., Section~12.3][]{vaart98_asymp_statis}.

\begin{theorem}\label{thm:quadratic_asymptotic_degenerate}
  Let $k \colon \Delta^m \times \Delta^m \to \mathcal{L}(\mathbb{R}^m)$ be a
  kernel, and assume that $k(\cdot,t) u$ is measurable for all $t \in \Delta^m$ and
  $u \in \mathbb{R}^m$, and $\|k\|_{\mu_G,2} < \infty$.

  If $g$ is strongly calibrated, then
  \begin{equation*}
    n \unbiasedestimator[k, g, \mathcal{D}] \xrightarrow{d} \sum_{i=1}^\infty \lambda_i (Z_i^2 - 1),
  \end{equation*}
  where $Z_i$ are independent standard normally distributed random variables and
  $\lambda_i$ with $\sum_{i=1}^\infty \lambda_i^2 < \infty$ are eigenvalues of
  the integral operator
  \begin{equation*}
    Kf(\xi,y) \coloneqq \int \langle e_y - \xi, k(\xi, \xi') (e_{y'} - \xi') \rangle_{\mathbb{R}^m} f(\xi',y') \,\mu_{G\times Y}(\mathrm{d}(\xi',y'))
  \end{equation*}
  on the space $L^2(\Delta^m \times \{1,\ldots,m\}, \mu_{G \times Y})$.
\end{theorem}

\begin{proof}
  From \cref{lemma:square_integrable} we know that
  \begin{equation*}
    \Var[\langle \Delta, k(G,G') \Delta' \rangle_{\mathbb{R}^m}] < \infty.
  \end{equation*}
  Moreover, since $g$ is strongly calibrated, $\Delta = 0$ almost surely and by
  \cref{thm:ce_zero} $\kernelmeasure[k, g] = 0$. Thus we obtain
  \begin{equation*}
    \Expect[\langle \Delta, k(G,G') \Delta' \rangle_{\mathbb{R}^m} \langle \Delta, k(G,G'') \Delta'' \rangle_{\mathbb{R}^m}] - \squaredkernelmeasure^2[k, g] = 0.
  \end{equation*}

  The statement follows from
  \citet[Theorem in Section~5.5.2][]{serfling80_approx_theor_mathem_statis}.
\end{proof}

As discussed by \citet{gretton12_kernel_two_sampl_test} in the case of two-sample
tests, a natural idea is to find a threshold $c$ such that
$\mathbb{P}[n \unbiasedestimator[k, g, \mathcal{D}] > c \given \nullhypothesis] \leq \alpha$,
where $\nullhypothesis$ is the null hypothesis that the model is strongly
calibrated. The desired quantile can be estimated by fitting Pearson curves to
the empirical distribution by moment matching
\citep{gretton12_kernel_two_sampl_test}, or alternatively by
bootstrapping \citep{arcones92_boots_u_v_statis}, both computed and performed
under the assumption that the model is strongly calibrated.

If model $g$ is strongly calibrated we know
$\Expect[\unbiasedestimator[k, g, \mathcal{D}]] = 0$. Moreover, it follows from
\citet[p.~299][]{hoeffding48_class_statis_with_asymp_normal_distr} that
\begin{equation*}
  \Expect\left[\unbiasedestimator^2[k, g, \mathcal{D}]\right] = \frac{2}{n(n-1)} \Expect\left[{\left(\langle e_Y - g(X), k(g(X),g(X')) (e_{Y'} - g(X')) \rangle_{\mathbb{R}^m}\right)}^2\right],
\end{equation*}
where $(X',Y')$ is an independent copy of $(X,Y)$. By some tedious calculations
we can retrieve higher-order moments as well. If model $g$ is strongly
calibrated, we know from
\citet[Lemma~B, Section~5.2.2]{serfling80_approx_theor_mathem_statis} that for
$r \geq 2$
\begin{equation*}
  \Expect\left[\unbiasedestimator^r[k, g, \mathcal{D}]\right] = O(n^{-r})
\end{equation*}
as the number of samples $n$ goes to infinity, provided that
\begin{equation*}
  \Expect\left[{\left|\langle e_Y - g(X), k(g(X),g(X')) (e_{Y'} - g(X')) \rangle_{\mathbb{R}^m}\right|}^r\right] < \infty.
\end{equation*}

Alternatively, as discussed by \citet[Section~5][]{arcones92_boots_u_v_statis},
we can estimate $c$ by using quantiles of the bootstrap statistic
\begin{equation*}
  \begin{split}
    T &= 2 n^{-1} \sum_{1 \leq i < j \leq n} \left[ h((X_{*,i},Y_{*,i}), (X_{*,j},Y_{*,j})) - n^{-1} \sum_{k=1}^n h((X_{*,i},Y_{*,i}),(X_k,Y_k)) \right. \\
    &\qquad \qquad \qquad \qquad \left. - n^{-1} \sum_{k=1}^n h((X_k,Y_k),(X_{*,j},Y_{*,j})) + n^{-2} \sum_{k,l=1}^n h((X_k,Y_k),(X_l,Y_l))\right],
  \end{split}
\end{equation*}
where
\begin{equation*}
  h((x,y),(x',y')) \coloneqq \langle \delta(x,y), k(g(x),g(x')) \delta(x',y') \rangle_{\mathbb{R}^m}
\end{equation*}
and $(X_{*,1},Y_{*,1}),\ldots,(X_{*,n},Y_{*,n})$ are sampled with replacement
from the data set $\mathcal{D}$. Then asymptotically
\begin{equation*}
  \Prob\left[n \unbiasedestimator[k, g, \mathcal{D}] > c \given \nullhypothesis \right] \approx \Prob[T > c \given \mathcal{D}].
\end{equation*}

For the linear estimator, the asymptotical behaviour follows from the central
limit theorem
\citep[e.g., Theorem~A in Section 1.9][]{serfling80_approx_theor_mathem_statis}.

\begin{corollary}\label{cor:linear_asymptotic}
  Let $k \colon \Delta^m \times \Delta^m \to \mathcal{L}(\mathbb{R}^m)$ be a
  kernel, and assume that $k(\cdot,t) u$ is measurable for all $t \in \Delta^m$
  and $u \in \mathbb{R}^m$, and $\|k\|_{\mu_G,1} < \infty$.

  If $\sigma^2 \coloneqq \Var[\langle \delta(X,Y), k(g(X),g(X')) \delta(X',Y')\rangle_{\mathbb{R}^m}] < \infty$,
  where $(X',Y')$ is an independent copy of $(X,Y)$, then
  \begin{equation*}
    \sqrt{\lfloor n/ 2\rfloor} \left(\linearestimator[k, g, \mathcal{D}] - \squaredkernelmeasure[k, g]\right) \xrightarrow{d} \mathcal{N}(0, \sigma^2).
  \end{equation*}
\end{corollary}

As noted in the following statement, the variance $\sigma^2$ is finite if
$t \mapsto \|k(t,t)\|$ is $L^2$-integrable with respect to measure $\mu_G$.

\begin{corollary}
  Let $k \colon \Delta^m \times \Delta^m \to \mathcal{L}(\mathbb{R}^m)$ be a
  kernel, and assume that $k(\cdot,t) u$ is measurable for all $t \in \Delta^m$
  and $u \in \mathbb{R}^m$, and $\|k\|_{\mu_G,2} < \infty$.

  Then $\sigma^2 \coloneqq \Var[\langle \delta(X,Y), k(G,G') \delta(X',Y')\rangle_{\mathbb{R}^m}] < \infty$,
  where $(X',Y')$ is an independent copy of $(X,Y)$ with $G'= g(X')$, and
  \begin{equation*}
    \sqrt{\lfloor n/ 2\rfloor} \left(\linearestimator[k, g, \mathcal{D}] - \squaredkernelmeasure[k, g]\right) \xrightarrow{d} \mathcal{N}(0, \sigma^2).
  \end{equation*}
\end{corollary}

\begin{proof}
 The statement follows from \cref{cor:linear_asymptotic}.
\end{proof}

The weak convergence of $\linearestimator$ yields the following asymptotic test.

\begin{corollary}
  Let the assumptions of \cref{cor:linear_asymptotic} be satisfied.

  A one-sided statistical test with test statistic
  $\linearestimator[k, g, \mathcal{D}]$ and asymptotic significance level
  $\alpha$ has the acceptance region
  \begin{equation*}
    \sqrt{\lfloor n / 2\rfloor} \linearestimator[k, g, \mathcal{D}] < z_{1-\alpha} \hat{\sigma},
  \end{equation*}
  where $z_{1-\alpha}$ is the $(1-\alpha)$-quantile of the standard normal
  distribution and $\hat{\sigma}$ is a consistent estimator of the standard
  deviation of
  $\langle \delta(X,Y), k(g(X),g(X')) \delta(X',Y')\rangle_{\mathbb{R}^m}$.
\end{corollary}

\section{Distribution-free bounds}

First we prove a helpful bound.

\begin{lemma}\label{lemma:term_bound}
  Let $k \colon \Delta^m \times \Delta^m \to \mathcal{L}(\mathbb{R}^m)$ be a
  kernel, and assume that
  $K_{p;q} \coloneqq \sup_{s,t \in \Delta^m} \|k(s,t)\|_{p;q} < \infty$ for some
  $1 \leq p,q \leq \infty$. Then
  \begin{equation*}
    \sup_{x,x' \in \mathcal{X}, y,y' \in \{1,\ldots,m\}} \left|\langle \delta(x, y), k(g(x), g(x')) \delta(x', y') \rangle_{\mathbb{R}^m}\right|
    \leq 2^{1 + 1 / p - 1 / q} K_{p;q} \eqqcolon B_{p;q}.
  \end{equation*}
\end{lemma}

\begin{proof}
  By Hölder's inequality and the definition of the operator norm for all
  $s,t \in \Delta^m$ and $u,v \in \mathbb{R}^m$
  \begin{equation*}
    |\langle u, k(s,t)v \rangle_{\mathbb{R}^m}| \leq \|u\|_{q'} \|k(s,t)v\|_q \leq \|u\|_{q'} \|v\|_p \|k(s,t)\|_{p;q} \leq K_{p;q} \|u\|_{q'} \|v\|_p.
  \end{equation*}
  The result follows from the fact that
  $\max_{s,t \in \Delta^m} \|s - t\|_p = 2^{1/p}$ and
  $\max_{s,t \in \Delta^m} \|s - t\|_{q'} = 2^{1/q'} = 2^{1 - 1/q}$.
\end{proof}

Unfortunately, the tightness of the bound in \cref{lemma:term_bound} depends on
the choice of $p$ and $q$, as the following example shows.

\begin{example}
  Let $k = \phi \mathbf{I}_m$, where
  $\phi \colon \Delta^m \times \Delta^m \to \mathbb{R}$ is a scalar-valued
  kernel and $\mathbf{I}_m \in \mathbb{R}^{m \times m}$ is the identity matrix.
  Assume that $\Phi \coloneqq \sup_{s,t \in \Delta^m} |\phi(s,t)| < \infty$. One
  can show that for all $s,t \in \Delta^m$
  \begin{equation*}
    \|k(s,t)\|_{p;q} = \begin{cases}
      \phi(s,t) & \text{if } p \leq q,\\
      m^{1/q - 1/p} \phi(s,t) & \text{if } p > q,
    \end{cases}
  \end{equation*}
  which implies that
  \begin{equation*}
    K_{p;q} = \begin{cases}
      \Phi & \text{if } p \leq q,\\
      m^{1/q - 1/p} \Phi & \text{if } p > q.
    \end{cases}
  \end{equation*}
  Thus the bound $B_{p;q}$ in \cref{lemma:term_bound} is
  \begin{equation*}
    B_{p; q} = \begin{cases}
      2^{1 + 1/p - 1/q} \Phi & \text{if } p \leq q,\\
      2^{1 + 1/p - 1/q} m^{1/q - 1/p} \Phi & \text{if } p > q,
    \end{cases}
  \end{equation*}
  which attains its smallest value
  $\min_{1 \leq p,q \leq \infty} B_{p;q} = 2 \Phi$ if and only if $p = q$ or
  $m = 2$ and $p > q$. Thus for any other choice of $p$ and $q$
  \cref{lemma:term_bound} provides a non-optimal bound.
\end{example}

\begin{theorem}\label{thm:biased_bound_uniform_general}
  Let $k \colon \Delta^m \times \Delta^m \to \mathcal{L}(\mathbb{R}^m)$ be a
  kernel, and assume that $k(\cdot, t)u$ is measurable for all $t \in \Delta^m$
  and $u \in \mathbb{R}^m$, and
  $K_{p;q} \coloneqq \sup_{s,t \in \Delta^m} \|k(s,t)\|_{p;q} < \infty$ for some
  $1 \leq p,q \leq \infty$. Then for all $\epsilon > 0$
  \begin{equation*}
    \Prob\left[\left|\widehat{\kernelmeasure}[k, g, \mathcal{D}] - \kernelmeasure[k, g]\right| \geq 2{(B_{p;q}/n)}^{1/2} + \epsilon\right] \leq \exp{\left(-\frac{\epsilon^2 n}{2 B_{p;q}}\right)}.
  \end{equation*}
\end{theorem}

\begin{proof}
  Let $\mathcal{F}$ be the unit ball in the RKHS
  $\mathcal{H} \subset {(\mathbb{R}^m)}^{\Delta^m}$ corresponding to kernel $k$.
  We consider the random variable
  \begin{equation*}
    F \coloneqq \left|\widehat{\kernelmeasure}[k, g, \mathcal{D}] - \kernelmeasure[k, g]\right|.
  \end{equation*}
  The randomness of $F$ is due to the randomness of the data points
  $(X_i, Y_i)$, and by \cref{lemma:meance,lemma:skce_biased} we can rewrite $F$
  as
  \begin{equation*}
    F = n^{-1} \left| \left\|\sum_{i=1}^n k(\cdot, g(X_i)) \delta(X_i, Y_i) \right\|_{\mathcal{H}} - n \left\| \mu_g \right\|_{\mathcal{H}} \right| \eqqcolon f((X_1, Y_1), \ldots, (X_n, Y_n)),
  \end{equation*}
  where $\mu_g$ is the embedding defined in \cref{lemma:embedding}. The triangle
  inequality implies that for all
  $z_i = (x_i, y_i) \in \mathcal{X} \times \{1,\ldots,m\}$
  \begin{equation}\label{eq:F_leq_H}
    \begin{split}
      f(z_1, \ldots, z_n) &= n^{-1} \left| \left\| \sum_{i=1}^n k(\cdot, g(x_i)) \delta(x_i, y_i) \right\|_{\mathcal{H}} - n \|\mu_g \|_{\mathcal{H}} \right| \\
      &\leq n^{-1} \left\|\sum_{i=1}^n \left(k(\cdot, g(x_i)) \delta(x_i, y_i) - \mu_g\right)\right\|_{\mathcal{H}} \eqqcolon h(z_1, \ldots, z_n),
    \end{split}
  \end{equation}
  where $h \colon {(\mathcal{X} \times \{1,\ldots,m\})}^n \to \mathbb{R}$ is
  measurable and hence induces a random variable
  $H \coloneqq h((X_1,Y_1),\ldots,(X_n,Y_n))$.

  By the reproducing property and \cref{lemma:term_bound}, for all
  $x,x' \in \mathcal{X}$ and $y,y' \in \{1,\ldots,m\}$ we have
  \begin{equation*}
    \begin{split}
      \left\|k(\cdot, g(x)) \delta(x,y) - k(\cdot, g(x')) \delta(x', y')\right\|^2_{\mathcal{H}} &= \langle \delta(x,y), k(g(x),g(x)) \delta(x,y)\rangle_{\mathbb{R}^m} \\
      &\quad - \langle \delta(x,y), k(g(x), g(x')) \delta(x', y') \rangle_{\mathbb{R}^m} \\
      &\quad - \langle \delta(x', y'), k(g(x'), g(x)) \delta(x,y)\rangle_{\mathbb{R}^m} \\
      &\quad + \langle \delta(x', y'), k(g(x'), g(x')) \delta(x',y')\rangle_{\mathbb{R}^m} \\
      &\leq 4 B_{p;q}.
    \end{split}
  \end{equation*}
  Thus for all $i \in \{1,\ldots,m\}$ the triangle inequality implies
  \begin{multline*}
    \sup_{z, z', z_j (j \neq i)} |h(z_1,\ldots,z_{i-1},z,z_{i+1},\ldots,z_n) - h(z_1,\ldots,z_{i-1},z',z_{i+1},\ldots,z_n)| \\
    \leq \sup_{x,x,y,y'} n^{-1} \left\|k(\cdot, g(x)) \delta(x,y) - k(\cdot, g(x')) \delta(x', y')\right\|_{\mathcal{H}} \leq \frac{2 B_{p;q}^{1/2}}{n}.
  \end{multline*}
  Hence we can apply McDiarmid's inequality to the random variable $H$, which
  yields for all $\epsilon > 0$
  \begin{equation}\label{eq:mcdiarmid_skce}
    \Prob\left[H \geq \Expect[H] + \epsilon \right] \leq \exp{\left(-\frac{\epsilon^2 n}{2 B_{p;q}}\right)}.
  \end{equation}

  In the final parts of the proof we bound the expectation $\Expect[H]$. By
  \cref{lemma:embedding,lemma:skce_biased}, we know that
  \begin{equation*}
    \begin{split}
      H &= h((X_1,Y_1),\ldots,(X_n,Y_n)) \\
      &= \sup_{f \in \mathcal{F}} n^{-1} \left| \sum_{i=1}^n \Big(\langle \delta(X_i,Y_i), f(g(X_i)) \rangle_{\mathbb{R}^m} - \Expect\left[\langle \delta(X, Y), f(g(X)) \rangle_{\mathbb{R}^m} \right]\Big) \right| \\
      &= \sup_{f \in \mathcal{F}_0} n^{-1} \left| \sum_{i=1}^n f(X_i,Y_i) - \Expect[f(X,Y)] \right|,
    \end{split}
  \end{equation*}
  where $\mathcal{F}_0 \coloneqq \{ f \colon \mathcal{X} \times \{1,\ldots,m\} \to \mathbb{R}, (x,y) \mapsto \langle \delta(x,y), \tilde{f}(g(x)) \rangle_{\mathbb{R}^m} \colon \tilde{f} \in \mathcal{F} \}$
  is a class of measurable functions. As \citet{gretton12_kernel_two_sampl_test},
  we make use of symmetrization ideas
  \citep[p.~108]{vaart96_weak_conver_empir_proces}. From
  \citet[Lemma~2.3.1]{vaart96_weak_conver_empir_proces} it follows that
  \begin{equation*}
    \Expect[H] =\Expect\left[\sup_{f \in \mathcal{F}_0} n^{-1} \left| \sum_{i=1}^n f(X_i,Y_i) - \Expect[f(X,Y)] \right|\right] \leq 2 \Expect\left[\sup_{f \in \mathcal{F}_0} \left|n^{-1} \sum_{i=1}^n \epsilon_i f(X_i,Y_i) \right|\right],
  \end{equation*}
  where $\epsilon_1,\ldots,\epsilon_n$ are independent Rademacher random
  variables. Similar to \citet[Lemma~22]{bartlett02_radem_gauss_compl}, we
  obtain
  \begin{equation*}
    \begin{split}
      \Expect[H] &\leq 2n^{-1} \Expect\left[\sup_{f \in \mathcal{F}} \left| \sum_{i=1}^n \epsilon_i \langle \delta(X_i, Y_i), f(g(X_i))\rangle_{\mathbb{R}^m} \right|\right] \\
      &= 2n^{-1} \Expect\left[\sup_{f \in \mathcal{F}} \left| \left\langle \sum_{i=1}^n \epsilon_i k(\cdot, g(X_i)) \delta(X_i, Y_i), f \right\rangle_{\mathcal{H}} \right|\right] \\
      &= 2n^{-1} \Expect\left[ \left\|\sum_{i=1}^n \epsilon_i k(\cdot, g(X_i)) \delta(X_i, Y_i) \right\|_{\mathcal{H}} \right] \\
      &= 2n^{-1}\Expect\left[ {\left( \sum_{i,j=1}^n \epsilon_i \epsilon_j \langle k(\cdot, g(X_i)) \delta(X_i, Y_i), k(\cdot, g(X_j)) \delta(X_j, Y_j) \rangle_{\mathcal{H}}\right)}^{1/2} \right].
    \end{split}
  \end{equation*}
  By Jensen's inequality we get
  \begin{equation}\label{eq:EH_bound_weak}
    \begin{split}
      \Expect[H] &\leq 2n^{-1} {\left(\sum_{i,j=1}^n \Expect\left[\epsilon_i \epsilon_j \langle k(\cdot, g(X_i)) \delta(X_i, Y_i), k(\cdot, g(X_j)) \delta(X_j, Y_j) \rangle_{\mathcal{H}} \right]\right)}^{1/2} \\
      &= 2n^{-1/2} {\Big(\Expect\left[\langle k(\cdot, g(X)) \delta(X, Y), k(\cdot, g(X)) \delta(X, Y) \rangle_{\mathcal{H}} \right]\Big)}^{1/2} \\
      &\leq 2{(B_{p;q}/n)}^{1/2}.
    \end{split}
  \end{equation}

  All in all, from \cref{eq:F_leq_H,eq:mcdiarmid_skce,eq:EH_bound_weak} we
  obtain for all $\epsilon > 0$
  \begin{equation*}
    \begin{split}
      \Prob\left[\left|\widehat{\kernelmeasure}[k, g, \mathcal{D}] - \kernelmeasure[k, g]\right| \geq 2{(B_{p;q} / n)}^{1/2} + \epsilon \right] &= \Prob[F \geq 2{(B_{p;q}/n)}^{1/2} + \epsilon] \\
      &\leq \Prob[H \geq 2{(B_{p;q}/n)}^{1/2} + \epsilon] \\
      &\leq \Prob[H \geq \Expect[H] + \epsilon] \\
      &\leq \exp{\left( - \frac{\epsilon^2n}{2B_{p;q}}\right)},
    \end{split}
  \end{equation*}
  which concludes our proof.
\end{proof}

If model $g$ is calibrated in the strong sense, we can improve the bound.

\begin{theorem}\label{thm:biased_bound_uniform}
  Let $k \colon \Delta^m \times \Delta^m \to \mathcal{L}(\mathbb{R}^m)$ be a
  kernel, and assume that $k(\cdot, t)u$ is measurable for all $t \in \Delta^m$
  and $u \in \mathbb{R}^m$, and
  $K_{p;q} \coloneqq \sup_{s,t \in \Delta^m} \|k(s,t)\|_{p;q} < \infty$ for some
  $1 \leq p,q \leq \infty$. Define
  \begin{align*}
    B_1 &\coloneqq n^{-1/2} \left[\Expect\left[\langle \delta(X, Y), k(g(X), g(X)) \delta(X,Y)\rangle_{\mathbb{R}^m}\right] \right]^{1/2}, \qquad \text{and}\\
    B_2 &\coloneqq {\left(B_{p;q} / n\right)}^{1/2}.
  \end{align*}

  Then $B_1 \leq B_2$, and for all $\epsilon > 0$ and $i \in \{1,2\}$
  \begin{equation*}
    \Prob\left[\widehat{\kernelmeasure}[k, g, \mathcal{D}] \geq B_i + \epsilon\right] \leq  \exp{\left(- \frac{\epsilon^2 n}{2 B_{p;q}}\right)},
  \end{equation*}
  if $g$ is calibrated in the strong sense.
\end{theorem}

\begin{proof}
  Let $\mathcal{F}$ be the unit ball in the RKHS
  $\mathcal{H} \subset {(\mathbb{R}^m)}^{\Delta^m}$ corresponding to kernel $k$.
  \Cref{lemma:term_bound} implies
  \begin{equation}\label{eq:Bi_inequality}
    \begin{split}
      B_1 &= n^{-1/2} \left[\Expect\left[\langle \delta(X, Y), k(g(X), g(X)) \delta(X,Y)\rangle_{\mathbb{R}^m}\right] \right]^{1/2} \\
      &\leq n^{-1/2} \left[\Expect[B_{p;q}] \right]^{1/2} = {(B_{p;q}/n)}^{1/2} = B_2.
    \end{split}
  \end{equation}

  Let $H$ be defined as in the proof of \cref{thm:biased_bound_uniform_general}.
  Since $g$ is strongly calibrated, it follows from
  \cref{thm:ce_zero,lemma:meance} that $\mu_g = 0$, and thus by
  \cref{lemma:skce_biased}
  \begin{equation*}
    H = n^{-1} \left\|\sum_{i=1}^n k(\cdot, g(x_i)) \delta(x_i, y_i)\right\|_{\mathcal{H}} = \widehat{\kernelmeasure}[k, g, \mathcal{D}].
  \end{equation*}
  Thus \cref{eq:mcdiarmid_skce} implies
  \begin{equation}\label{eq:mcdiarmid_skce_measure}
    \Prob\left[\widehat{\kernelmeasure}[k, g, \mathcal{D}] \geq \Expect[\widehat{\kernelmeasure}[k, g, \mathcal{D}] + \epsilon\right] \leq \exp{\left(- \frac{\epsilon^2 n}{2 B_{p;q}}\right)}.
  \end{equation}

  Next we bound $\Expect[\widehat{\kernelmeasure}[k, g, \mathcal{D}]]$. From
  \cref{lemma:skce_biased} we get
  \begin{equation*}
    \Expect[\widehat{\kernelmeasure}[k, g, \mathcal{D}]] = \frac{1}{n} \Expect\left[{\left(\sum_{i,j=1}^n \langle \delta(X_i, Y_i), k(g(X_i), g(X_j)) \delta(X_j,Y_j)\rangle_{\mathbb{R}^m}\right)}^{1/2} \right],
  \end{equation*}
  and hence by Jensen's inequality we obtain
  \begin{equation*}
    \begin{split}
      \Expect[\widehat{\kernelmeasure}[k, g, \mathcal{D}]] &\leq \frac{1}{n} {\left( \Expect\left[\sum_{i,j=1}^n \langle \delta(X_i, Y_i), k(g(X_i), g(X_j)) \delta(X_j,Y_j)\rangle_{\mathbb{R}^m}\right]\right)}^{1/2} \\
      &= \frac{1}{n} \bigg(n \Expect\left[\langle \delta(X, Y), k(g(X), g(X)) \delta(X,Y)\rangle_{\mathbb{R}^m}\right] \\
      &\qquad + n(n-1) \Expect\left[\langle \delta(X, Y), k(g(X), g(X')) \delta(X',Y')\rangle_{\mathbb{R}^m}\right] \bigg)^{1/2},
    \end{split}
  \end{equation*}
  where $(X',Y')$ denotes an independent copy of $(X,Y)$. From
  \cref{lemma:meance} it follows that
  \begin{equation*}
    \Expect[\widehat{\kernelmeasure}[k, g, \mathcal{D}] \leq \left(\frac{1}{n} \Expect\left[\langle \delta(X, Y), k(g(X), g(X)) \delta(X,Y)\rangle_{\mathbb{R}^m}\right] + \left(1 - \frac{1}{n}\right) \squaredkernelmeasure[k, g] \right)^{1/2}.
  \end{equation*}
  If model $g$ is calibrated in the strong sense, we know from \cref{thm:ce_zero}
  that $\squaredkernelmeasure[k, g] = 0$. Thus we obtain
  \begin{equation}\label{eq:mcdiarmid_skce_chain}
    \Expect[\widehat{\kernelmeasure}[k, g, \mathcal{D}]] \leq B_1.
  \end{equation}

  All in all, from
  \cref{eq:Bi_inequality,eq:mcdiarmid_skce_measure,eq:mcdiarmid_skce_chain} it
  follows that for all $\epsilon > 0$ and $i \in \{1,2\}$
  \begin{equation*}
    \begin{split}
      \Prob\left[\widehat{\kernelmeasure}[k, g, \mathcal{D}] \geq B_i + \epsilon\right] &\leq \Prob\left[\widehat{\kernelmeasure}[k, g, \mathcal{D}] \geq B_1 + \epsilon\right] \\
      &\leq \Prob\left[\widehat{\kernelmeasure}[k, g, \mathcal{D}] \geq \Expect[\widehat{\kernelmeasure}[k, g, \mathcal{D}] + \epsilon \right] \\
      &\leq \exp{\left(-\frac{- \epsilon^2 n}{2 B_{p;q}}\right)},
    \end{split}
  \end{equation*}
  if $g$ is calibrated in the strong sense.
\end{proof}

Thus we obtain the following distribution-free hypothesis test.

\begin{corollary}
  Let the assumptions of \cref{thm:biased_bound_uniform} be satisfied.

  A statistical test with test statistic
  $\widehat{\kernelmeasure}[k, g, \mathcal{D}]$ and significance level
  $\alpha$ for the null hypothesis of model $g$ being calibrated in the strong
  sense has the acceptance region
  \begin{equation*}
    \widehat{\kernelmeasure}[k, g, \mathcal{D}] < {(B_{p;q} / n)}^{1/2} (1 + \sqrt{-2\log \alpha}).
  \end{equation*}
\end{corollary}

A distribution-free bound for the deviation of the unbiased estimator can be
obtained from the theory of U-statistics.

\begin{theorem}\label{thm:unbiased_bound_uniform}
  Let $k \colon \Delta^m \times \Delta^m \to \mathcal{L}(\mathbb{R}^m)$ be a
  kernel, and assume
  that $k(\cdot,t)u$ is measurable for all $t \in \Delta^m$ and
  $u \in \mathbb{R}^m$, and
  $K_{p;q} \coloneqq \sup_{s,t \in \Delta^m} \|k(s,t)\|_{p;q}$ for some
  $1 \leq p,q \leq \infty$. Then for all $t > 0$
  \begin{equation*}
    \mathbb{P}\left[\unbiasedestimator[k, g, \mathcal{D}] - \squaredkernelmeasure[k, g] \geq t\right] \leq \exp{\left(-\frac{\lfloor n/2 \rfloor t^2}{2 B_{p;q}^2}\right)}.
  \end{equation*}
  The same bound holds for
  $\mathbb{P}\left[\unbiasedestimator[k, g, \mathcal{D}] - \squaredkernelmeasure[k, g] \leq -t\right]$.
\end{theorem}

\begin{proof}
  By \cref{lemma:skce_unbiased},
  $\Expect[\unbiasedestimator[k, g, \mathcal{D}]] = \squaredkernelmeasure[k, g]$.
  Moreover, by \cref{lemma:term_bound} we know that
  \begin{equation*}
    \sup_{x,x' \in \mathcal{X}, y,y' \in \{1,\ldots,m\}} \left|\langle \delta(x,y), k(g(x),g(x')) \delta(x',y')\rangle_{\mathbb{R}^m} \right| \leq B_{p;q}.
  \end{equation*}
  Thus the result follows from the bound on U-statistics by
  \citet[p.~25]{hoeffding63_probab_inequal_sums_bound_random_variab}.
\end{proof}

We can derive a hypothesis test using the unbiased estimator.

\begin{corollary}
  Let the assumptions of \cref{thm:unbiased_bound_uniform} be satisfied.

  A one-sided statistical test with test statistic
  $\unbiasedestimator[k, g, \mathcal{D}]$ and significance level
  $\alpha$ for the null hypothesis of model $g$ being calibrated in the strong
  sense has the acceptance region
  \begin{equation*}
    \unbiasedestimator[k, g, \mathcal{D}] < \frac{B_{p;q}}{\sqrt{\lfloor n / 2 \rfloor}} \sqrt{- 2\log{\alpha}}.
  \end{equation*}
\end{corollary}

Analogously we can obtain a bound for the linear estimator.

\begin{theorem}\label{thm:linear_bound_uniform}
  Let $k \colon \Delta^m \times \Delta^m \to \mathcal{L}(\mathbb{R}^m)$ be a
  kernel, and assume that $k(\cdot,t)u$ is measurable for all $t \in \Delta^m$
  and $u \in \mathbb{R}^m$, and
  $K_{p;q} \coloneqq \sup_{s,t \in \Delta^m} \|k(s,t)\|_{p;q}$ for some
  $1 \leq p,q \leq \infty$. Then for all $t > 0$
  \begin{equation*}
    \Prob\left[\linearestimator[k, g, \mathcal{D}] - \squaredkernelmeasure[k, g] \geq t \right] \leq \exp{\left(-\frac{\lfloor n / 2\rfloor t^2}{2 B_{p;q}^2}\right)}.
  \end{equation*}
  The same bound holds for
  $\Prob\left[\linearestimator[k, g, \mathcal{D}] - \squaredkernelmeasure[k, g] \leq -t \right]$.
\end{theorem}

\begin{proof}
  By \cref{lemma:skce_linear},
  $\Expect[\linearestimator[k, g, \mathcal{D}]] = \squaredkernelmeasure[k, g]$.
  Moreover, by \cref{lemma:term_bound} we know that
  \begin{equation*}
    \sup_{x,x' \in \mathcal{X}, y,y' \in \{1,\ldots,m\}} \left|\langle \delta(x,y), k(g(x),g(x')) \delta(x',y')\rangle_{\mathbb{R}^m} \right| \leq B_{p;q}.
  \end{equation*}
  Thus by Hoeffding's inequality
  \citep[Theorem~2]{hoeffding63_probab_inequal_sums_bound_random_variab} for all
  $t > 0$
  \begin{equation*}
    \Prob\left[\linearestimator[k, g, \mathcal{D}] - \squaredkernelmeasure[k, g] \geq t \right] \leq \exp{\left(-\frac{\lfloor n / 2\rfloor t^2}{2 B_{p;q}^2}\right)}. \qedhere
  \end{equation*}
\end{proof}

Obviously this results yields another distribution-free hypothesis test.

\begin{corollary}
  Let the assumptions of \cref{thm:linear_bound_uniform} be satisfied.

  A one-sided statistical test with test statistic
  $\linearestimator[k, g, \mathcal{D}]$ and significance level
  $\alpha$ for the null hypothesis of model $g$ being calibrated in the strong
  sense has the acceptance region
  \begin{equation*}
    \linearestimator[k, g, \mathcal{D}] < \frac{B_{p;q}}{\sqrt{\lfloor n / 2 \rfloor}} \sqrt{- 2\log{\alpha}}.
  \end{equation*}
\end{corollary}

\section{Comparisons}

\subsection{Expected calibration error and maximum calibration error}

For certain spaces of bounded functions the calibration error $\measure$ turns
out to be a form of the $\ECE$. In particular, the $\ECE$ with respect to the
cityblock distance, the total variation distance, and the squared Euclidean
distance are special cases of $\measure$. Choosing $p=1$ in the following
statement corresponds to the special case of the $\MCE$.

\begin{lemma}[$\ECE$ and $\MCE$ as special cases]\label{lemma:ce_ece}
  Let $1 \leq p \leq \infty$ with Hölder conjugate $p'$. If
  $\mathcal{F} = K^p(\Delta^m, \mu_G; \mathbb{R}^m)$, then
  $\measure[\mathcal{F}, g] = \|\delta\|_{\mu_G,p'}$.
\end{lemma}

\begin{proof}
  Note that $\mathcal{F}$ is well defined since
  $\mathcal{F} \subset L^1(\Delta^m, \mu_G; \mathbb{R}^m)$.

  The statement follows from the extremal case of Hölder's inequality. More
  explicitly, let $\nu$ denote the counting measure on $\{1,\ldots,m\}$. Since
  both $\mu_G$ and $\nu$ are $\sigma$-finite measures, the product measure
  $\mu_G \otimes \nu$ on the product space
  $B \coloneqq \Delta^m \times \{1,\ldots,m\}$ is uniquely defined and
  $\sigma$-finite. Define
  $\tilde{\delta}(t, k) \coloneqq \delta_k(t)$ for all $(t,k) \in B$. Then we
  can rewrite
  \begin{equation*}
    \begin{split}
      \measure[\mathcal{F}, g] &= \sup_{f \in K^p(\Delta^m, \mu_G; \mathbb{R}^m)} \int_{\Delta^m} \langle \delta(x), f(x) \rangle_{\mathbb{R}^m} \, \mu_G(\mathrm{d}x) \\
      &= \sup_{f \in K^p(B, \mu_G \otimes \nu; \mathbb{R}^m)} \int_B | \tilde{\delta}(x, k) f(x, k) | \, (\mu_G \times \nu)(\mathrm{d}(x, k)) \\
      &= \|\tilde{\delta}\|_{\mu_G \otimes \nu, p'} = \|\delta\|_{\mu_G, p'},
    \end{split}
  \end{equation*}
  to make the reasoning more apparent. Since $\mu_G \otimes \nu$ is
  $\sigma$-finite the statement holds even for $p = 1$.
\end{proof}

\subsection{Maximum mean calibration error}

The so-called
\textquote[\citet{kumar18_train_calib_measur_neural_networ}]{correctness score}
$c(x, y)$ of an input $x$ and a class $y$ is defined as
$c(x, y) = \mathbbm{1}_{\{\argmax_{y'} g_{y'}(x)\}}(y)$. It is $1$ if class $y$
is equal to the class that is most likely for input $x$
according to model $g$, and $0$ otherwise. Let
$k \colon [0,1] \times [0,1] \to \mathbb{R}$ be a scalar-valued kernel. Then the
maximum mean calibration error $\MMCE[k,g]$ of a model $g$ with respect to
kernel $k$ is defined\footnote{For illustrative purposes we present a variation
  of the original definition of the $\MMCE$ by
  \citet{kumar18_train_calib_measur_neural_networ}.} as
\begin{multline*}
  \MMCE[k, g] \\
  = {\bigg(\Expect[(c(X, Y) - g_{\max}(X))(c(X',Y') - g_{\max}(X')) k(g_{\max}(X), g_{\max}(X'))]\bigg)}^{1/2},
\end{multline*}
where $(X',Y')$ is an independent copy of $(X,Y)$.

\Cref{ex:mmce_kumar_special} shows that the $\kernelmeasure$ allows exactly the
same analysis of the common notion of calibration as the $\MMCE$ proposed by
\citet{kumar18_train_calib_measur_neural_networ} by applying it to a model that
is reduced to the most confident predictions.

\begin{example}[MMCE as special case]\label{ex:mmce_kumar_special}
  Reduce model $g$ to its most confident predictions by defining a new model
  $\tilde{g}$ with $\tilde{g}(x) \coloneqq (g_{\max}(x), 1 - g_{\max}(x))$. The
  predictions $\tilde{g}(x)$ of this new model can be viewed as a model of the
  conditional probabilities
  $(\Prob[\tilde{Y} = 1 \given X = x], \Prob[\tilde{Y} = 2 \given X = x])$ in a
  classification problem with inputs $X$ and classes
  $\tilde{Y} \coloneqq 2 - c(X, Y)$.\footnote{In the
    words of \citet{vaicenavicius19_evaluat}, $\tilde{g}$ is induced by the
    maximum calibration lens.}

  Let $k \colon [0,1] \times [0,1] \to \mathbb{R}$ be a scalar-valued kernel.
  Define a matrix-valued function
  $\tilde{k} \colon \Delta^2 \times \Delta^2 \to \mathbb{R}^{2 \times 2}$ by
  \begin{equation*}
    \tilde{k}((p_1, p_2), (q_1, q_2)) = \frac{k(p_1, q_1)}{2} \mathbf{I}_2.
  \end{equation*}

  Then by \citet[Example~1 and Theorem~14]{caponnetto08_univer_multi_task_kernel}
  $\tilde{k}$ is a matrix-valued kernel and, if it is continuous, it is universal
  if and only if $k$ is universal. By construction
  $e_{\tilde{Y}} - \tilde{g}(X) = (c(X, Y) - g_{\max}(X)) (1, -1)$, and hence
  \begin{equation*}
    \begin{split}
      \squaredkernelmeasure[\tilde{k}, \tilde{g}] &= \Expect[{(e_{\tilde{Y}} - \tilde{g}(X))}^\intercal \tilde{k}(\tilde{g}(X), \tilde{g}(X')) (e_{\tilde{Y}} - \tilde{g}(X'))] \\
      &= \Expect[(c(X, Y) - g_{\max}(X)) (c(X', Y') - g_{\max}(X')) k(g_{\max}(X), g_{\max}(X'))] \\
      &= \MMCE^2[k, g],
    \end{split}
  \end{equation*}
  where $(X', \tilde{Y}')$ and $(X', Y')$ are independent copies of
  $(X, \tilde{Y})$ and $(X, Y)$, respectively.
\end{example}

\clearpage
\section{Experiments}\label{sec:experiment_app}

The Julia implementation for all experiments is available online at
\url{https://github.com/devmotion/CalibrationPaper}. The code is written and
documented with the literate programming tool
\href{https://github.com/JunoLab/Weave.jl}{Weave.jl}~\citep{pastell17_weave}
and exported to HTML files that include results and figures.

\subsection{Calibration errors}

In our experiments we evaluate the proposed estimators of the
$\squaredkernelmeasure$ and compare them with two estimators of the $\ECE$.

\subsubsection{Expected calibration error}

As commonly done~\citep{broecker07_increas_reliab_reliab_diagr,guo17_calib_moder_neural_networ,vaicenavicius19_evaluat},
we study the $\ECE$ with respect to the total variation distance.

The standard histogram-regression estimator of the $\ECE$ is based on a
partitioning of the probability simplex
\citep{guo17_calib_moder_neural_networ,vaicenavicius19_evaluat}. In our experiments
we use two different partitioning schemes. The first scheme is the commonly
used partitioning into bins of uniform size, based on splitting the predictions
of each class into 10 bins. The other partitioning is data-dependent: the data
set is split iteratively along the median of the class predictions with the
highest variance such that the number of samples in each bin is at least 5.

\subsubsection{Kernel calibration error}

We consider the matrix-valued kernel
$k(x, y) = \exp{(-\|x - y\| / \nu)} \mathbf{I}_m$ with kernel bandwidth
$\nu > 0$. Analogously to the $\ECE$, we take the total variation distance as
distance measure. Moreover, we choose the bandwidth adaptively with the
so-called median heuristic. The median heuristic is a common heuristic that
proposes to set the bandwidth to the median of the pairwise distances of samples
in a, not necessarily separate, validation data set
\citep[see, e.g.,][]{gretton12_kernel_two_sampl_test}.

\subsection{Generative models}\label{sec:generative_models}

Since the considered calibration errors depend only on the predictions and
labels, we specify generative models of labeled predictions $(g(X), Y)$ without
considering $X$. Instead we only specify the distribution of the predictions
$g(X)$ and the conditional distribution of $Y$ given $g(X) = g(x)$. This setup
allows us to design calibrated and uncalibrated models in a straightforward
way, which enables clean numerical evaluations with known calibration errors.

We study the generative model
\begin{align*}
  g(X) &\sim \Dir(\alpha), \\
  Z &\sim \Ber(\pi), \\
  Y \given Z = 1,\, g(X) = \gamma &\sim \Categorical(\beta),\\
  Y \given Z = 0,\, g(X) = \gamma &\sim \Categorical(\gamma),
\end{align*}
with parameters $\alpha \in \mathbb{R}_{>0}^m$, $\beta \in \Delta^m$, and
$\pi \in [0,1]$. The model is calibrated if and only if $\pi = 0$,
since for all labels $y \in \{1,\ldots,m\}$ we obtain
\begin{equation*}
  \Prob[Y = y \given g(X)] = \pi \beta_y + (1 - \pi) g_y(X),
\end{equation*}
and hence $\Delta = \pi (\beta - g(X)) = 0$ almost surely if and only if
$\pi = 0$.

By setting $\alpha = (1,\ldots,1)$ we can model uniformly distributed
predictions, and by decreasing the magnitude of $\alpha$ we can push the
predictions towards the edges of the probability simplex, mimicking the
predictions of a trained model
\citep[cf., e.g.,][]{vaicenavicius19_evaluat}.

\subsubsection{Theoretical expected calibration error}\label{sec:theoretical_ece}

For the considered model, the $\ECE$ with respect to the total variation
distance is
\begin{equation*}
  \begin{split}
    \ECE[\|.\|_{\mathrm{TV}},g] &= \Expect[\|\Delta\|_{\mathrm{TV}}] = \pi \Expect[\|\beta - g(X)\|_{\mathrm{TV}}] = \pi/2 \sum_{i=1}^m \Expect[|\beta_i - g_i(X)|] \\
    &= \frac{\pi}{2} \sum_{i=1}^m \left(\left(\frac{\alpha_i}{\alpha_0} - \beta_i\right) \left(1 - \frac{2 B(\beta_i; \alpha_i, \alpha_0 - \alpha_i)}{B(\alpha_i,\alpha_0-\alpha_i)}\right) + \frac{2\beta_i^{\alpha_i} {(1-\beta_i)}^{\alpha_0-\alpha_i}}{\alpha_0 B(\alpha_i,\alpha_0-\alpha_i)}\right),
  \end{split}
\end{equation*}
where $\alpha_0 \coloneqq \sum_{i=1}^m \alpha_i$ and $B(x; a,b)$ denotes the
incomplete Beta function $\int_0^x t^{a-1} {(1-t)}^{b-1} \,\mathrm{d}t$. By
exploiting $\sum_{i=1}^m \beta_i = 1$, we get
\begin{equation*}
  \ECE[\|.\|_{\mathrm{TV}},g] = \frac{\pi}{\alpha_0} \sum_{i=1}^m \frac{(\alpha_0 \beta_i - \alpha_i)B(\beta_i; \alpha_i, \alpha_0-\alpha_i) + \beta_i^{\alpha_i} {(1-\beta_i)}^{\alpha_0-\alpha_i}}{B(\alpha_i,\alpha_0-\alpha_i)}.
\end{equation*}
Let $I(x; a, b) \coloneqq B(x; a, b) / B(a, b)$ denote the regularized
incomplete Beta function. Due to the identity
$x^a (1-x)^b / B(a, b) = a (I(x; a, b) - I(x; a+1, b)$, we obtain
\begin{equation*}
  \ECE[\|.\|_{\mathrm{TV}},g] = \pi \sum_{i=1}^m \left(\beta_i I(\beta_i; \alpha_i, \alpha_0 - \alpha_i) - \frac{\alpha_i}{\alpha_0} I(\beta_i; \alpha_i + 1, \alpha_0 - \alpha_i)\right).
\end{equation*}

If $\alpha = (a, \ldots, a)$ for some $a > 0$, then
\begin{equation*}
  \ECE[\|.\|_{\mathrm{TV}},g] = \pi \sum_{i=1}^m \left(\beta_i I(\beta_i; a, (m-1)a) - m^{-1} I(\beta_i; a + 1, (m-1)a)\right).
\end{equation*}
If $\beta = e_j$ for some $j \in \{1,\ldots,m\}$ we get
\begin{equation*}
  \begin{split}
  \ECE[\|.\|_{\mathrm{TV}},g] &= \pi \left(I(1; a, (m-1)a) - m^{-1} I(1; a + 1, (m-1)a)\right) \\
  &= \pi(1 - m^{-1}) = \frac{\pi(m-1)}{m},
  \end{split}
\end{equation*}
whereas if $\beta = (1/m, \ldots, 1/m)$ we obtain
\begin{equation*}
  \begin{split}
  \ECE[\|.\|_{\mathrm{TV}},g] &= \pi \left(I(m^{-1}; a, (m-1)a) - I(m^{-1}; a + 1, (m-1)a)\right) \\
  &= \pi \frac{m^{-a}{(1-m^{-1})}^{(m-1)a}}{a B(a, (m-1)a)} = \frac{\pi}{a B(a, (m-1)a)} {\left(\frac{{(m-1)}^{m-1}}{m^m}\right)}^a.
  \end{split}
\end{equation*}
We see that, as the number of classes goes to infinity, the $\ECE$ with respect
to the total variation distance tends to $\pi$ and
$\pi \exp{(-a)} a^{a-1} / \Gamma(a)$, respectively.

\subsubsection{Mean total variation distance}

For the considered generative models, we can compute the mean total variation
distance $\Expect[\|X - X'\|_{\mathrm{TV}}]$, which does not depend on the number of
available samples (but, of course, is usually not available). If $X$ and $X'$
are i.i.d.\ according to $\Dir(\alpha)$ with parameter
$\alpha \in \mathbb{R}^m_{>0}$, then their mean total variation distance is
\begin{equation*}
  \begin{split}
    \Expect[\|X - X'\|_{\mathrm{TV}}] &= 1/2 \sum_{i=1}^m \Expect[|X_i - X'_i|] \\
    &= \sum_{i=1}^m \Expect[X_i - X'_i \given X_i > X'_i] \\
    &= \frac{2 B(\alpha_0, \alpha_0)}{\alpha_0} \sum_{i=1}^m {[B(\alpha_i, \alpha_i) B(\alpha_0 - \alpha_i, \alpha_0 - \alpha_i)]}^{-1},
  \end{split}
\end{equation*}
where $\alpha_0 \coloneqq \sum_{i=1}^m \alpha_i$. We conduct additional
experiments in which we set the kernel bandwidth to the mean total variation
distance.

\clearpage

\subsubsection{Calibration error estimates: Additional figures}\label{sec:additional_estimates}

\begin{figure}[!htbp]
  \begin{center}
    \tikzsetnextfilename{errors_ECE_uniform_alpha_i=1.0_only_firstclass=false}
    \input{figures/errors/ECE_uniform_alpha_i=1.0_only_firstclass=false.tex}
    \caption{Distribution of $\widehat{\ECE}$ with bins of uniform size, evaluated on $10^4$ data sets of 250 labeled
      predictions that are randomly sampled from generative models with
      $\alpha = (1,\ldots,1)$ and $\beta = (1/m, \ldots, 1/m)$.}
  \end{center}
\end{figure}

\begin{figure}[!htbp]
  \begin{center}
    \tikzsetnextfilename{errors_ECE_uniform_alpha_i=0.1_only_firstclass=false}
    \input{figures/errors/ECE_uniform_alpha_i=0.1_only_firstclass=false.tex}
    \caption{Distribution of $\widehat{\ECE}$ with bins of uniform size, evaluated on $10^4$ data sets of 250 labeled
      predictions that are randomly sampled from generative models with
      $\alpha = (0.1,\ldots,0.1)$ and $\beta = (1/m, \ldots, 1/m)$.}
  \end{center}
\end{figure}

\begin{figure}[!htbp]
  \begin{center}
    \tikzsetnextfilename{errors_ECE_uniform_alpha_i=1.0_only_firstclass=true}
    \input{figures/errors/ECE_uniform_alpha_i=1.0_only_firstclass=true.tex}
    \caption{Distribution of $\widehat{\ECE}$ with bins of uniform size, evaluated on $10^4$ data sets of 250 labeled
      predictions that are randomly sampled from generative models with
      $\alpha = (1,\ldots,1)$ and $\beta = (1, 0,\ldots, 0)$.}
  \end{center}
\end{figure}

\begin{figure}[!htbp]
  \begin{center}
    \tikzsetnextfilename{errors_ECE_uniform_alpha_i=0.1_only_firstclass=true}
    \input{figures/errors/ECE_uniform_alpha_i=0.1_only_firstclass=true.tex}
    \caption{Distribution of $\widehat{\ECE}$ with bins of uniform size, evaluated on $10^4$ data sets of 250 labeled
      predictions that are randomly sampled from generative models with
      $\alpha = (0.1,\ldots,0.1)$ and $\beta = (1, 0, \ldots, 0)$.}
  \end{center}
\end{figure}%

\begin{figure}[!htbp]
  \begin{center}
    \tikzsetnextfilename{errors_ECE_dynamic_alpha_i=1.0_only_firstclass=false}
    \input{figures/errors/ECE_dynamic_alpha_i=1.0_only_firstclass=false.tex}
    \caption{Distribution of $\widehat{\ECE}$ with data-dependent bins, evaluated on $10^4$ data sets of 250 labeled
      predictions that are randomly sampled from generative models with
      $\alpha = (1,\ldots,1)$ and $\beta = (1/m, \ldots, 1/m)$.}
  \end{center}
\end{figure}

\begin{figure}[!htbp]
  \begin{center}
    \tikzsetnextfilename{errors_ECE_dynamic_alpha_i=0.1_only_firstclass=false}
    \input{figures/errors/ECE_dynamic_alpha_i=0.1_only_firstclass=false.tex}
    \caption{Distribution of $\widehat{\ECE}$ with data-dependent bins, evaluated on $10^4$ data sets of 250 labeled
      predictions that are randomly sampled from generative models with
      $\alpha = (0.1,\ldots,0.1)$ and $\beta = (1/m, \ldots, 1/m)$.}
  \end{center}
\end{figure}

\begin{figure}[!htbp]
  \begin{center}
    \tikzsetnextfilename{errors_ECE_dynamic_alpha_i=1.0_only_firstclass=true}
    \input{figures/errors/ECE_dynamic_alpha_i=1.0_only_firstclass=true.tex}
    \caption{Distribution of $\widehat{\ECE}$ with data-dependent bins, evaluated on $10^4$ data sets of 250 labeled
      predictions that are randomly sampled from generative models with
      $\alpha = (1,\ldots,1)$ and $\beta = (1, 0,\ldots, 0)$.}
  \end{center}
\end{figure}

\begin{figure}[!htbp]
  \begin{center}
    \tikzsetnextfilename{errors_ECE_dynamic_alpha_i=0.1_only_firstclass=true}
    \input{figures/errors/ECE_dynamic_alpha_i=0.1_only_firstclass=true.tex}
    \caption{Distribution of $\widehat{\ECE}$ with data-dependent bins, evaluated on $10^4$ data sets of 250 labeled
      predictions that are randomly sampled from generative models with
      $\alpha = (0.1,\ldots,0.1)$ and $\beta = (1, 0, \ldots, 0)$.}
  \end{center}
\end{figure}%

\begin{figure}[!htbp]
  \begin{center}
    \tikzsetnextfilename{errors_SKCEb_median_alpha_i=1.0_only_firstclass=false}
    \input{figures/errors/SKCEb_median_alpha_i=1.0_only_firstclass=false.tex}
    \caption{Distribution of $\biasedestimator$ with the median heuristic, evaluated on $10^4$ data sets of 250 labeled
      predictions that are randomly sampled from generative models with
      $\alpha = (1,\ldots,1)$ and $\beta = (1/m, \ldots, 1/m)$.}
  \end{center}
\end{figure}

\begin{figure}[!htbp]
  \begin{center}
    \tikzsetnextfilename{errors_SKCEb_median_alpha_i=0.1_only_firstclass=false}
    \input{figures/errors/SKCEb_median_alpha_i=0.1_only_firstclass=false.tex}
    \caption{Distribution of $\biasedestimator$ with the median heuristic, evaluated on $10^4$ data sets of 250 labeled
      predictions that are randomly sampled from generative models with
      $\alpha = (0.1,\ldots,0.1)$ and $\beta = (1/m, \ldots, 1/m)$.}
  \end{center}
\end{figure}

\begin{figure}[!htbp]
  \begin{center}
    \tikzsetnextfilename{errors_SKCEb_median_alpha_i=1.0_only_firstclass=true}
    \input{figures/errors/SKCEb_median_alpha_i=1.0_only_firstclass=true.tex}
    \caption{Distribution of $\biasedestimator$ with the median heuristic, evaluated on $10^4$ data sets of 250 labeled
      predictions that are randomly sampled from generative models with
      $\alpha = (1,\ldots,1)$ and $\beta = (1, 0,\ldots, 0)$.}
  \end{center}
\end{figure}

\begin{figure}[!htbp]
  \begin{center}
    \tikzsetnextfilename{errors_SKCEb_median_alpha_i=0.1_only_firstclass=true}
    \input{figures/errors/SKCEb_median_alpha_i=0.1_only_firstclass=true.tex}
    \caption{Distribution of $\biasedestimator$ with the median heuristic, evaluated on $10^4$ data sets of 250 labeled
      predictions that are randomly sampled from generative models with
      $\alpha = (0.1,\ldots,0.1)$ and $\beta = (1, 0, \ldots, 0)$.}
  \end{center}
\end{figure}%

\begin{figure}[!htbp]
  \begin{center}
    \tikzsetnextfilename{errors_SKCEb_mean_alpha_i=1.0_only_firstclass=false}
    \input{figures/errors/SKCEb_mean_alpha_i=1.0_only_firstclass=false.tex}
    \caption{Distribution of $\biasedestimator$ with the mean total variation distance, evaluated on $10^4$ data sets of 250 labeled
      predictions that are randomly sampled from generative models with
      $\alpha = (1,\ldots,1)$ and $\beta = (1/m, \ldots, 1/m)$.}
  \end{center}
\end{figure}

\begin{figure}[!htbp]
  \begin{center}
    \tikzsetnextfilename{errors_SKCEb_mean_alpha_i=0.1_only_firstclass=false}
    \input{figures/errors/SKCEb_mean_alpha_i=0.1_only_firstclass=false.tex}
    \caption{Distribution of $\biasedestimator$ with the mean total variation distance, evaluated on $10^4$ data sets of 250 labeled
      predictions that are randomly sampled from generative models with
      $\alpha = (0.1,\ldots,0.1)$ and $\beta = (1/m, \ldots, 1/m)$.}
  \end{center}
\end{figure}

\begin{figure}[!htbp]
  \begin{center}
    \tikzsetnextfilename{errors_SKCEb_mean_alpha_i=1.0_only_firstclass=true}
    \input{figures/errors/SKCEb_mean_alpha_i=1.0_only_firstclass=true.tex}
    \caption{Distribution of $\biasedestimator$ with the mean total variation distance, evaluated on $10^4$ data sets of 250 labeled
      predictions that are randomly sampled from generative models with
      $\alpha = (1,\ldots,1)$ and $\beta = (1, 0,\ldots, 0)$.}
  \end{center}
\end{figure}

\begin{figure}[!htbp]
  \begin{center}
    \tikzsetnextfilename{errors_SKCEb_mean_alpha_i=0.1_only_firstclass=true}
    \input{figures/errors/SKCEb_mean_alpha_i=0.1_only_firstclass=true.tex}
    \caption{Distribution of $\biasedestimator$ with the mean total variation distance, evaluated on $10^4$ data sets of 250 labeled
      predictions that are randomly sampled from generative models with
      $\alpha = (0.1,\ldots,0.1)$ and $\beta = (1, 0, \ldots, 0)$.}
  \end{center}
\end{figure}%

\begin{figure}[!htbp]
  \begin{center}
    \tikzsetnextfilename{errors_SKCEuq_median_alpha_i=1.0_only_firstclass=false}
    \input{figures/errors/SKCEuq_median_alpha_i=1.0_only_firstclass=false.tex}
    \caption{Distribution of $\unbiasedestimator$ with the median heuristic, evaluated on $10^4$ data sets of 250 labeled
      predictions that are randomly sampled from generative models with
      $\alpha = (1,\ldots,1)$ and $\beta = (1/m, \ldots, 1/m)$.}
  \end{center}
\end{figure}

\begin{figure}[!htbp]
  \begin{center}
    \tikzsetnextfilename{errors_SKCEuq_median_alpha_i=0.1_only_firstclass=false}
    \input{figures/errors/SKCEuq_median_alpha_i=0.1_only_firstclass=false.tex}
    \caption{Distribution of $\unbiasedestimator$ with the median heuristic, evaluated on $10^4$ data sets of 250 labeled
      predictions that are randomly sampled from generative models with
      $\alpha = (0.1,\ldots,0.1)$ and $\beta = (1/m, \ldots, 1/m)$.}
  \end{center}
\end{figure}

\begin{figure}[!htbp]
  \begin{center}
    \tikzsetnextfilename{errors_SKCEuq_median_alpha_i=1.0_only_firstclass=true}
    \input{figures/errors/SKCEuq_median_alpha_i=1.0_only_firstclass=true.tex}
    \caption{Distribution of $\unbiasedestimator$ with the median heuristic, evaluated on $10^4$ data sets of 250 labeled
      predictions that are randomly sampled from generative models with
      $\alpha = (1,\ldots,1)$ and $\beta = (1, 0,\ldots, 0)$.}
  \end{center}
\end{figure}

\begin{figure}[!htbp]
  \begin{center}
    \tikzsetnextfilename{errors_SKCEuq_median_alpha_i=0.1_only_firstclass=true}
    \input{figures/errors/SKCEuq_median_alpha_i=0.1_only_firstclass=true.tex}
    \caption{Distribution of $\unbiasedestimator$ with the median heuristic, evaluated on $10^4$ data sets of 250 labeled
      predictions that are randomly sampled from generative models with
      $\alpha = (0.1,\ldots,0.1)$ and $\beta = (1, 0, \ldots, 0)$.}
  \end{center}
\end{figure}%

\begin{figure}[!htbp]
  \begin{center}
    \tikzsetnextfilename{errors_SKCEuq_mean_alpha_i=1.0_only_firstclass=false}
    \input{figures/errors/SKCEuq_mean_alpha_i=1.0_only_firstclass=false.tex}
    \caption{Distribution of $\unbiasedestimator$ with the mean total variation distance, evaluated on $10^4$ data sets of 250 labeled
      predictions that are randomly sampled from generative models with
      $\alpha = (1,\ldots,1)$ and $\beta = (1/m, \ldots, 1/m)$.}
  \end{center}
\end{figure}

\begin{figure}[!htbp]
  \begin{center}
    \tikzsetnextfilename{errors_SKCEuq_mean_alpha_i=0.1_only_firstclass=false}
    \input{figures/errors/SKCEuq_mean_alpha_i=0.1_only_firstclass=false.tex}
    \caption{Distribution of $\unbiasedestimator$ with the mean total variation distance, evaluated on $10^4$ data sets of 250 labeled
      predictions that are randomly sampled from generative models with
      $\alpha = (0.1,\ldots,0.1)$ and $\beta = (1/m, \ldots, 1/m)$.}
  \end{center}
\end{figure}

\begin{figure}[!htbp]
  \begin{center}
    \tikzsetnextfilename{errors_SKCEuq_mean_alpha_i=1.0_only_firstclass=true}
    \input{figures/errors/SKCEuq_mean_alpha_i=1.0_only_firstclass=true.tex}
    \caption{Distribution of $\unbiasedestimator$ with the mean total variation distance, evaluated on $10^4$ data sets of 250 labeled
      predictions that are randomly sampled from generative models with
      $\alpha = (1,\ldots,1)$ and $\beta = (1, 0,\ldots, 0)$.}
  \end{center}
\end{figure}

\begin{figure}[!htbp]
  \begin{center}
    \tikzsetnextfilename{errors_SKCEuq_mean_alpha_i=0.1_only_firstclass=true}
    \input{figures/errors/SKCEuq_mean_alpha_i=0.1_only_firstclass=true.tex}
    \caption{Distribution of $\unbiasedestimator$ with the mean total variation distance, evaluated on $10^4$ data sets of 250 labeled
      predictions that are randomly sampled from generative models with
      $\alpha = (0.1,\ldots,0.1)$ and $\beta = (1, 0, \ldots, 0)$.}
  \end{center}
\end{figure}%

\begin{figure}[!htbp]
  \begin{center}
    \tikzsetnextfilename{errors_SKCEul_median_alpha_i=1.0_only_firstclass=false}
    \input{figures/errors/SKCEul_median_alpha_i=1.0_only_firstclass=false.tex}
    \caption{Distribution of $\linearestimator$ with the median heuristic, evaluated on $10^4$ data sets of 250 labeled
      predictions that are randomly sampled from generative models with
      $\alpha = (1,\ldots,1)$ and $\beta = (1/m, \ldots, 1/m)$.}
  \end{center}
\end{figure}

\begin{figure}[!htbp]
  \begin{center}
    \tikzsetnextfilename{errors_SKCEul_median_alpha_i=0.1_only_firstclass=false}
    \input{figures/errors/SKCEul_median_alpha_i=0.1_only_firstclass=false.tex}
    \caption{Distribution of $\linearestimator$ with the median heuristic, evaluated on $10^4$ data sets of 250 labeled
      predictions that are randomly sampled from generative models with
      $\alpha = (0.1,\ldots,0.1)$ and $\beta = (1/m, \ldots, 1/m)$.}
  \end{center}
\end{figure}

\begin{figure}[!htbp]
  \begin{center}
    \tikzsetnextfilename{errors_SKCEul_median_alpha_i=1.0_only_firstclass=true}
    \input{figures/errors/SKCEul_median_alpha_i=1.0_only_firstclass=true.tex}
    \caption{Distribution of $\linearestimator$ with the median heuristic, evaluated on $10^4$ data sets of 250 labeled
      predictions that are randomly sampled from generative models with
      $\alpha = (1,\ldots,1)$ and $\beta = (1, 0,\ldots, 0)$.}
  \end{center}
\end{figure}

\begin{figure}[!htbp]
  \begin{center}
    \tikzsetnextfilename{errors_SKCEul_median_alpha_i=0.1_only_firstclass=true}
    \input{figures/errors/SKCEul_median_alpha_i=0.1_only_firstclass=true.tex}
    \caption{Distribution of $\linearestimator$ with the median heuristic, evaluated on $10^4$ data sets of 250 labeled
      predictions that are randomly sampled from generative models with
      $\alpha = (0.1,\ldots,0.1)$ and $\beta = (1, 0, \ldots, 0)$.}
  \end{center}
\end{figure}%

\begin{figure}[!htbp]
  \begin{center}
    \tikzsetnextfilename{errors_SKCEul_mean_alpha_i=1.0_only_firstclass=false}
    \input{figures/errors/SKCEul_mean_alpha_i=1.0_only_firstclass=false.tex}
    \caption{Distribution of $\linearestimator$ with the mean total variation distance, evaluated on $10^4$ data sets of 250 labeled
      predictions that are randomly sampled from generative models with
      $\alpha = (1,\ldots,1)$ and $\beta = (1/m, \ldots, 1/m)$.}
  \end{center}
\end{figure}

\begin{figure}[!htbp]
  \begin{center}
    \tikzsetnextfilename{errors_SKCEul_mean_alpha_i=0.1_only_firstclass=false}
    \input{figures/errors/SKCEul_mean_alpha_i=0.1_only_firstclass=false.tex}
    \caption{Distribution of $\linearestimator$ with the mean total variation distance, evaluated on $10^4$ data sets of 250 labeled
      predictions that are randomly sampled from generative models with
      $\alpha = (0.1,\ldots,0.1)$ and $\beta = (1/m, \ldots, 1/m)$.}
  \end{center}
\end{figure}

\begin{figure}[!htbp]
  \begin{center}
    \tikzsetnextfilename{errors_SKCEul_mean_alpha_i=1.0_only_firstclass=true}
    \input{figures/errors/SKCEul_mean_alpha_i=1.0_only_firstclass=true.tex}
    \caption{Distribution of $\linearestimator$ with the mean total variation distance, evaluated on $10^4$ data sets of 250 labeled
      predictions that are randomly sampled from generative models with
      $\alpha = (1,\ldots,1)$ and $\beta = (1, 0,\ldots, 0)$.}
  \end{center}
\end{figure}

\begin{figure}[!htbp]
  \begin{center}
    \tikzsetnextfilename{errors_SKCEul_mean_alpha_i=0.1_only_firstclass=true}
    \input{figures/errors/SKCEul_mean_alpha_i=0.1_only_firstclass=true.tex}
    \caption{Distribution of $\linearestimator$ with the mean total variation distance, evaluated on $10^4$ data sets of 250 labeled
      predictions that are randomly sampled from generative models with
      $\alpha = (0.1,\ldots,0.1)$ and $\beta = (1, 0, \ldots, 0)$.}
  \end{center}
\end{figure}%

\clearpage

\subsubsection{Calibration tests: Additional figures}\label{sec:additional_tests}

\begin{figure}[!htbp]
  \begin{center}
    \tikzsetnextfilename{pvalues_ECE_uniform_alpha_i=1.0_only_firstclass=false}
    \input{figures/pvalues/ECE_uniform_alpha_i=1.0_only_firstclass=false.tex}
    \caption{Empirical test error versus significance level for approximations of the p-value with consistency resampling based
on $\widehat{\ECE}$ with bins of uniform size,
      evaluated on $10^4$ data sets of 250 labeled predictions that
      are randomly sampled from generative models with
      $\alpha = (1,\ldots,1)$ and $\beta = (1/m, \ldots, 1/m)$.}
  \end{center}
\end{figure}

\begin{figure}[!htbp]
  \begin{center}
    \tikzsetnextfilename{pvalues_ECE_uniform_alpha_i=0.1_only_firstclass=false}
    \input{figures/pvalues/ECE_uniform_alpha_i=0.1_only_firstclass=false.tex}
    \caption{Empirical test error versus significance level for approximations of the p-value with consistency resampling based
on $\widehat{\ECE}$ with bins of uniform size,
      evaluated on $10^4$ data sets of 250 labeled predictions that
      are randomly sampled from generative models with
      $\alpha = (0.1,\ldots,0.1)$ and $\beta = (1/m, \ldots, 1/m)$.}
  \end{center}
\end{figure}

\begin{figure}[!htbp]
  \begin{center}
    \tikzsetnextfilename{pvalues_ECE_uniform_alpha_i=1.0_only_firstclass=true}
    \input{figures/pvalues/ECE_uniform_alpha_i=1.0_only_firstclass=true.tex}
    \caption{Empirical test error versus significance level for approximations of the p-value with consistency resampling based
on $\widehat{\ECE}$ with bins of uniform size,
      evaluated on $10^4$ data sets of 250 labeled predictions that
      are randomly sampled from generative models with
      $\alpha = (1,\ldots,1)$ and $\beta = (1, 0,\ldots, 0)$.}
  \end{center}
\end{figure}

\begin{figure}[!htbp]
  \begin{center}
    \tikzsetnextfilename{pvalues_ECE_uniform_alpha_i=0.1_only_firstclass=true}
    \input{figures/pvalues/ECE_uniform_alpha_i=0.1_only_firstclass=true.tex}
    \caption{Empirical test error versus significance level for approximations of the p-value with consistency resampling based
on $\widehat{\ECE}$ with bins of uniform size,
      evaluated on $10^4$ data sets of 250 labeled predictions that
      are randomly sampled from generative models with
      $\alpha = (0.1,\ldots,0.1)$ and $\beta = (1, 0, \ldots, 0)$.}
  \end{center}
\end{figure}%

\begin{figure}[!htbp]
  \begin{center}
    \tikzsetnextfilename{pvalues_ECE_dynamic_alpha_i=1.0_only_firstclass=false}
    \input{figures/pvalues/ECE_dynamic_alpha_i=1.0_only_firstclass=false.tex}
    \caption{Empirical test error versus significance level for approximations of the p-value with consistency resampling based
on $\widehat{\ECE}$ with data-dependent bins,
      evaluated on $10^4$ data sets of 250 labeled predictions that
      are randomly sampled from generative models with
      $\alpha = (1,\ldots,1)$ and $\beta = (1/m, \ldots, 1/m)$.}
  \end{center}
\end{figure}

\begin{figure}[!htbp]
  \begin{center}
    \tikzsetnextfilename{pvalues_ECE_dynamic_alpha_i=0.1_only_firstclass=false}
    \input{figures/pvalues/ECE_dynamic_alpha_i=0.1_only_firstclass=false.tex}
    \caption{Empirical test error versus significance level for approximations of the p-value with consistency resampling based
on $\widehat{\ECE}$ with data-dependent bins,
      evaluated on $10^4$ data sets of 250 labeled predictions that
      are randomly sampled from generative models with
      $\alpha = (0.1,\ldots,0.1)$ and $\beta = (1/m, \ldots, 1/m)$.}
  \end{center}
\end{figure}

\begin{figure}[!htbp]
  \begin{center}
    \tikzsetnextfilename{pvalues_ECE_dynamic_alpha_i=1.0_only_firstclass=true}
    \input{figures/pvalues/ECE_dynamic_alpha_i=1.0_only_firstclass=true.tex}
    \caption{Empirical test error versus significance level for approximations of the p-value with consistency resampling based
on $\widehat{\ECE}$ with data-dependent bins,
      evaluated on $10^4$ data sets of 250 labeled predictions that
      are randomly sampled from generative models with
      $\alpha = (1,\ldots,1)$ and $\beta = (1, 0,\ldots, 0)$.}
  \end{center}
\end{figure}

\begin{figure}[!htbp]
  \begin{center}
    \tikzsetnextfilename{pvalues_ECE_dynamic_alpha_i=0.1_only_firstclass=true}
    \input{figures/pvalues/ECE_dynamic_alpha_i=0.1_only_firstclass=true.tex}
    \caption{Empirical test error versus significance level for approximations of the p-value with consistency resampling based
on $\widehat{\ECE}$ with data-dependent bins,
      evaluated on $10^4$ data sets of 250 labeled predictions that
      are randomly sampled from generative models with
      $\alpha = (0.1,\ldots,0.1)$ and $\beta = (1, 0, \ldots, 0)$.}
  \end{center}
\end{figure}%

\begin{figure}[!htbp]
  \begin{center}
    \tikzsetnextfilename{pvalues_SKCEb_median_distribution_free_alpha_i=1.0_only_firstclass=false}
    \input{figures/pvalues/SKCEb_median_distribution_free_alpha_i=1.0_only_firstclass=false.tex}
    \caption{Empirical test error versus significance level for distribution-free bounds of the p-value based
on $\biasedestimator$ with the median heuristic,
      evaluated on $10^4$ data sets of 250 labeled predictions that
      are randomly sampled from generative models with
      $\alpha = (1,\ldots,1)$ and $\beta = (1/m, \ldots, 1/m)$.}
  \end{center}
\end{figure}

\begin{figure}[!htbp]
  \begin{center}
    \tikzsetnextfilename{pvalues_SKCEb_median_distribution_free_alpha_i=0.1_only_firstclass=false}
    \input{figures/pvalues/SKCEb_median_distribution_free_alpha_i=0.1_only_firstclass=false.tex}
    \caption{Empirical test error versus significance level for distribution-free bounds of the p-value based
on $\biasedestimator$ with the median heuristic,
      evaluated on $10^4$ data sets of 250 labeled predictions that
      are randomly sampled from generative models with
      $\alpha = (0.1,\ldots,0.1)$ and $\beta = (1/m, \ldots, 1/m)$.}
  \end{center}
\end{figure}

\begin{figure}[!htbp]
  \begin{center}
    \tikzsetnextfilename{pvalues_SKCEb_median_distribution_free_alpha_i=1.0_only_firstclass=true}
    \input{figures/pvalues/SKCEb_median_distribution_free_alpha_i=1.0_only_firstclass=true.tex}
    \caption{Empirical test error versus significance level for distribution-free bounds of the p-value based
on $\biasedestimator$ with the median heuristic,
      evaluated on $10^4$ data sets of 250 labeled predictions that
      are randomly sampled from generative models with
      $\alpha = (1,\ldots,1)$ and $\beta = (1, 0,\ldots, 0)$.}
  \end{center}
\end{figure}

\begin{figure}[!htbp]
  \begin{center}
    \tikzsetnextfilename{pvalues_SKCEb_median_distribution_free_alpha_i=0.1_only_firstclass=true}
    \input{figures/pvalues/SKCEb_median_distribution_free_alpha_i=0.1_only_firstclass=true.tex}
    \caption{Empirical test error versus significance level for distribution-free bounds of the p-value based
on $\biasedestimator$ with the median heuristic,
      evaluated on $10^4$ data sets of 250 labeled predictions that
      are randomly sampled from generative models with
      $\alpha = (0.1,\ldots,0.1)$ and $\beta = (1, 0, \ldots, 0)$.}
  \end{center}
\end{figure}%

\begin{figure}[!htbp]
  \begin{center}
    \tikzsetnextfilename{pvalues_SKCEb_mean_distribution_free_alpha_i=1.0_only_firstclass=false}
    \input{figures/pvalues/SKCEb_mean_distribution_free_alpha_i=1.0_only_firstclass=false.tex}
    \caption{Empirical test error versus significance level for distribution-free bounds of the p-value based
on $\biasedestimator$ with the mean total variation distance,
      evaluated on $10^4$ data sets of 250 labeled predictions that
      are randomly sampled from generative models with
      $\alpha = (1,\ldots,1)$ and $\beta = (1/m, \ldots, 1/m)$.}
  \end{center}
\end{figure}

\begin{figure}[!htbp]
  \begin{center}
    \tikzsetnextfilename{pvalues_SKCEb_mean_distribution_free_alpha_i=0.1_only_firstclass=false}
    \input{figures/pvalues/SKCEb_mean_distribution_free_alpha_i=0.1_only_firstclass=false.tex}
    \caption{Empirical test error versus significance level for distribution-free bounds of the p-value based
on $\biasedestimator$ with the mean total variation distance,
      evaluated on $10^4$ data sets of 250 labeled predictions that
      are randomly sampled from generative models with
      $\alpha = (0.1,\ldots,0.1)$ and $\beta = (1/m, \ldots, 1/m)$.}
  \end{center}
\end{figure}

\begin{figure}[!htbp]
  \begin{center}
    \tikzsetnextfilename{pvalues_SKCEb_mean_distribution_free_alpha_i=1.0_only_firstclass=true}
    \input{figures/pvalues/SKCEb_mean_distribution_free_alpha_i=1.0_only_firstclass=true.tex}
    \caption{Empirical test error versus significance level for distribution-free bounds of the p-value based
on $\biasedestimator$ with the mean total variation distance,
      evaluated on $10^4$ data sets of 250 labeled predictions that
      are randomly sampled from generative models with
      $\alpha = (1,\ldots,1)$ and $\beta = (1, 0,\ldots, 0)$.}
  \end{center}
\end{figure}

\begin{figure}[!htbp]
  \begin{center}
    \tikzsetnextfilename{pvalues_SKCEb_mean_distribution_free_alpha_i=0.1_only_firstclass=true}
    \input{figures/pvalues/SKCEb_mean_distribution_free_alpha_i=0.1_only_firstclass=true.tex}
    \caption{Empirical test error versus significance level for distribution-free bounds of the p-value based
on $\biasedestimator$ with the mean total variation distance,
      evaluated on $10^4$ data sets of 250 labeled predictions that
      are randomly sampled from generative models with
      $\alpha = (0.1,\ldots,0.1)$ and $\beta = (1, 0, \ldots, 0)$.}
  \end{center}
\end{figure}%

\begin{figure}[!htbp]
  \begin{center}
    \tikzsetnextfilename{pvalues_SKCEuq_median_distribution_free_alpha_i=1.0_only_firstclass=false}
    \input{figures/pvalues/SKCEuq_median_distribution_free_alpha_i=1.0_only_firstclass=false.tex}
    \caption{Empirical test error versus significance level for distribution-free bounds of the p-value based
on $\unbiasedestimator$ with the median heuristic,
      evaluated on $10^4$ data sets of 250 labeled predictions that
      are randomly sampled from generative models with
      $\alpha = (1,\ldots,1)$ and $\beta = (1/m, \ldots, 1/m)$.}
  \end{center}
\end{figure}

\begin{figure}[!htbp]
  \begin{center}
    \tikzsetnextfilename{pvalues_SKCEuq_median_distribution_free_alpha_i=0.1_only_firstclass=false}
    \input{figures/pvalues/SKCEuq_median_distribution_free_alpha_i=0.1_only_firstclass=false.tex}
    \caption{Empirical test error versus significance level for distribution-free bounds of the p-value based
on $\unbiasedestimator$ with the median heuristic,
      evaluated on $10^4$ data sets of 250 labeled predictions that
      are randomly sampled from generative models with
      $\alpha = (0.1,\ldots,0.1)$ and $\beta = (1/m, \ldots, 1/m)$.}
  \end{center}
\end{figure}

\begin{figure}[!htbp]
  \begin{center}
    \tikzsetnextfilename{pvalues_SKCEuq_median_distribution_free_alpha_i=1.0_only_firstclass=true}
    \input{figures/pvalues/SKCEuq_median_distribution_free_alpha_i=1.0_only_firstclass=true.tex}
    \caption{Empirical test error versus significance level for distribution-free bounds of the p-value based
on $\unbiasedestimator$ with the median heuristic,
      evaluated on $10^4$ data sets of 250 labeled predictions that
      are randomly sampled from generative models with
      $\alpha = (1,\ldots,1)$ and $\beta = (1, 0,\ldots, 0)$.}
  \end{center}
\end{figure}

\begin{figure}[!htbp]
  \begin{center}
    \tikzsetnextfilename{pvalues_SKCEuq_median_distribution_free_alpha_i=0.1_only_firstclass=true}
    \input{figures/pvalues/SKCEuq_median_distribution_free_alpha_i=0.1_only_firstclass=true.tex}
    \caption{Empirical test error versus significance level for distribution-free bounds of the p-value based
on $\unbiasedestimator$ with the median heuristic,
      evaluated on $10^4$ data sets of 250 labeled predictions that
      are randomly sampled from generative models with
      $\alpha = (0.1,\ldots,0.1)$ and $\beta = (1, 0, \ldots, 0)$.}
  \end{center}
\end{figure}%

\begin{figure}[!htbp]
  \begin{center}
    \tikzsetnextfilename{pvalues_SKCEuq_mean_distribution_free_alpha_i=1.0_only_firstclass=false}
    \input{figures/pvalues/SKCEuq_mean_distribution_free_alpha_i=1.0_only_firstclass=false.tex}
    \caption{Empirical test error versus significance level for distribution-free bounds of the p-value based
on $\unbiasedestimator$ with the mean total variation distance,
      evaluated on $10^4$ data sets of 250 labeled predictions that
      are randomly sampled from generative models with
      $\alpha = (1,\ldots,1)$ and $\beta = (1/m, \ldots, 1/m)$.}
  \end{center}
\end{figure}

\begin{figure}[!htbp]
  \begin{center}
    \tikzsetnextfilename{pvalues_SKCEuq_mean_distribution_free_alpha_i=0.1_only_firstclass=false}
    \input{figures/pvalues/SKCEuq_mean_distribution_free_alpha_i=0.1_only_firstclass=false.tex}
    \caption{Empirical test error versus significance level for distribution-free bounds of the p-value based
on $\unbiasedestimator$ with the mean total variation distance,
      evaluated on $10^4$ data sets of 250 labeled predictions that
      are randomly sampled from generative models with
      $\alpha = (0.1,\ldots,0.1)$ and $\beta = (1/m, \ldots, 1/m)$.}
  \end{center}
\end{figure}

\begin{figure}[!htbp]
  \begin{center}
    \tikzsetnextfilename{pvalues_SKCEuq_mean_distribution_free_alpha_i=1.0_only_firstclass=true}
    \input{figures/pvalues/SKCEuq_mean_distribution_free_alpha_i=1.0_only_firstclass=true.tex}
    \caption{Empirical test error versus significance level for distribution-free bounds of the p-value based
on $\unbiasedestimator$ with the mean total variation distance,
      evaluated on $10^4$ data sets of 250 labeled predictions that
      are randomly sampled from generative models with
      $\alpha = (1,\ldots,1)$ and $\beta = (1, 0,\ldots, 0)$.}
  \end{center}
\end{figure}

\begin{figure}[!htbp]
  \begin{center}
    \tikzsetnextfilename{pvalues_SKCEuq_mean_distribution_free_alpha_i=0.1_only_firstclass=true}
    \input{figures/pvalues/SKCEuq_mean_distribution_free_alpha_i=0.1_only_firstclass=true.tex}
    \caption{Empirical test error versus significance level for distribution-free bounds of the p-value based
on $\unbiasedestimator$ with the mean total variation distance,
      evaluated on $10^4$ data sets of 250 labeled predictions that
      are randomly sampled from generative models with
      $\alpha = (0.1,\ldots,0.1)$ and $\beta = (1, 0, \ldots, 0)$.}
  \end{center}
\end{figure}%

\begin{figure}[!htbp]
  \begin{center}
    \tikzsetnextfilename{pvalues_SKCEul_median_distribution_free_alpha_i=1.0_only_firstclass=false}
    \input{figures/pvalues/SKCEul_median_distribution_free_alpha_i=1.0_only_firstclass=false.tex}
    \caption{Empirical test error versus significance level for distribution-free bounds of the p-value based
on $\linearestimator$ with the median heuristic,
      evaluated on $10^4$ data sets of 250 labeled predictions that
      are randomly sampled from generative models with
      $\alpha = (1,\ldots,1)$ and $\beta = (1/m, \ldots, 1/m)$.}
  \end{center}
\end{figure}

\begin{figure}[!htbp]
  \begin{center}
    \tikzsetnextfilename{pvalues_SKCEul_median_distribution_free_alpha_i=0.1_only_firstclass=false}
    \input{figures/pvalues/SKCEul_median_distribution_free_alpha_i=0.1_only_firstclass=false.tex}
    \caption{Empirical test error versus significance level for distribution-free bounds of the p-value based
on $\linearestimator$ with the median heuristic,
      evaluated on $10^4$ data sets of 250 labeled predictions that
      are randomly sampled from generative models with
      $\alpha = (0.1,\ldots,0.1)$ and $\beta = (1/m, \ldots, 1/m)$.}
  \end{center}
\end{figure}

\begin{figure}[!htbp]
  \begin{center}
    \tikzsetnextfilename{pvalues_SKCEul_median_distribution_free_alpha_i=1.0_only_firstclass=true}
    \input{figures/pvalues/SKCEul_median_distribution_free_alpha_i=1.0_only_firstclass=true.tex}
    \caption{Empirical test error versus significance level for distribution-free bounds of the p-value based
on $\linearestimator$ with the median heuristic,
      evaluated on $10^4$ data sets of 250 labeled predictions that
      are randomly sampled from generative models with
      $\alpha = (1,\ldots,1)$ and $\beta = (1, 0,\ldots, 0)$.}
  \end{center}
\end{figure}

\begin{figure}[!htbp]
  \begin{center}
    \tikzsetnextfilename{pvalues_SKCEul_median_distribution_free_alpha_i=0.1_only_firstclass=true}
    \input{figures/pvalues/SKCEul_median_distribution_free_alpha_i=0.1_only_firstclass=true.tex}
    \caption{Empirical test error versus significance level for distribution-free bounds of the p-value based
on $\linearestimator$ with the median heuristic,
      evaluated on $10^4$ data sets of 250 labeled predictions that
      are randomly sampled from generative models with
      $\alpha = (0.1,\ldots,0.1)$ and $\beta = (1, 0, \ldots, 0)$.}
  \end{center}
\end{figure}%

\begin{figure}[!htbp]
  \begin{center}
    \tikzsetnextfilename{pvalues_SKCEul_mean_distribution_free_alpha_i=1.0_only_firstclass=false}
    \input{figures/pvalues/SKCEul_mean_distribution_free_alpha_i=1.0_only_firstclass=false.tex}
    \caption{Empirical test error versus significance level for distribution-free bounds of the p-value based
on $\linearestimator$ with the mean total variation distance,
      evaluated on $10^4$ data sets of 250 labeled predictions that
      are randomly sampled from generative models with
      $\alpha = (1,\ldots,1)$ and $\beta = (1/m, \ldots, 1/m)$.}
  \end{center}
\end{figure}

\begin{figure}[!htbp]
  \begin{center}
    \tikzsetnextfilename{pvalues_SKCEul_mean_distribution_free_alpha_i=0.1_only_firstclass=false}
    \input{figures/pvalues/SKCEul_mean_distribution_free_alpha_i=0.1_only_firstclass=false.tex}
    \caption{Empirical test error versus significance level for distribution-free bounds of the p-value based
on $\linearestimator$ with the mean total variation distance,
      evaluated on $10^4$ data sets of 250 labeled predictions that
      are randomly sampled from generative models with
      $\alpha = (0.1,\ldots,0.1)$ and $\beta = (1/m, \ldots, 1/m)$.}
  \end{center}
\end{figure}

\begin{figure}[!htbp]
  \begin{center}
    \tikzsetnextfilename{pvalues_SKCEul_mean_distribution_free_alpha_i=1.0_only_firstclass=true}
    \input{figures/pvalues/SKCEul_mean_distribution_free_alpha_i=1.0_only_firstclass=true.tex}
    \caption{Empirical test error versus significance level for distribution-free bounds of the p-value based
on $\linearestimator$ with the mean total variation distance,
      evaluated on $10^4$ data sets of 250 labeled predictions that
      are randomly sampled from generative models with
      $\alpha = (1,\ldots,1)$ and $\beta = (1, 0,\ldots, 0)$.}
  \end{center}
\end{figure}

\begin{figure}[!htbp]
  \begin{center}
    \tikzsetnextfilename{pvalues_SKCEul_mean_distribution_free_alpha_i=0.1_only_firstclass=true}
    \input{figures/pvalues/SKCEul_mean_distribution_free_alpha_i=0.1_only_firstclass=true.tex}
    \caption{Empirical test error versus significance level for distribution-free bounds of the p-value based
on $\linearestimator$ with the mean total variation distance,
      evaluated on $10^4$ data sets of 250 labeled predictions that
      are randomly sampled from generative models with
      $\alpha = (0.1,\ldots,0.1)$ and $\beta = (1, 0, \ldots, 0)$.}
  \end{center}
\end{figure}%

\begin{figure}[!htbp]
  \begin{center}
    \tikzsetnextfilename{pvalues_SKCEuq_median_asymptotic_alpha_i=1.0_only_firstclass=false}
    \input{figures/pvalues/SKCEuq_median_asymptotic_alpha_i=1.0_only_firstclass=false.tex}
    \caption{Empirical test error versus significance level for asymptotic approximations of the p-value based on
$\unbiasedestimator$ with the median heuristic,
      evaluated on $10^4$ data sets of 250 labeled predictions that
      are randomly sampled from generative models with
      $\alpha = (1,\ldots,1)$ and $\beta = (1/m, \ldots, 1/m)$.}
  \end{center}
\end{figure}

\begin{figure}[!htbp]
  \begin{center}
    \tikzsetnextfilename{pvalues_SKCEuq_median_asymptotic_alpha_i=0.1_only_firstclass=false}
    \input{figures/pvalues/SKCEuq_median_asymptotic_alpha_i=0.1_only_firstclass=false.tex}
    \caption{Empirical test error versus significance level for asymptotic approximations of the p-value based on
$\unbiasedestimator$ with the median heuristic,
      evaluated on $10^4$ data sets of 250 labeled predictions that
      are randomly sampled from generative models with
      $\alpha = (0.1,\ldots,0.1)$ and $\beta = (1/m, \ldots, 1/m)$.}
  \end{center}
\end{figure}

\begin{figure}[!htbp]
  \begin{center}
    \tikzsetnextfilename{pvalues_SKCEuq_median_asymptotic_alpha_i=1.0_only_firstclass=true}
    \input{figures/pvalues/SKCEuq_median_asymptotic_alpha_i=1.0_only_firstclass=true.tex}
    \caption{Empirical test error versus significance level for asymptotic approximations of the p-value based on
$\unbiasedestimator$ with the median heuristic,
      evaluated on $10^4$ data sets of 250 labeled predictions that
      are randomly sampled from generative models with
      $\alpha = (1,\ldots,1)$ and $\beta = (1, 0,\ldots, 0)$.}
  \end{center}
\end{figure}

\begin{figure}[!htbp]
  \begin{center}
    \tikzsetnextfilename{pvalues_SKCEuq_median_asymptotic_alpha_i=0.1_only_firstclass=true}
    \input{figures/pvalues/SKCEuq_median_asymptotic_alpha_i=0.1_only_firstclass=true.tex}
    \caption{Empirical test error versus significance level for asymptotic approximations of the p-value based on
$\unbiasedestimator$ with the median heuristic,
      evaluated on $10^4$ data sets of 250 labeled predictions that
      are randomly sampled from generative models with
      $\alpha = (0.1,\ldots,0.1)$ and $\beta = (1, 0, \ldots, 0)$.}
  \end{center}
\end{figure}%

\begin{figure}[!htbp]
  \begin{center}
    \tikzsetnextfilename{pvalues_SKCEuq_mean_asymptotic_alpha_i=1.0_only_firstclass=false}
    \input{figures/pvalues/SKCEuq_mean_asymptotic_alpha_i=1.0_only_firstclass=false.tex}
    \caption{Empirical test error versus significance level for asymptotic approximations of the p-value based on
$\unbiasedestimator$ with the mean total variation distance,
      evaluated on $10^4$ data sets of 250 labeled predictions that
      are randomly sampled from generative models with
      $\alpha = (1,\ldots,1)$ and $\beta = (1/m, \ldots, 1/m)$.}
  \end{center}
\end{figure}

\begin{figure}[!htbp]
  \begin{center}
    \tikzsetnextfilename{pvalues_SKCEuq_mean_asymptotic_alpha_i=0.1_only_firstclass=false}
    \input{figures/pvalues/SKCEuq_mean_asymptotic_alpha_i=0.1_only_firstclass=false.tex}
    \caption{Empirical test error versus significance level for asymptotic approximations of the p-value based on
$\unbiasedestimator$ with the mean total variation distance,
      evaluated on $10^4$ data sets of 250 labeled predictions that
      are randomly sampled from generative models with
      $\alpha = (0.1,\ldots,0.1)$ and $\beta = (1/m, \ldots, 1/m)$.}
  \end{center}
\end{figure}

\begin{figure}[!htbp]
  \begin{center}
    \tikzsetnextfilename{pvalues_SKCEuq_mean_asymptotic_alpha_i=1.0_only_firstclass=true}
    \input{figures/pvalues/SKCEuq_mean_asymptotic_alpha_i=1.0_only_firstclass=true.tex}
    \caption{Empirical test error versus significance level for asymptotic approximations of the p-value based on
$\unbiasedestimator$ with the mean total variation distance,
      evaluated on $10^4$ data sets of 250 labeled predictions that
      are randomly sampled from generative models with
      $\alpha = (1,\ldots,1)$ and $\beta = (1, 0,\ldots, 0)$.}
  \end{center}
\end{figure}

\begin{figure}[!htbp]
  \begin{center}
    \tikzsetnextfilename{pvalues_SKCEuq_mean_asymptotic_alpha_i=0.1_only_firstclass=true}
    \input{figures/pvalues/SKCEuq_mean_asymptotic_alpha_i=0.1_only_firstclass=true.tex}
    \caption{Empirical test error versus significance level for asymptotic approximations of the p-value based on
$\unbiasedestimator$ with the mean total variation distance,
      evaluated on $10^4$ data sets of 250 labeled predictions that
      are randomly sampled from generative models with
      $\alpha = (0.1,\ldots,0.1)$ and $\beta = (1, 0, \ldots, 0)$.}
  \end{center}
\end{figure}%

\begin{figure}[!htbp]
  \begin{center}
    \tikzsetnextfilename{pvalues_SKCEul_median_asymptotic_alpha_i=1.0_only_firstclass=false}
    \input{figures/pvalues/SKCEul_median_asymptotic_alpha_i=1.0_only_firstclass=false.tex}
    \caption{Empirical test error versus significance level for asymptotic approximations of the p-value based on
$\linearestimator$ with the median heuristic,
      evaluated on $10^4$ data sets of 250 labeled predictions that
      are randomly sampled from generative models with
      $\alpha = (1,\ldots,1)$ and $\beta = (1/m, \ldots, 1/m)$.}
  \end{center}
\end{figure}

\begin{figure}[!htbp]
  \begin{center}
    \tikzsetnextfilename{pvalues_SKCEul_median_asymptotic_alpha_i=0.1_only_firstclass=false}
    \input{figures/pvalues/SKCEul_median_asymptotic_alpha_i=0.1_only_firstclass=false.tex}
    \caption{Empirical test error versus significance level for asymptotic approximations of the p-value based on
$\linearestimator$ with the median heuristic,
      evaluated on $10^4$ data sets of 250 labeled predictions that
      are randomly sampled from generative models with
      $\alpha = (0.1,\ldots,0.1)$ and $\beta = (1/m, \ldots, 1/m)$.}
  \end{center}
\end{figure}

\begin{figure}[!htbp]
  \begin{center}
    \tikzsetnextfilename{pvalues_SKCEul_median_asymptotic_alpha_i=1.0_only_firstclass=true}
    \input{figures/pvalues/SKCEul_median_asymptotic_alpha_i=1.0_only_firstclass=true.tex}
    \caption{Empirical test error versus significance level for asymptotic approximations of the p-value based on
$\linearestimator$ with the median heuristic,
      evaluated on $10^4$ data sets of 250 labeled predictions that
      are randomly sampled from generative models with
      $\alpha = (1,\ldots,1)$ and $\beta = (1, 0,\ldots, 0)$.}
  \end{center}
\end{figure}

\begin{figure}[!htbp]
  \begin{center}
    \tikzsetnextfilename{pvalues_SKCEul_median_asymptotic_alpha_i=0.1_only_firstclass=true}
    \input{figures/pvalues/SKCEul_median_asymptotic_alpha_i=0.1_only_firstclass=true.tex}
    \caption{Empirical test error versus significance level for asymptotic approximations of the p-value based on
$\linearestimator$ with the median heuristic,
      evaluated on $10^4$ data sets of 250 labeled predictions that
      are randomly sampled from generative models with
      $\alpha = (0.1,\ldots,0.1)$ and $\beta = (1, 0, \ldots, 0)$.}
  \end{center}
\end{figure}%

\begin{figure}[!htbp]
  \begin{center}
    \tikzsetnextfilename{pvalues_SKCEul_mean_asymptotic_alpha_i=1.0_only_firstclass=false}
    \input{figures/pvalues/SKCEul_mean_asymptotic_alpha_i=1.0_only_firstclass=false.tex}
    \caption{Empirical test error versus significance level for asymptotic approximations of the p-value based on
$\linearestimator$ with the mean total variation distance,
      evaluated on $10^4$ data sets of 250 labeled predictions that
      are randomly sampled from generative models with
      $\alpha = (1,\ldots,1)$ and $\beta = (1/m, \ldots, 1/m)$.}
  \end{center}
\end{figure}

\begin{figure}[!htbp]
  \begin{center}
    \tikzsetnextfilename{pvalues_SKCEul_mean_asymptotic_alpha_i=0.1_only_firstclass=false}
    \input{figures/pvalues/SKCEul_mean_asymptotic_alpha_i=0.1_only_firstclass=false.tex}
    \caption{Empirical test error versus significance level for asymptotic approximations of the p-value based on
$\linearestimator$ with the mean total variation distance,
      evaluated on $10^4$ data sets of 250 labeled predictions that
      are randomly sampled from generative models with
      $\alpha = (0.1,\ldots,0.1)$ and $\beta = (1/m, \ldots, 1/m)$.}
  \end{center}
\end{figure}

\begin{figure}[!htbp]
  \begin{center}
    \tikzsetnextfilename{pvalues_SKCEul_mean_asymptotic_alpha_i=1.0_only_firstclass=true}
    \input{figures/pvalues/SKCEul_mean_asymptotic_alpha_i=1.0_only_firstclass=true.tex}
    \caption{Empirical test error versus significance level for asymptotic approximations of the p-value based on
$\linearestimator$ with the mean total variation distance,
      evaluated on $10^4$ data sets of 250 labeled predictions that
      are randomly sampled from generative models with
      $\alpha = (1,\ldots,1)$ and $\beta = (1, 0,\ldots, 0)$.}
  \end{center}
\end{figure}

\begin{figure}[!htbp]
  \begin{center}
    \tikzsetnextfilename{pvalues_SKCEul_mean_asymptotic_alpha_i=0.1_only_firstclass=true}
    \input{figures/pvalues/SKCEul_mean_asymptotic_alpha_i=0.1_only_firstclass=true.tex}
    \caption{Empirical test error versus significance level for asymptotic approximations of the p-value based on
$\linearestimator$ with the mean total variation distance,
      evaluated on $10^4$ data sets of 250 labeled predictions that
      are randomly sampled from generative models with
      $\alpha = (0.1,\ldots,0.1)$ and $\beta = (1, 0, \ldots, 0)$.}
  \end{center}
\end{figure}%

\clearpage
\subsection{Modern neural networks}\label{sec:neural_networks}

In the main experiments of our paper discussed in \cref{sec:generative_models}
we focus on an experimental confirmation of the derived theoretical properties
of the kernel-based estimators and their comparison with the commonly used
$\ECE$. In contrast to \citet{guo17_calib_moder_neural_networ}, neither the study of the
calibration of different neural network architectures nor the re-calibration of
uncalibrated models are the main goal of our paper. The calibration measures
that we consider only depend on the predictions and the true labels, not
on how these predictions are computed. We therefore believe that directly
specifying the predictions in a \enquote{controlled way} results in a cleaner
and more informative numerical evaluation.

That being said, we recognize that this approach might result in an unnecessary
disconnect between the results of the paper and a practical use case. We
therefore conduct additional evaluations with different modern neural networks as
well. We consider pretrained ResNet, DenseNet, VGGNet, GoogLeNet, MobileNet, and
Inception neural networks~\citep{phan19_pytor_cifar} for the classification of the
CIFAR-10 image data
set~\citep{krizhevsky09_learn_multip_layer_featur_from_tiny_images}. The CIFAR-10
data set is a labeled data set of $32 \times 32$ colour images and consists of
50000 training and 10000 test images in 10 classes. The calibration of the
neural network models is estimated from their predictions on the CIFAR-10 test
data set. We use the same calibration error estimators and p-value approximations
as for the generative models above; however, the minimum number of samples per
bin in the data-dependent binning scheme of the $\ECE$ estimator is increased to
100 to account for the increased number of data samples.

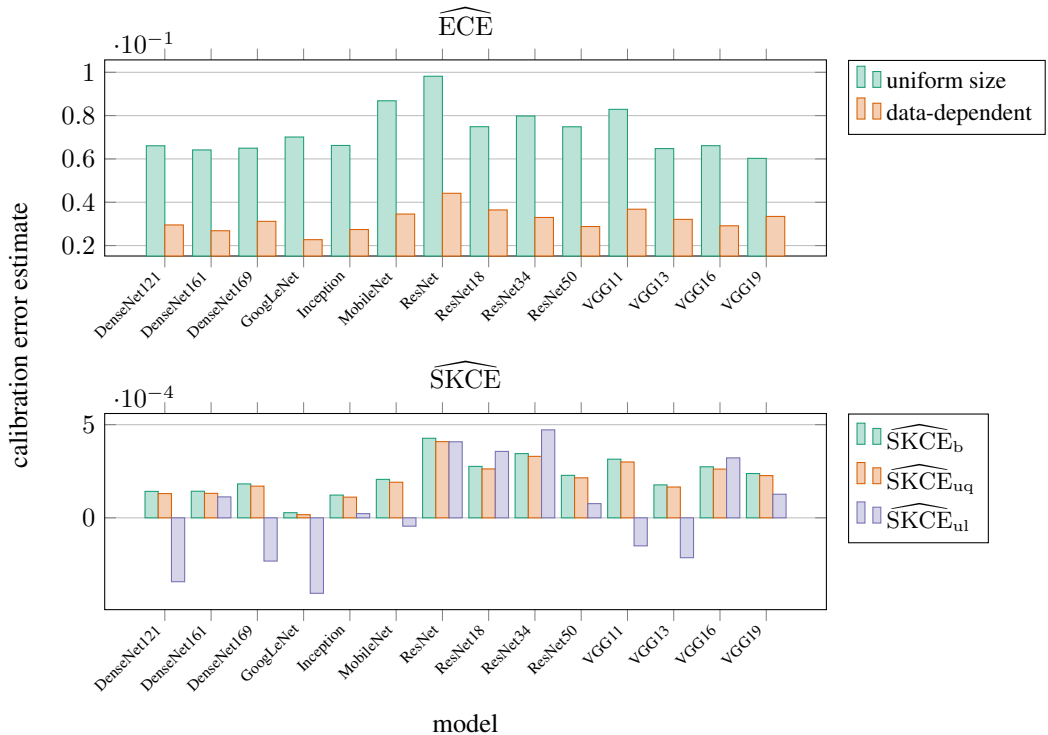
\begin{figure}[!htbp]
  \begin{center}
    \tikzsetnextfilename{PyTorch-CIFAR10_errors_comparison}
    \begin{tikzpicture}
\pgfplotstableread[col sep=comma, header=true]{/home/david/Documents/Projects/github/CalibrationPaper/paper/../experiments/data/PyTorch-CIFAR10/errors.csv}\datatable
\begin{groupplot}[group style={group size={1 by 2}, xlabels at={edge bottom}, vertical sep={0.15\linewidth}}, ybar={0pt}, ymajorgrids, width={0.8\linewidth}, height={0.3\linewidth}, xlabel={model}, xticklabels={DenseNet121,DenseNet161,DenseNet169,GoogLeNet,Inception,MobileNet,ResNet,ResNet18,ResNet34,ResNet50,VGG11,VGG13,VGG16,VGG19}, xtick={data}, xticklabel style={rotate={45}, anchor={east}, font={\tiny}}, scale ticks below exponent={0}, legend style={cells={anchor={west}}, legend pos={outer north east}, font={\small}}, cycle list={{Dark2-A,fill=Dark2-A!30!white,mark=none}, {Dark2-B,fill=Dark2-B!30!white,mark=none}, {Dark2-C,fill=Dark2-C!30!white,mark=none}}]
    \nextgroupplot[title={$\widehat{\ECE}$}, bar width={7pt}]
    \addplot
        table[x expr={\coordindex}, y={ECE_uniform}] {\datatable};
    \addplot
        table[x expr={\coordindex}, y={ECE_dynamic}] {\datatable};
    \legend{{uniform size},{data-dependent}}
    \nextgroupplot[title={$\widehat{\squaredkernelmeasure}$}, bar width={5pt}]
    \addplot
        table[x expr={\coordindex}, y={SKCEb_median}] {\datatable};
    \addplot
        table[x expr={\coordindex}, y={SKCEuq_median}] {\datatable};
    \addplot
        table[x expr={\coordindex}, y={SKCEul_median}] {\datatable};
    \legend{{$\biasedestimator$},{$\unbiasedestimator$},{$\linearestimator$}}
\end{groupplot}
\node[anchor=south, rotate=90, yshift=1em] at ($(group c1r1.north -| group c1r1.outer west)!0.5!(group c1r2.south -| group c1r2.outer west)$){calibration error estimate};
\end{tikzpicture}
    \caption{Calibration error estimates of modern neural networks for
      classification of the CIFAR-10 image data set.}
    \label{fig:cifar10_errors}
  \end{center}
\end{figure}

The computed calibration error estimates are shown in \cref{fig:cifar10_errors}.
As we argue in our paper, the raw calibration estimates are not interpretable
and can be misleading. The results in \cref{fig:cifar10_errors} endorse this
opinion. The estimators rank the models in different order (also the two
estimators of the $\ECE$), and it is completely unclear if the observed
calibration error estimates (in the order of $10^{-2}$ and $10^{-4}$!) actually
indicate that the neural network models are not calibrated.

\begin{figure}[!htbp]
  \begin{center}
    \tikzsetnextfilename{PyTorch-CIFAR10_pvalues_comparison}
    \begin{tikzpicture}
\pgfplotstableread[col sep=comma, header=true]{/home/david/Documents/Projects/github/CalibrationPaper/paper/../experiments/data/PyTorch-CIFAR10/pvalues.csv}\datatable
\begin{groupplot}[group style={group size={1 by 2}, xlabels at={edge bottom}, vertical sep={0.1\linewidth}}, ybar={0pt}, ymajorgrids, width={0.9\linewidth}, height={0.3\linewidth}, xlabel={model}, xtick={data}, xticklabel style={rotate={45}, anchor={east}, font={\tiny}}, legend style={legend columns={-1}, /tikz/every even column/.append style={column sep={0.05\linewidth}}, font={\small}}, cycle list={{Dark2-A,fill=Dark2-A!30!white,mark=none}, {Dark2-B,fill=Dark2-B!30!white,mark=none}, {Dark2-C,fill=Dark2-C!30!white,mark=none}, {Dark2-D,fill=Dark2-D!30!white,mark=none}, {Dark2-E,fill=Dark2-E!30!white,mark=none}, {Dark2-F,fill=Dark2-F!30!white,mark=none}, {Dark2-G,fill=Dark2-G!30!white,mark=none}}]
    \nextgroupplot[xticklabels={DenseNet121,DenseNet161,DenseNet169,GoogLeNet,Inception,MobileNet,ResNet}, skip coords between index={{7}{14}}, legend to name={cifar10_pvalues_legend}, bar width={5pt}]
    \addplot
        table[x expr={\coordindex}, y={ECE_uniform}] {\datatable};
    \addplot
        table[x expr={\coordindex}, y={ECE_dynamic}] {\datatable};
    \addplot
        table[x expr={\coordindex}, y={SKCEb_median_distribution_free}] {\datatable};
    \addplot
        table[x expr={\coordindex}, y={SKCEuq_median_distribution_free}] {\datatable};
    \addplot
        table[x expr={\coordindex}, y={SKCEul_median_distribution_free}] {\datatable};
    \addplot
        table[x expr={\coordindex}, y={SKCEuq_median_asymptotic}] {\datatable};
    \addplot
        table[x expr={\coordindex}, y={SKCEul_median_asymptotic}] {\datatable};
    \legend{{$\mathbf{C}_{\mathrm{uniform}}$},{$\mathbf{C}_{\mathrm{data-dependent}}$},{$\mathbf{D}_{\mathrm{b}}$},{$\mathbf{D}_{\mathrm{uq}}$},{$\mathbf{D}_{\mathrm{l}}$},{$\mathbf{A}_{\mathrm{uq}}$},{$\mathbf{A}_{\mathrm{l}}$}}
    \nextgroupplot[xticklabels={ResNet18,ResNet34,ResNet50,VGG11,VGG13,VGG16,VGG19}, skip coords between index={{0}{7}}, bar width={5pt}]
    \addplot
        table[x expr={\coordindex}, y={ECE_uniform}] {\datatable};
    \addplot
        table[x expr={\coordindex}, y={ECE_dynamic}] {\datatable};
    \addplot
        table[x expr={\coordindex}, y={SKCEb_median_distribution_free}] {\datatable};
    \addplot
        table[x expr={\coordindex}, y={SKCEuq_median_distribution_free}] {\datatable};
    \addplot
        table[x expr={\coordindex}, y={SKCEul_median_distribution_free}] {\datatable};
    \addplot
        table[x expr={\coordindex}, y={SKCEuq_median_asymptotic}] {\datatable};
    \addplot
        table[x expr={\coordindex}, y={SKCEul_median_asymptotic}] {\datatable};
\end{groupplot}
\node[anchor=north] at ($(group c1r2.west |- group c1r2.outer south)!0.5!(group c1r2.east |- group c1r2.outer south)$){\pgfplotslegendfromname{cifar10_pvalues_legend}};
\node[anchor=south, rotate=90] at ($(group c1r1.north -| group c1r1.outer west)!0.5!(group c1r2.south -| group c1r2.outer west)$){bound/approximation of p-value};
\end{tikzpicture}
    \caption{Bounds and approximations of the p-value of modern neural networks
      for classification of the CIFAR-10 image data set for different calibration
      error estimators, assuming the models are calibrated.}
    \label{fig:cifar10_pvalues}
  \end{center}
\end{figure}
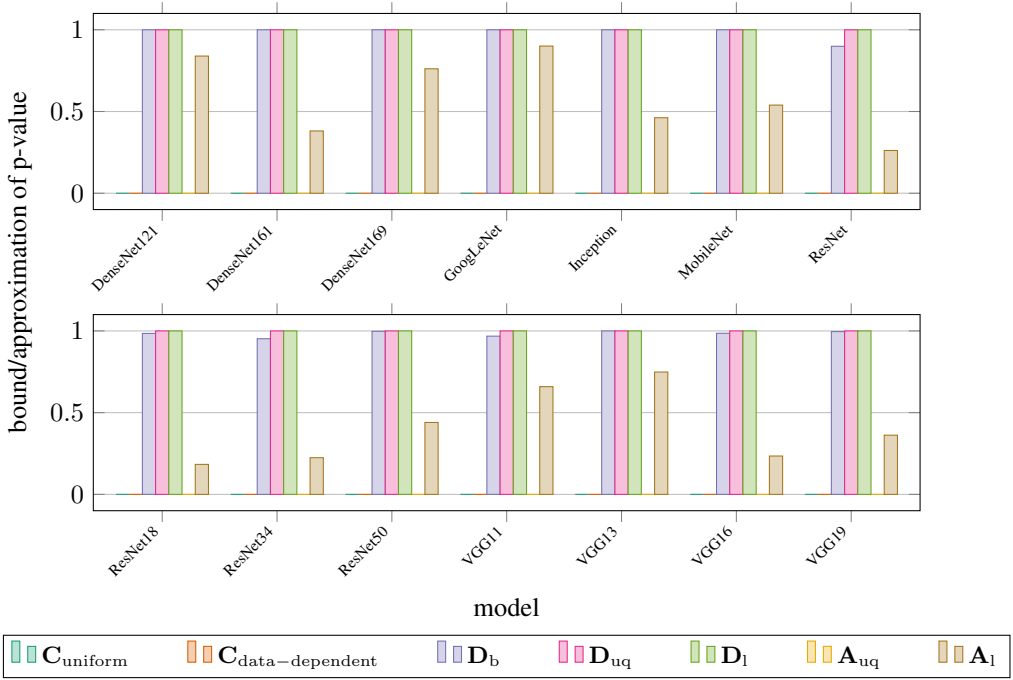

Hence to obtain an interpretable measure, we consider different bounds and
approximations of the p-value for the calibration error estimators, assuming the
models are calibrated. More concretely, we estimate the p-value by consistency
resampling of the standard ($\mathbf{C}_{\mathrm{uniform}}$) and the data-dependent
($\mathbf{C}_{\mathrm{data-dependent}}$) estimator of the $\ECE$, evaluate the
distribution-free bounds of the p-value for the estimators $\biasedestimator$
($\mathbf{D}_{\mathrm{b}}$), $\unbiasedestimator$ ($\mathbf{D}_{\mathrm{uq}}$), and
$\linearestimator$ ($\mathbf{D}_{\mathrm{l}}$) of the $\squaredkernelmeasure$,
and approximate the p-value using the asymptotic distribution of the estimators
$\unbiasedestimator$ ($\mathbf{A}_{\mathrm{uq}}$) and $\linearestimator$
($\mathbf{A}_{\mathrm{l}}$). The results are shown in \cref{fig:cifar10_pvalues}.

The approximations obtained by consistency resampling are always zero.
However, since our controlled experiments with the generative models showed that
consistency resampling might underestimate the p-value of calibrated
models on average, these approximations could be misleading. On the contrary, the
bounds and approximations of the p-value for the estimators of the
$\squaredkernelmeasure$ are theoretically well-founded. In our experiments with
the generative models, the asymptotic distribution of the estimator
$\unbiasedestimator$ seemed to allow to approximate the p-value quite accurately
on average and yielded very powerful tests. For all studied neural network models
these p-value approximations are zero, and hence for all models we would always
reject the null hypothesis of calibration. The p-value approximations based on
the asymptotic distribution of the estimator $\linearestimator$ vary between
around $0.18$ for the ResNet18 and $0.91$ for the GoogLeNet model. The higher
p-value approximations correspond to the increased empirical test errors with the
uncalibrated generative models compared to the tests based on the asymptotic
distribution of the estimator $\unbiasedestimator$. Most distribution-free
bounds of the p-value are between 0.99 and 1, indicating again that these bounds
are quite loose.

All in all, the evaluations of the modern neural networks seem to match the
theoretical expectations and are consistent with the results we obtained in
the experiments with the generative models. Moreover, the p-value
approximations of zero are consistent with \citet{guo17_calib_moder_neural_networ}'s
finding that modern neural networks are often not calibrated.

\subsection{Computational time}\label{sec:computational_time}

The computational time, although dependent on our Julia implementation and
the hardware used, might provide some insights to the interested reader in
addition to the algorithmic complexity. However, in our opinion, a fair
comparison of the proposed calibration error estimators should take into
account the error of the calibration error estimation, similar to work
precision diagrams for numerical differential equation solvers.

A simple comparison of the computational time for the calibration error
estimators used in the experiments with the generative models in
\cref{sec:generative_models} on our computer (3.6~GHz) shows the expected
scaling of the computational time with increasing number of samples. As
\cref{fig:timings} shows, even for 1000 samples and 1000 classes the estimators
$\biasedestimator$ and $\unbiasedestimator$ with the median heuristic can be
evaluated in around 0.1~seconds.

\begin{figure}[!htbp]
  \begin{center}
    \tikzsetnextfilename{timings}
    \input{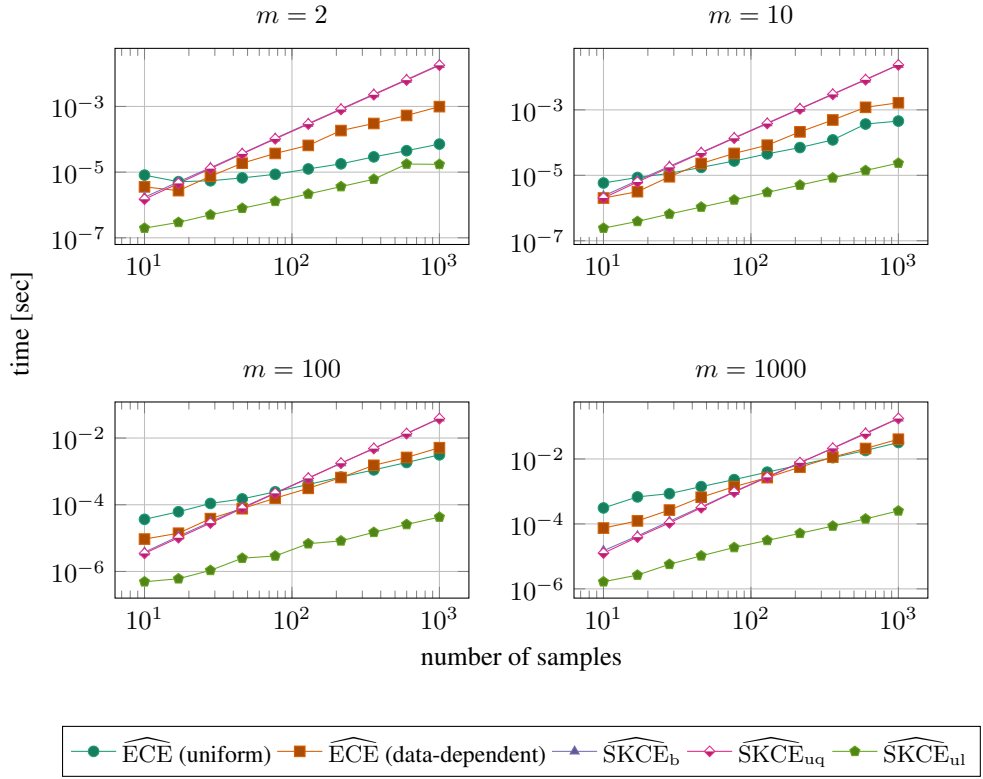}
    \caption{Computational time for the evaluation of calibration error estimators
      on data sets with different number of classes versus number of data samples.}
    \label{fig:timings}
  \end{center}
\end{figure}

\end{document}